\documentclass[11pt]{article}

\usepackage{fullpage}
\usepackage[style=alphabetic]{biblatex}

\usepackage{amsmath,amsthm,amsfonts,amssymb}
\usepackage{hyperref}
\usepackage{color}
\usepackage{mathrsfs}
\usepackage{bm}
\usepackage{multirow}
\usepackage{booktabs}
\usepackage{makecell}
\usepackage[pdftex]{graphicx}
\usepackage{subfigure}
\usepackage{caption}
\usepackage{thmtools}
\usepackage{thm-restate}
\usepackage{setspace}
\usepackage{makecell}
\definecolor{light-gray}{gray}{0.85}
\usepackage{colortbl}
\usepackage{xcolor}
\usepackage{hhline}

\usepackage{url}    
\usepackage{booktabs}       
\usepackage{nicefrac}       
\usepackage{microtype}  

\usepackage{amsxtra,verbatim,epsfig,enumerate,array,mathtools,dsfont,multicol}

\usepackage{algorithm,etoolbox}
\usepackage{algorithmic}

\newtheorem{theorem}{Theorem}[section]
\newtheorem{lemma}[theorem]{Lemma}

\newtheorem{remark}[theorem]{Remark}

\theoremstyle{definition}
\newtheorem{definition}[theorem]{Definition}

\newcommand{\cA}{\mathcal{A}}

\newcommand{\cD}{\mathcal{D}}
\newcommand{\cE}{\mathcal{E}}

\newcommand{\cL}{\mathcal{L}}
\newcommand{\cM}{\mathcal{M}}
\newcommand{\cN}{\mathcal{N}}

\newcommand{\cP}{\mathcal{P}}

\newcommand{\cR}{\mathcal{R}}
\newcommand{\cS}{\mathcal{S}}
\newcommand{\cT}{\mathcal{T}}
\newcommand{\cU}{\mathcal{U}}

\newcommand{\cW}{\mathcal{W}}

\newcommand{\R}{\mathbb{R}}
\newcommand{\E}{\mathbb{E}}

\newcommand{\argmax}{\mathop{\mathrm{argmax}}}

\newcommand{\beq}{\begin{equation}}
	\newcommand{\eeq}{\end{equation}}
\newcommand{\beqn}{\begin{equation*}}
	\newcommand{\eeqn}{\end{equation*}}
\newcommand{\beqa}{\begin{eqnarray}}
	\newcommand{\eeqa}{\end{eqnarray}}
\newcommand{\beqan}{\begin{eqnarray*}}
	\newcommand{\eeqan}{\end{eqnarray*}}

\renewcommand{\epsilon}{\varepsilon}
\renewcommand{\leq}{\leqslant}
\renewcommand{\geq}{\geqslant}
\renewcommand{\hat}{\widehat}

\newcommand{\deemph}[1]{{\color{black!40}#1}}

\newcommand{\High}[1]{{\color{black}{#1}}}

\usepackage{times}

\begin{document}

\title{Distributed Linear Bandits with Differential Privacy}

\author{%
Fengjiao Li\footnotemark[1]\footnote{{Fengjiao Li (\texttt{fengjiaoli@sxu.edu}), Bo Ji (\texttt{boji@vt.edu}), Department of Computer Science, Virginia Tech.}} \quad Xingyu Zhou\footnote{Xingyu Zhou (\texttt{xingyu.zhou@wayne.edu}), Department of Electrical and Computer Engineering, Wayne State University. This work is supported in part by the NSF  grants under CNS-2312833, CNS-2112694, CNS-2153220, and CNS-2312835, Fundamental Research Program of Shanxi Province under 202303021222030, the Commonwealth Cyber Initiative (CCI), and Nokia Corporation. A preliminary version of this work was presented at IEEE/IFIP WiOpt 2022 \cite{li2022differentially}, and a journal version was accepted to IEEE Transactions on Network Science and Engineering (TNSE)\cite{li2024distributed}.} \quad Bo Ji\footnotemark[1]  
}

\date{}

\maketitle

\begin{abstract}
In this paper, we study the problem of global reward maximization with only partial distributed feedback. This problem is motivated by several real-world applications (e.g., cellular network configuration, dynamic pricing, and policy selection) where an action taken by a central entity influences a large population that contributes to the global reward. However, collecting such reward feedback from the entire population not only incurs a prohibitively high cost, but often leads to privacy concerns. To tackle this problem, we consider differentially private distributed linear bandits, where only a subset of users from the population are selected (called clients) to participate in the learning process and the central server learns the global model from such partial feedback by iteratively aggregating these clients' local feedback in a differentially private fashion. We then propose a unified algorithmic learning framework, called differentially private distributed phased elimination (DP-DPE), which can be naturally integrated with popular differential privacy (DP) models (including central DP, local DP, and shuffle DP). 
Furthermore, we prove that DP-DPE achieves both sublinear regret and sublinear communication cost.
Interestingly, DP-DPE also achieves privacy protection ``for free'' in the sense that the additional cost due to privacy guarantees is a lower-order additive term. In addition, as a by-product of our techniques, the same results of ``free" privacy can also be achieved for the standard differentially private linear bandits. Finally, we conduct simulations to corroborate our theoretical results and demonstrate the effectiveness of DP-DPE.
\end{abstract}

\section{Introduction}
The bandit learning models have been widely adopted for many sequential decision-making problems, such as clinical trials, recommender systems, and configuration selection. Each action (called arm), if selected in a round, generates a (noisy) reward. 	By observing such reward feedback, the learning agent gradually learns the unknown parameters of the model (e.g., mean rewards) and decides the action in the next round. The objective here is to maximize the cumulative reward over a finite time horizon, balancing the tradeoff between \emph{exploitation} and \emph{exploration}. While the stochastic multi-armed bandits (MAB) model is useful for these applications~\cite{lai1985asymptotically}, one key limitation is that actions are assumed to be independent, which, however, is usually not the case in practice.	Therefore, the linear bandit model that captures the correlation among actions has been extensively studied~\cite{lattimore2020bandit,abbasi2011improved,li2010contextual}. 
	
\begin{figure}[!th]
\centering
\includegraphics[width=0.5\columnwidth]{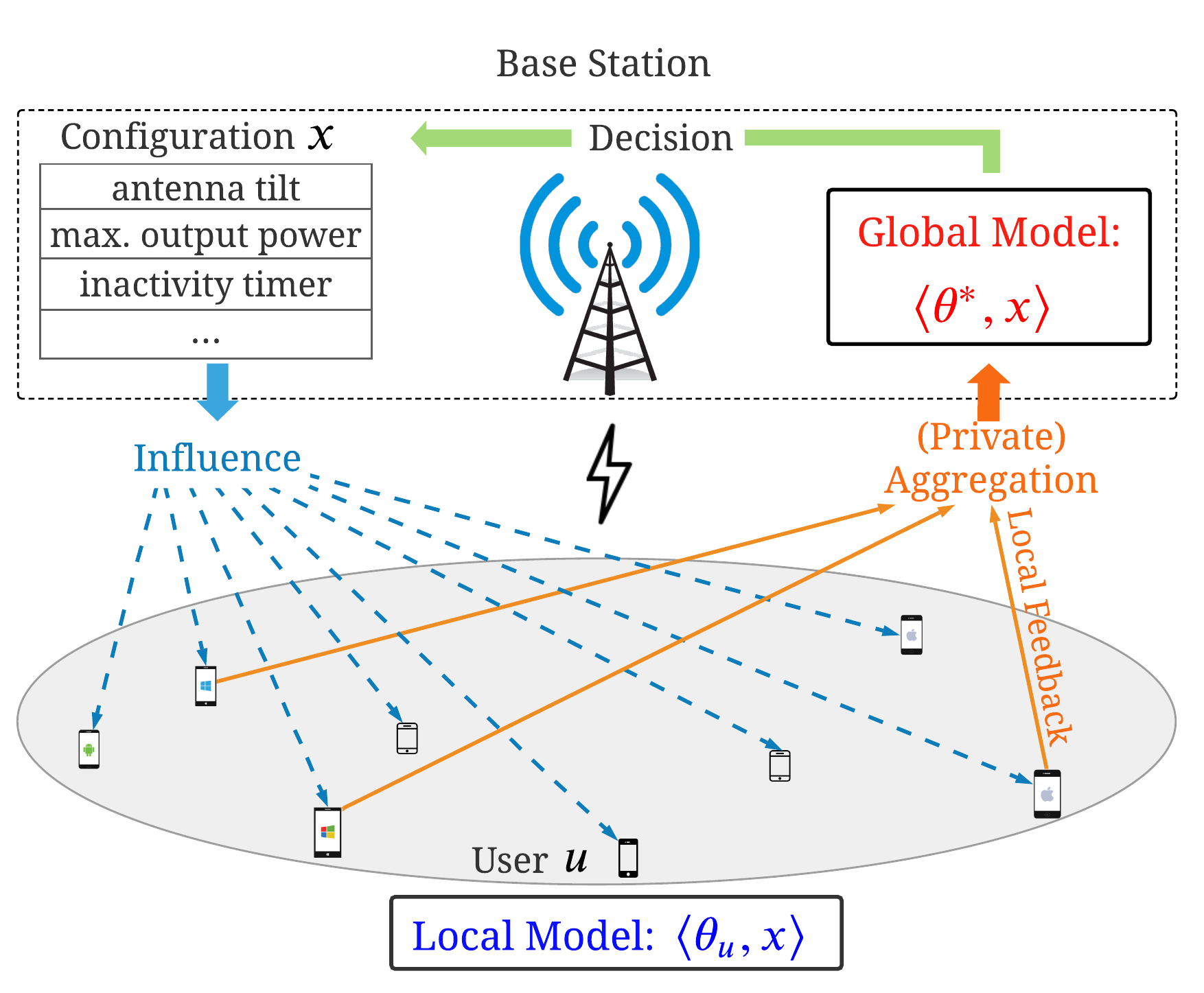}
\caption{Cellular network configuration: a motivating application of global reward maximization with partial feedback in a linear bandit setting.		} 
\label{fig:network_config}
\end{figure}
	
In this paper, we introduce a new linear bandit setting where the reward of an action could be from a large population. Take the cellular network configuration as an example (see Fig.~\ref{fig:network_config}). The configuration (antenna tilt, maximum output power, inactivity timer, etc.) of a base station (BS),  
with feature representation\footnote{Similar to many linearly parameterized bandits (e.g., \cite{li2010contextual}), 
we may represent each configuration by a $d$-dimensional feature vector through some feature mapping. 
} $x\in \R^d$,
influences all the users under the coverage of this BS~\cite{mahimkar2021auric}.
After a configuration is applied, the BS receives a reward in terms of the network-level performance, which accounts for the performance of all users within the coverage (e.g., average user throughput). 
Specifically, let the mean global reward of configuration $x$ be $f(x)=\langle \theta^*, x\rangle$, where $\theta^* \in \R^d$ represents the unknown global parameter.
While some configuration may work best for a specific user, only one configuration can be applied at the BS at a time, which, however, simultaneously influences all the users within the coverage. Therefore, the goal here is to find the best configuration that maximizes the global reward (i.e., the network-level performance).

\begin{table*}[t]
\centering
\caption{Summary of main results}
\label{tab:performance_summary}
		\scalebox{0.88}{
	\begin{tabular}{lccc}
	\toprule
	Algorithm\footnotemark & Regret\footnotemark & Communication cost\footnotemark & Privacy\\
	\midrule
	DPE& $O\left(  T^{1-\alpha/2}\sqrt{\log (kT)}\right)$  &  $O(dT^{\alpha})$  & None  \\
	CDP-DPE & $ O\left( T^{1-\alpha/2}\sqrt{\log (kT)} + d^{3/2}T^{1-\alpha}\sqrt{\ln(1/\delta)\log(kT)}/\epsilon\right)$  & $O(dT^{\alpha}) $ & $(\epsilon, \delta)$-DP\\
	LDP-DPE& $O\left(T^{1-\alpha/2}\sqrt{\log(kT)} + d^{3/2}T^{1-\alpha/2}\sqrt{\ln(1/\delta)\log(kT)}/\epsilon\right)$& $O(dT^{\alpha})$ & $(\epsilon, \delta)$-LDP\\
	SDP-DPE & $ O\left( T^{1-\alpha/2}\sqrt{\log (kT)}+d^{3/2}T^{1-\alpha}\ln(d/\delta)\sqrt{\log(kT)}/\epsilon\right)$ & $O(dT^{3\alpha/2})$ (bits)&$(\epsilon, \delta)$-SDP\\
	\bottomrule
	\multicolumn{4}{l}{\footnotesize $^1$DPE is the non-private DP-DPE algorithm; CDP-DPE, LDP-DPE, and SDP-DPE represent the DP-DPE algorithm in the central, local, and} \\ 
	\multicolumn{4}{l}{\footnotesize shuffle models, respectively, which guarantee $(\epsilon, \delta)$-DP, $(\epsilon, \delta)$-LDP, and $(\epsilon, \delta)$-SDP, respectively.}\\
	\multicolumn{4}{l}{\footnotesize $^2$In the regret upper bounds, we ignore lower-order terms for simplicity. $T$ is the time horizon, $k$ is the number of actions, $d$ is the dimension}\\
	\multicolumn{4}{l}{\footnotesize  of the action space, and $\alpha$ is a design parameter that can be used to tune the tradeoff between the regret and the communication cost. }\\
	\multicolumn{4}{l}{\footnotesize $^3$While the communication cost of CDP-DPE and LDP-DPE is measured in the number of real numbers transmitted between the clients and }\\
	\multicolumn{4}{l}{\footnotesize the server, SDP-DPE directly uses bits for reporting feedback. A detailed discussion is provided in Section~\ref{sec:DP-instantiation}.}
	\end{tabular}
		}
\end{table*}
At first glance, it seems that one can address the above problem by applying existing linear bandit algorithms (e.g., LinUCB \cite{li2010contextual}) to learn the global parameter $\theta^*$. However, this would require collecting reward feedback from the entire population, which could incur a prohibitively high cost or could even be impossible to implement in practice when the population is large. To learn the global parameter, one natural way is to sample a subset of users from the population and aggregate this distributed partial feedback. This leads to a new problem we consider in this paper: \emph{global reward maximization with partial feedback in a distributed linear bandit setting,} which can be also applied to several other practical applications, including dynamic pricing and public policy selection~\cite{bouneffouf2019survey,liprivate}.  As in many distributed supervised learning problems~\cite{bassily2019private,geyer2017differentially,girgis2021shuffled}, privacy protection is also of significant importance in our setting as clients' local feedback may contain their sensitive information. In summary, we are interested in the following fundamental question:
\emph{How to privately achieve global reward maximization with only partial distributed feedback?}

To that end, we introduce a new model called \emph{differentially private distributed linear bandit (DP-DLB)}. 
In DP-DLB, there is a global linear bandit model $f(x) = \langle \theta^*, x\rangle$ with an unknown parameter $\theta^* \in \R^d$ at the central server (e.g., the BS); each user $u$ of a large population has a local linear bandit model $f_u(x) = \langle \theta_u, x\rangle$, which represents the mean local reward for user $u$. Here, we assume that each user $u$ has a local parameter $\theta_u \in \R^d$, motivated by the fact that the mean local reward (e.g., the expected throughput of a user under a certain network configuration) varies across the users. In addition, each local parameter $\theta_u$ is unknown and is assumed to be a realization of a random vector with the mean being the global model parameter $\theta^*$.
The server makes decisions based on the estimated global model, which can be learned through sampling a subset of users (referred to as clients) and iteratively aggregating these distributed partial feedback. 
While sampling more clients could improve the learning accuracy and thus lead to a better performance, it also incurs a higher communication cost. Therefore, it is important to address this tradeoff in the design of communication protocols.
Furthermore, to protect users' privacy, we resort to \emph{differential privacy (DP)} to guarantee that clients' sensitive information will not be inferred by an adversary. 
Therefore, the goal is to maximize the cumulative global reward (or equivalently minimize the regret due to not choosing the optimal action in hindsight) in a communication-efficient manner while providing privacy guarantees for the participating clients. 
Our main contributions are summarized as follows. 

\begin{list}{\labelitemi}{\leftmargin=1em \itemindent=0em \itemsep=.2em}

\item We present a new distributed linear bandit setting where only partial feedback is available, leading to a novel problem of global reward maximization with distributed partial feedback. In addition to the traditional exploitation and exploration tradeoff, 
learning with distributed feedback introduces two practical challenges: communication efficiency and privacy concerns. This adds an extra layer of difficulty in the design of learning algorithms. 
		
\item To address these challenges, we introduce a DP-DLB model and develop a carefully-crafted algorithmic learning framework, called differentially private distributed phased elimination (DP-DPE), which allows the server and the clients to work in concert and can be naturally integrated with several state-of-the-art DP trust models (including central model, local model, and shuffle model). This unified framework enables us to systemically study the key regret-communication-privacy tradeoff. 
		
\item We then establish the regret-communication-privacy tradeoff of DP-DPE in various settings including the non-private case as well as the central, local, and shuffle DP models. Our main results are summarized in Table~\ref{tab:performance_summary}.
From Table~\ref{tab:performance_summary}, we observe that the additional regret incurred by privacy is only a lower-order additive term, which is dominated by the regret from learning (i.e.,  $\tilde{O}({T^{1-\alpha}}/{\epsilon})$ vs.\footnote{Here the $\tilde{O}(\cdot)$ notation hides the dependence on $polylog(T)$, the dimension $d$, and privacy parameter  $\delta$.}  $\tilde{O}(T^{1-\alpha/2})$). In this sense, we say that DP-DPE might achieve privacy ``for free'' following \cite{agarwal2017price}. 
Moreover, this is the first work that considers the shuffle model in distributed linear bandits to attain a better regret-privacy tradeoff, i.e., guaranteeing a similar privacy protection as the strong local model while achieving the same regret as the central model. We further perform simulations on synthetic data to corroborate our theoretical results.
\item {Finally, we provide an interesting discussion about achieving privacy ``for free".  We first highlight an interesting connection between our introduced DP-DLB formulation and the differentially private stochastic convex optimization (DP-SCO) problem in terms of achieving privacy ``for-free''. This bridge between our online bandit learning and the standard supervised learning might be of independent interest. Furthermore,  differential privacy may also be ensured ``for free'' for standard linear bandits as well with minor modifications of our developed techniques. }
\end{list}

\section{System Model and Problem Formulation}\label{sec:problem_formulation}
We begin with some notations: $[N]\triangleq \{1,\dots, N\}$ for any positive integer $N$; $|S|$ denotes the cardinality of set $S$; $\Vert x\Vert_2$ denotes the $\ell_2$-norm of vector $x$; the inner product is denoted by $\langle \cdot, \cdot \rangle$. For a positive definite matrix $A\in\R^{d\times d}$, the weighted $\ell_2$-norm of vector $x\in\R^d$ is defined as $\Vert x\Vert_{A} \triangleq \sqrt{x^{\top}Ax}$. 
\subsection{Global Reward Maximization with Partial Feedback}
We consider the global reward maximization problem over a large population containing an infinite number of users, which is a sequential decision making problem. In each round $t$, the learning agent (e.g., the BS or the policy maker) selects an action $x_t$ from a finite decision set $\cD \subseteq \{x\in \R^d: \Vert x\Vert_2^2\leq 1\}$ with $|\cD|=k$. This action leads to a global reward with mean $\langle \theta^*, x_t \rangle$, where $\theta^*\in \R^d$ with $\Vert\theta^*\Vert_2\leq 1$ is unknown to the agent. This global reward captures the overall effectiveness of action $x_t$ over a large population $\cU$. 
The local reward of action $x_t$ at user $u$ has a mean $\langle \theta_u, x_t \rangle$, where $\theta_u\in \R^d$ is the local parameter, which is assumed to be a realization of a random vector with mean $\theta^*$ and is also unknown. 
Let $x^* \triangleq \argmax_{x\in\cD} \langle \theta^*, x\rangle$	be the unique global optimal action. Then, the objective of the agent is to maximize the cumulative global reward, or equivalently, to minimize the regret defined as follows:
\begin{equation}
	\label{eq:regret}
	R(T) \triangleq T \langle  \theta^*, x^* \rangle - \sum_{t=1}^T \langle  \theta^*,x_{t} \rangle.
\end{equation}

At first glance, standard linear bandit algorithms (e.g., LinUCB in~\cite{li2010contextual}) can be applied to addressing the above problem. However, the exact reward here is a global quantity, which is the average over the entire population. The learning agent may not be able to observe this exact reward, since collecting such global information from the entire population incurs a	prohibitively high cost, is often impossible to implement in practice, and could lead to privacy concerns.

\subsection{Differentially Private Distributed Linear Bandits} \label{sec:DLB}
To address the above problem, we consider a \emph{differentially private distributed linear bandit (DP-DLB)} formulation, where there are two important entities: a central server (which wants to learn the global model) and participating clients (i.e., a subset of users from the population who are willing to share their feedback). In the following, we discuss important  aspects of the DP-DLB formulation.
	
\noindent \textbf{Server.} The server aims to learn the global linear bandit model, i.e., unknown parameter $\theta^*$. In each round~$t$, it selects an action $x_t$ with the objective of maximizing the cumulative global reward $\sum_{t=1}^T\langle \theta^*, x_t \rangle$.  Without observing the exact reward of action $x_t$, the server collects only partial feedback from a subset of users sampled from the population, called \emph{clients}, and then aggregates this partial feedback to update the estimate of the global parameter $\theta^*$. Based on the updated model, the server chooses an action in the next round. 
	
\noindent \textbf{Clients.} We assume that each participating client is randomly sampled from the population and is independent from each other and also from other randomness. Specifically, we assume that local parameter $\theta_u$ at client $u$ satisfies $\theta_u = \theta^* + \xi_u $, where $\xi_u\in \R^d$ is a zero-mean $\sigma$-sub-Gaussian random vector\footnote{A random vector $\xi \in \R^d$ is said to be $\sigma$-sub-Gaussian if $\E[\xi]=0$ and $v^{\top} \xi$ is $\sigma$-sub-Gaussian for any unit vector $v\in \R^d$ and $\Vert v\Vert_2=1$~\cite{buhlmann2011statistics}.} and is independently and identically distributed (\emph{i.i.d.}) across all clients.
Let $U_t$ be the set of clients in round $t$. After action $x_t$ is chosen by the server in round $t$, each client $u \in U_t$ observes a noisy local reward: $y_{u,t}= \langle \theta_u, x_{t}\rangle + \eta_{u,t}$, where $\eta_{u,t}$ is a conditionally $1$-sub-Gaussian\footnote{Consider noise sequence $\{\eta_t\}_{t=1}^{\infty}$. As in the general linear bandit model \cite{lattimore2020bandit},  $\eta_{t}$ is assumed to be conditionally $1$-sub-Gaussian, meaning $\E[e^{\lambda \eta_{t}}|x_{1:t}, \eta_{1:t}] \leq \exp(\lambda^2/2)$ for all $\lambda\in \R$, where $a_{i:j}$ denotes the subsequence $a_i, \dots, a_j$.} noise and \emph{i.i.d.} across the clients and over time. We also assume that the local rewards are bounded, i.e., $| y_{u,t} |\leq B$, for all $u\in \cU$ and $t\in [T]$.

\noindent \textbf{Communication.} 
The communication happens when the clients report their feedback to the server. At the beginning of each communication step, each participating client reports feedback to the server based on the local reward observations during a certain number of rounds. In particular, the time duration between reporting feedback is called a phase. By aggregating such feedback from the clients, the server estimates the global parameter $\theta^*$ and adjusts its decisions in the following rounds accordingly. We assume that the clients do not quit before a phase ends. By slightly abusing the notation, we use $U_l$ to denote the set of clients in the $l$-th phase.
	
The communication cost is a critical factor in DP-DLB.	As in \cite{wang2019distributed}, we define the communication cost as the total number of real numbers (or bits, depending on the adopted DP model) communicated between the server and the clients. Let $L$ be the number of phases in $T$ rounds, and let $N_l$ be the number of real numbers (or bits) communicated in the $l$-th phase. 	Then, the total communication cost, denoted by $C(T)$, is  
\begin{equation}
C(T) \triangleq \sum_{l=1}^L |U_l|N_l. \label{eq:comm_cost}
\end{equation}

\noindent \textbf{Data privacy.} In practice, even if users are willing to share their feedback, they typically require privacy protection as a premise.
\High{\emph{Differential privacy (DP)}~\cite{dwork2006calibrating} is a mathematical framework for ensuring the privacy of individuals in datasets. 
Specifically, by observing the calculation/statistics/model update from a set of individual data, an adversary cannot infer too much information about any specific individual. In this sense, DP can protect any existing or future attacks in that any adversary tries to infer any individual's information would fail no matter how much computation power they have or how much side information they have (i.e., even though the adversary has access to all the others' information except the targeted one).}
To that end, we resort to 
DP to formally address the privacy concerns in the learning process. 
More importantly, instead of only considering the standard central model where the central server is responsible for protecting privacy, we will also incorporate other popular DP models, including the stronger local model (where each client directly protects her data)~\cite{kasiviswanathan2011can} and the recently proposed shuffle model (where a trusted shuffler between clients and server is adopted to amplify privacy)~\cite{cheu2019distributed}, in a unified algorithmic learning framework.

\section{Algorithm Design}\label{sec:alg_design}
In this section, we first present the key challenges associated with the introduced DP-DLB model and then explain how the developed DP-DPE framework addresses these challenges.
	
\subsection{Key Challenges} \label{sec:challenges}
To solve the problem of global reward maximization with partial distributed feedback using the DP-DLB formulation, we face four key challenges, discussed in detail below.
	
As in the standard stochastic bandits problem, there is an uncertainty due to noisy rewards of each chosen action, which is called the \emph{action-related uncertainty}. In addition to this, we face another type of uncertainty related to the sampled clients in DP-DLB, called the \emph{client-related uncertainty}. The client-related uncertainty lies in estimating the global model at the server based on randomly sampled clients with \emph{biased} local models. Note that the global model may not be accurately estimated even if exact rewards of the sampled clients are known when the number of clients is insufficient.	Therefore, the first challenge lies in \emph{simultaneously addressing both types of uncertainty in a sample-efficient way} (Challenge \textcircled{a}).

To handle the newly introduced client-related uncertainty, we must sample a sufficiently large number of clients so that the global parameter can be accurately estimated using the partial distributed feedback. However, too many clients result in a large communication cost (see Eq.~(\ref{eq:comm_cost})). 	Therefore, the second challenge is to \emph{decide the number of sampled clients to balance the regret (due to the client-related uncertainty) and the communication cost} (Challenge \textcircled{b}).
	
Finally, to ensure privacy guarantees for the clients, one needs to add additional perturbations (or noises) to the local feedback. \emph{Such randomness introduces another type of uncertainty to the learning process} (Challenge \textcircled{c}), and \emph{it is unclear how to integrate different trust DP models into a unified algorithmic learning framework} (Challenge \textcircled{d}). These add an extra layer of difficulty to the design of learning algorithms.

\smallskip

\noindent \textbf{Main ideas.} 
In the following, we present our main ideas for addressing the above challenges. We propose a phased elimination algorithm \High{as in \cite{lattimore2020learning}} that gradually eliminates suboptimal actions by periodically aggregating the local feedback from the sampled clients in a privacy-preserving manner. 
To address the multiple types of uncertainty when estimating the global reward (\textcircled{a} and \textcircled{c}), we carefully construct a confidence width to incorporate all three types of uncertainty.
To achieve a sublinear regret while saving communication cost (\textcircled{b}), we increase both the phase length and the number of clients exponentially. To ensure privacy guarantees (\textcircled{d}), we introduce a \textsc{Privatizer} that can be easily tailored under different DP models. The \textsc{Privatizer} is a process consisting of tasks to be collaboratively completed by the clients, the server, and/or even a trusted third party. To keep it general, we use $\cP = (\cR, \cS, \cA)$ to denote a \textsc{Privatizer}, where $\cR$ is the procedure at each client (usually a local randomizer), $\cS$ is a trusted third party that helps privatize data (e.g., a shuffler that permutes received messages), and $\cA$ is an analyzer operated at the central server. Next, we will show how to integrate these main ideas into a unified algorithmic learning framework.
	
\subsection{Differentially Private Distributed Phased Elimination (DP-DPE)}  
With the main ideas presented above, we now propose a unified algorithmic learning framework, 	called \emph{differentially private distributed phased elimination (DP-DPE)}, which is presented in Algorithm~\ref{alg:dpe}.
The DP-DPE  runs in phases and operates with the coordination of the central server and the participating clients in a synchronized manner. At a high level, each phase consists of the following three steps: 
\begin{list}{\labelitemi}{\leftmargin=1em \itemindent=0em \itemsep=.2em}
\item \textbf{Action selection (Lines~\ref{alg_action_selection}-\ref{alg_action_slc_end}):} computing a near-$G$-optimal design (i.e., a distribution) over a set of possibly optimal actions and playing these actions; 
\item \textbf{Clients sampling and private feedback aggregation (Lines~\ref{alg_client_sample}-\ref{alg_output_privatizer}):} sampling participating clients and aggregating their local feedback in a privacy-preserving fashion; 
\item \textbf{Parameter estimation and action elimination {(Lines~\ref{alg_lse}-\ref{alg_update})}:} using (privately) aggregated data to estimate $\theta^*$ and eliminating actions that are likely to be suboptimal.
\end{list}

\begin{algorithm}[!t]
\caption{Differentially Private Distributed Phased Elimination (DP-DPE)}
\label{alg:dpe}
\begin{algorithmic}[1]
	\STATE \textbf{Input:}  $\cD\subseteq \R^d$, $\alpha\in (0,1)$, $\beta\in (0,1)$, and $\sigma_n$ 
	\STATE \textbf{Initialization:} $l=1$, $t_1=1$, $\cD_1=\cD$, and  $h_1=2$ 
	\WHILE{$t_l\leq T$} 
	\STATE Find a distribution $\pi_l(\cdot)$ over $ \cD_l$ such that  $ g(\pi_l) \triangleq \max_{x\in \cD_l} \Vert x \Vert_{V(\pi_l)^{-1}}^2 \leq 2d$ and $|\text{supp}(\pi_l)|\leq 4d\log\log d +16$, where $V(\pi_l) \triangleq \sum_{x\in\cD_l} \pi_l(x)xx^{\top}$			\label{alg_action_selection}
	\STATE Let $T_l(x)=\lceil h_l \pi_l(x)\rceil $ for each $x \in \text{supp}(\pi_l)$ and $T_l=\sum_{x\in \text{supp}(\pi_l)} T_l(x)$
	\STATE Play each action $x\in \text{supp}(\pi_l)$ exactly $T_l(x)$ times if not reaching $T$ 
	\label{alg_action_slc_end}
	\STATE Randomly select $\lceil 2^{\alpha l}\rceil$ participating clients $U_l$ \label{alg_client_sample}
	\item[] \deemph{\# Operations at each client }
	\FOR{each client $u  \in U_l$ }
	\FOR{each action $x\in \text{supp}(\pi_l)$} 
	\STATE Compute average local reward over $T_l(x)$ rounds: ${y}_l^u(x)=\frac{1}{T_l(x)}\sum_{t\in \cT_l(x)}(\langle\theta_u, x\rangle+\eta_{u,t})$ \label{alg_collect_data}
	\ENDFOR
	\STATE Let $\vec{y}_{l}^u = (y_l^u(x))_{x\in \text{supp}(\pi_l)}$
	\item[] \deemph{\# Apply the \textsc{Privatizer} $\cP = (\cR, \cS, \cA)$}
	\item[] \deemph{\# The local randomizer $\cR$ at each client:}
	\STATE Run the local randomizer $\mathcal{R}$ and send the output $\mathcal{R}(\Vec{y}_l^u)$ to $\cS$ \label{alg_randomizer} 			\ENDFOR
	\item[] \deemph{\# Computation $\cS$ at a trusted third party:}
	\STATE Run the computation function $\cS$ and send the output $\cS(\{\mathcal{R}(\Vec{y}_l^u)\}_{u\in U_l})$ to the analyzer $\cA$ \label{alg_shuffler}
	\item[] \deemph{\# The analyzer $\cA$ at the server:}
	\STATE Generate the privately aggregated statistics: $\Tilde{y}_l = \cA(\cS(\{\cR(\Vec{y}_l^u)\}_{u\in U_l}))$ 
	\label{alg_output_privatizer}
    \STATE Compute the following quantities: \label{alg_lse}
	$$\begin{cases}
		V_l = \sum_{x\in \text{supp}(\pi_l)} T_l(x)xx^\top \\
		G_l = \sum_{x\in \text{supp}(\pi_l)}   T_l(x) x \tilde{y}_l(x)\\	
		\Tilde{\theta}_l = V_l^{-1}G_l
	\end{cases}$$ 
	\STATE Find low-rewarding actions with confidence width $W_l$: \label{alg_line_elimination}%
	$$E_l = \left\{x\in \cD_l: \max_{b\in \cD_l} \langle \Tilde{\theta}_l, b-x\rangle > 2W_l\right\}$$ \label{alg_elimination}
	\STATE Update: $\cD_{l+1} = \cD_{l}\backslash E_l$, $h_{l+1} = 2h_l$, 
	$t_{l+1} = t_l+T_l$, and $l = l+1$ \label{alg_update}
	\ENDWHILE
	\end{algorithmic}
\end{algorithm}
	
In the following, we describe the detailed operations of DP-DPE. We begin by giving some necessary notations. Consider the $l$-th phase. Let $t_l$ and $T_l$ be the index of the starting round and the length of the $l$-th phase, respectively. Then, let $\cT_l\triangleq \{t\in [T]: t_l\leq t<t_l+T_l\}$ be the round indices in the $l$-th phase, let $\cT_l(x)\triangleq \{t\in \cT_l: x_t = x\}$ be the time indices in the $l$-th phase when action $x$ is selected, and let $\cD_l\subseteq \cD$ be the set of active actions in the $l$-th phase. 

\smallskip
\noindent \textbf{Action selection (Lines~\ref{alg_action_selection}-\ref{alg_action_slc_end}):} In the $l$-th phase, the action set $\cD_l$ consists of active actions that are possibly optimal. We compute a distribution $\pi_l(\cdot)$ over $\cD_l$ and choose actions according to $\pi_l(\cdot)$.	We briefly explain the intuition below. 	Let $V(\pi) \triangleq \sum_{x\in\cD} \pi(x)xx^{\top}$ and $g(\pi) \triangleq \max_{x\in \cD} \Vert x\Vert_{V(\pi)^{-1}}^2$.	According to the analysis in \cite[Chapter 21]{lattimore2020bandit}, if action $x\in \cD$ is played $\lceil h\pi(x)\rceil$ times (where $h$ is a positive constant), the estimation error associated with the action-related uncertainty for action $x$ is at most $\sqrt{2g(\pi)\log(1/\beta)/h}$ with probability $1-\beta$ for any $\beta\in (0,1)$. 
That is, for a fixed number of rounds, a distribution $\pi(\cdot)$ with a smaller value of $g(\pi)$ helps achieve a better estimation. Note that minimizing $g(\cdot)$ is a well-known \emph{$G$-optimal design} problem~\cite{pukelsheim2006optimal}. By the Kiefer-Wolfowitz Theorem 	\cite{kiefer1960equivalence}, one can find a distribution $\pi^*$ minimizing $g(\cdot)$ with $g(\pi^*)=d$, and the support set\footnote{The support set of a distribution $\pi$ over set $\cD$, denoted by $\text{supp}_{\cD}(\pi)$, is the subset of elements with a nonzero $\pi(\cdot)$, i.e., $\text{supp}_{\cD}(\pi) \triangleq \{x\in \cD: \pi(x) \neq 0\}$. We drop the subscript $\cD$ in $\text{supp}_{\cD}(\pi)$ for notational simplicity.\label{ft:support_set}} of $\pi^*$, denoted by $\text{supp}(\pi^*)$, has a size no greater than $d(d+1)/2$. In our problem, however, it suffices to  solve it near-optimally, i.e., finding a distribution $\pi_l$ such that $g(\pi_l)\leq 2d$ with $|\text{supp}(\pi_l)|\leq 4d\log\log d +16$ (Line~\ref{alg_action_selection}), which follows from \cite[Proposition 3.7]{lattimore2020learning}.  The near-$G$-optimal design reduces the complexity to $O(kd^2)$ while keeping the same order of regret.

\smallskip
\noindent \textbf{Clients sampling and private feedback aggregation  (Lines~\ref{alg_client_sample}-\ref{alg_output_privatizer}):} The central server randomly samples a subset $U_l$ of $\lceil2^{\alpha l}\rceil$ users (called clients) from the population $\cU$ to participate in the global bandit learning (Line~\ref{alg_client_sample}).	Each sampled client $u \in U_l$ collects their local reward observations of each chosen action $x\in \text{supp}(\pi_l)$ by the server and computes the average $y_l^u(x)$ as feedback (Line~\ref{alg_collect_data}).	Before being used to estimate the global parameter by the central server, these feedback $\Vec{y}_l^u\triangleq (y_l^u(x))_{x\in\text{supp}(\pi_l)}\in \R^{|\text{supp}(\pi_l)|}$ are processed by a \textsc{Privatizer}~$\cP$ to ensure differential privacy. Recall that a \textsc{Privatizer}~$\cP = (\cR, \cS, \cA)$ is a process completed by the clients, the server, and/or a trusted third party. In particular, according to the privacy requirement under different DP models, the \textsc{Privatizer} $\cP$ enjoys flexible instantiations (see detailed discussions in Section~\ref{sec:DP-instantiation}).	Generally, a \textsc{Privatizer} works in the following manner: each client $u$ runs the randomizer $\cR$ on its local average reward $\Vec{y}_l^u$ (over $T_l$ pulls) and then sends the resulting (potentially private) messages $\cR(\Vec{y}_l^u)$ to $\cS$ (Line~\ref{alg_randomizer}). The computation function in $\cS$ operates on these messages and then sends results	$\cS(\{\cR(\Vec{y}_l^u)\}_{u\in U_l})$ to the analyzer $\cA$ at the central server (Line~\ref{alg_shuffler}).  Finally, the analyzer $\cA$ aggregates received messages (potentially in a privacy-preserving manner) and outputs a private averaged local reward $\Tilde{y}_{l}(x)$ (over participating clients $U_l$) for each action $x\in \text{supp}(\pi_l)$ (Line~\ref{alg_output_privatizer}).	We provide the rigorous formulation of different DP models for \textsc{Privatizer} $\cP$ in Section~\ref{sec:DP-instantiation},
with corresponding detailed instantiations of $\cR, \cS$, and $\cA$.

\smallskip

\noindent \textbf{Parameter estimation and action elimination {(Lines~\ref{alg_lse}-\ref{alg_update})}:}	Using privately aggregated feedback (i.e., the private averaged local reward $\tilde{y}_l$  of the chosen actions $x\in \text{supp}(\pi_l)$), the central server computes the least-square estimator
$\Tilde{\theta}_l$	(Line~\ref{alg_lse}). 	Action elimination is based on the following confidence width:
\begin{equation}
	\begin{aligned}
	W_l\triangleq &\left(\underbrace{\sqrt{\frac{2d}{|U_l|h_l}}}_{\text{action-related}} + \underbrace{\frac{\sigma}{\sqrt{|U_l|}}}_{\text{client-related}} +\underbrace{\sigma_{n}}_{\text{privacy noise}}\right) \sqrt{2\log \left(\frac{1}{\beta}\right)}, \label{eq:W_l_DPE}
	\end{aligned}
\end{equation}
where $\sigma$ is the standard variance associated with client sampling, $\sigma_n$ is related to the privacy noise determined by the DP model, and $\beta$ is the confidence level. We choose this confidence width based on the concentration inequality for sub-Gaussian variables. Specifically, the three terms in Eq.~\eqref{eq:W_l_DPE} capture the action-related uncertainty, client-related uncertainty, and the added noise for privacy guarantees, respectively. This privacy noise $\sigma_n$ depends on the adopted DP model.	Using this confidence width $W_l$ and the estimated global model parameter $\Tilde{\theta}_l$, we can identify a subset of suboptimal actions $E_l$ with high probability (Line~\ref{alg_elimination}). At the end of the $l$-th phase, we update the set of active actions $\cD_{l+1}$ by eliminating $E_l$ from $\cD_l$ and double $h_l$ (Line~\ref{alg_update}). 

Finally, we make two remarks about the DP-DPE algorithm. 
\begin{remark}
While a finite number of actions is assumed in this paper, one could extend it to the case with an infinite number of actions by using the covering argument \cite[Lemma~20.1]{lattimore2020bandit}. Specifically, when the action set $\cD\subseteq \R^d$ is infinite, we can replace $\cD$ with a finite set $\cD_{\epsilon_0}\subseteq \R^d$ with $|\cD_{\epsilon_0}|\leq (3/\epsilon_0)^d$ such that for all $x\in \cD$, there exists an $x'\in \cD_{\epsilon_0}$ with $\Vert x-x'\Vert_2 \leq \epsilon_0$.
\end{remark}

\begin{remark}\label{rmk:spans}
In Algorithm~\ref{alg:dpe}, we assume that $\cD_l$ spans $\mathbb{R}^d$ such that matrices $V(\pi_l)$ and $V_l$ are invertible. Then, one could find the near optimal design $\pi_l(\cdot)$ (Line~\ref{alg_action_selection}) and compute the least-square estimator $\Tilde{\theta}_l$ (Line~\ref{alg_lse}). When $\cD_l$ does not span $\R^d$, one can simply work in the smaller space span$(\cD_l)$~ \cite{lattimore2020learning}.
\end{remark}

\section{DP-DPE under Different DP Models}\label{sec:DP-instantiation}
As alluded before, one of the key features of our general algorithmic framework DP-DPE is that it enables us to consider different trust models in DP (i.e., who the user can trust with her sensitive data) in a unified way by instantiating different mechanisms for the  \textsc{Privatizer}. 
In this section, we formalize DP models integrated with our DP-DLB formulation and provide concrete instantiations for the \textsc{Privatizer} $\cP = (\cR,\cS,\cA)$ in DP-DPE according to three representative DP trust models: the central, local, and shuffle models.

\subsection{DP-DPE under the Central DP Model}
	
In the central DP model, we assume that each client trusts the server, and hence, the server can collect clients' raw data (i.e., the local reward ${y}_l^u(x)$ for each chosen action $x$ in our case). The privacy guarantee is that any adversary with arbitrary auxiliary information cannot infer a particular client's data by observing the outputs of the server. To achieve this privacy protection, the central DP model requires that the outputs of the server on two neighboring datasets differing in only one client are indistinguishable~\cite{dwork2006calibrating}. To present the formal definition in our case, recall that the DP-DPE algorithm (Algorithm~\ref{alg:dpe}) runs in phases, and in each phase $l$, a set of new clients $U_l$ will participate in the global bandit learning by providing their feedback. Let\footnote{We use the superscript $^*$ to indicate that the length could be varying.}
$\cU_T \triangleq (U_l)_{l=1}^L \in \mathcal{U}^*$ be the sequence of all the participating clients in the total $L$ phases ($T$ rounds).
We use $\mathcal{M}(\mathcal{U}_T)=(x_1, \dots, x_T) \in \cD^{T}$ to denote the sequence of actions chosen in $T$ rounds by the central server.
Intuitively, we are interested in a randomized algorithm such that the output $\mathcal{M}(\mathcal{U}_T)$ does not reveal ``much'' information about any particular client $u \in \cU_T$. 
Formally, we have the following definition.
\begin{definition} (Differential Privacy (DP)). For any $\epsilon \geq 0$ and $\delta\in [0,1]$, a DP-DPE instantiation is $(\epsilon, \delta)$-differentially private (or $(\epsilon, \delta)$-DP) if for every $\mathcal{U}_T, \mathcal{U}_T^{\prime} \subseteq \mathcal{U}$ differing on a single client and for any subset of actions $Z \subseteq \cD^T$,
\begin{equation}
	\mathbb{P}[\mathcal{M}(\mathcal{U}_T)\in Z] \leq e^{\epsilon} \mathbb{P}[\mathcal{M}(\mathcal{U}_T')\in Z] + \delta.
	\end{equation}
\end{definition}
According to the post-processing property of DP (cf. Proposition~2.1 in \cite{dwork2014algorithmic}) and parallel-composition (thanks to the uniqueness of client sampling), it suffices to guarantee that the final analyzer  $\mathcal{A}$ in $\mathcal{P}$ is $(\epsilon, \delta)$-DP. That is, for any phase $l$, the \textsc{Privatizer} $\cP$ is $(\epsilon, \delta)$-DP if the following is satisfied for any pair of $U_l, U_l^{\prime}\subseteq \cU$ that differ by at most one client and for any output $\tilde{y}$ of $\cA$:
\begin{equation*}
	\mathbb{P}[\cA(\{\Vec{y}_{l}^u\}_{u\in U_l})= \tilde{y}] \leq e^{\epsilon} \cdot \mathbb{P}[\cA(\{\Vec{y}_{l}^u\}_{u\in U^{\prime}_l})= \tilde{y}] + \delta.
\end{equation*}
To achieve this, we resort to standard Gaussian mechanism at the server where $\cA$ computes the average of local rewards for each chosen action to guarantee $(\epsilon,\delta)$-DP. 
Specifically, in each phase $l$, the participating clients send their average local rewards $\{\Vec{y}_{l}^u\}_{u\in U_l}$  directly to the central server, and the central server adds Gaussian noise to the average local feedback (over clients) 	before estimating the global parameter and deciding the chosen actions in the next phase. That is, in the central DP model, both $\cR$ and $\cS$ of the \textsc{Privatizer} $\cP$ are identity mapping while $\cA$ adds Gaussian noise when computing the average. In this case, $\cP=\cA$, and the private aggregated feedback for the chosen actions in the $l$-th phase can be represented as \begin{equation}
\Tilde{y}_l = 
	\cP\left(\{\Vec{y}_{l}^u\}_{u\in U_l}\right)= 
	\cA\left(\{\Vec{y}_{l}^u\}_{u\in U_l}\right)
	= \frac{1}{|U_l|}\sum_{u\in U_l} \Vec{y}_{l}^u+(\gamma_1, \dots, \gamma_{s_l}),  \label{eq:privatizer_cdp_app} 
\end{equation}
where $s_l\triangleq |\text{supp}(\pi_l)|$,  $\gamma_j \overset{\emph{i.i.d.}}{\sim} \mathcal{N}(0,\sigma^2_{nc})$, and the variance $\sigma^2_{nc}$ is based on the $\ell_2$ sensitivity of the average $\frac{1}{|U_l|}\sum_{u\in U_l}\Vec{y}_{l}^u$. In the rest of the paper, we will continue to use $s_l$ instead of $|\text{supp}(\pi_l)|$ to denote the number of actions chosen in the $l$-th phase for notational simplicity, and it is also the dimension of $\Vec{y}_l^u$ for all $u$. 
	
With the above definition, we present the privacy guarantee of DP-DPE in the central DP model in Theorem~\ref{thm:cdp}. 
\begin{theorem} \label{thm:cdp}
	The DP-DPE instantiation using the \textsc{Privatizer} in Eq.~\eqref{eq:privatizer_cdp_app} with $\sigma_{nc} = \frac{2B\sqrt{2s_l\ln(1.25/\delta)}}{\epsilon |U_l|} 
	$ guarantees $(\epsilon, \delta)$-DP.
\end{theorem}
The relatively high trust model in the central DP is not always feasible in practice since some clients do not trust the server and are not willing to share any of their sensitive data. This motivates the introduction of a strictly stronger notion of privacy protection called the local DP~\cite{kasiviswanathan2011can}, which is the main focus of the next subsection.

\subsection{DP-DPE under the Local DP Model}
In the local DP model, the privacy burden is now at each client's local side, in the sense that any data sent by any client must already be private. In other words, even though an adversary can observe the data communicated from a client to the server, the adversary cannot infer any sensitive information about the client. Mathematically, this requires a local randomizer $\cR$ at each user's side to generate approximately indistinguishable outputs on any two different data inputs.  In particular, let $Y_u$ be the set of all possible values of the average local reward $\Vec{y}_l^u$ for client $u$. Then, we have the following formal definition.
	
\begin{definition} (Local Differential Privacy (LDP)). For any $\epsilon \geq 0$ and $\delta\in [0,1]$, a DP-DPE instantiation is $(\epsilon, \delta)$-local differentially private (or $(\epsilon, \delta)$-LDP) if for any client $u$, every two datasets $\Vec{y}, \Vec{y}^{\prime}  \in Y_u$ satisfies
\begin{equation}
	\mathbb{P}[\cR(\Vec{y}) = o ] \leq e^{\epsilon} \mathbb{P}[\cR(\Vec{y}^{\prime} ) = o] + \delta,
	\end{equation}
	for every possible output $ o \in \{\cR(\Vec{y})|\Vec{y} \in Y_u\}$.
\end{definition}
That is, an instantiation of DP-DPE is $(\epsilon, \delta)$-LDP if the local randomizer $\cR$ in $\cP$ is $(\epsilon, \delta)$-DP.
To this end, the randomizer $\cR$ at each client employs a Gaussian mechanism, the shuffler $\cS$ is a simple identity mapping, and the analyzer $\cA$ at the server side conducts a simple averaging. Then, the overall output of the \textsc{Privatizer} is the following:
\begin{equation}
	\Tilde{y}_l = 
	\frac{1}{|U_l|}\sum_{u\in U_l} \cR(\Vec{y}_{l}^u) = \frac{1}{|U_l|}\sum_{u\in U_l} \left(\Vec{y}_{l}^u+(\gamma_{u,1}, \dots, \gamma_{u,s_l})\right), \label{eq:privatizer_ldp_app}
\end{equation}
where $\gamma_{u,j} \overset{\text{\emph{i.i.d.}}}{\sim}  \mathcal{N}(0,\sigma^2_{nl})$, 
and the variance $\sigma^2_{nl}$ is based on the sensitivity of $\Vec{y}_{l}^u$.
	
With the above definition, we present the privacy guarantee of DP-DPE in the local DP model in Theorem~\ref{thm:ldp}.
\begin{theorem} \label{thm:ldp}
The DP-DPE instantiation using the \textsc{Privatizer} in Eq.~(\ref{eq:privatizer_ldp_app}) with $\sigma_{nl} = 
\frac{2B\sqrt{2s_l\ln(1.25/\delta)}}{\epsilon}$ guarantees $(\epsilon, \delta)$-LDP.
\end{theorem}
	
Although the local DP model offers a stronger privacy guarantee compared to the central DP model, it often comes at a price of the regret performance. As we will see, the regret performance of DP-DPE under the local DP model is much worse than that under the central DP model. Therefore, a fundamental question is whether there is a \textsc{Privatizer} for DP-DPE that can achieve the same regret as in the central DP \textsc{Privatizer} while assuming similar trust 	model as in the local DP  \textsc{Privatizer}. This motivates us to consider a recently proposed \emph{shuffle DP model}~\cite{cheu2019distributed,erlingsson2019amplification}, which is the main focus of the next subsection.

\subsection{DP-DPE under the Shuffle DP Model}
In the shuffle DP model, between the clients and the server, there exists a shuffler that permutes a batch of clients’ randomized data before they are observed by the server so that the server cannot distinguish between two clients’ data. Thus, an additional layer of randomness is introduced via shuffling, which can often be easily implemented using	cryptographic primitives (e.g., mixnets) due to its simple
operation~\cite{bittau2017prochlo}. Due to this, the clients now tend to trust the shuffler but still do not trust the central server as in the local DP model. This new trust model offers a possibility to achieve a better regret-privacy tradeoff. This is because the additional randomness of the shuffler creates a \emph{privacy blanket} so that by adding much less random noise, each client can now hide her information in the crowd, i.e., privacy amplification by shuffling~\cite{garcelon2021privacy}. 
	
Formally, a standard one-round shuffle protocol consists of all the three parts: a (local) randomizer $\cR$, a shuffler $\cS$, and an analyzer $\cA$. In this protocol, the clients trust the shuffler but not the analyzer. 
Hence, the privacy objective is to ensure that the outputs of the shuffler on two neighboring datasets are indistinguishable from the analyzer's point of view. Note that each client still does not send her raw data to the shuffler even though she trusts it. Due to this, a shuffle protocol often also offers a certain level of LDP guarantee.
	
In our case, the online learning procedure will proceed in multiple phases rather than a simple one-round computation. Thus, we need to guarantee that all the shuffled outputs are  indistinguishable. To this end, we define the (composite) mechanism $\cM_s(\cU_T)\triangleq ((\cS \circ \cR)(U_1), (\cS \circ \cR)(U_2),\ldots, (\cS \circ \cR)(U_L))$, where $(\cS\circ \cR)(U_l)\triangleq \cS (\{\cR(\Vec{y}_l^{u}) \}_{u\in U_l})$. We say a DP-DPE instantiation satisfies the shuffle differential privacy (SDP) if the composite mechanism $\cM_s$ is DP, which leads to the following formal definition.
\begin{definition} (Shuffle Differential Privacy (SDP)). For any $\epsilon \geq 0$ and $\delta\in [0,1]$, a DP-DPE instantiation is $(\epsilon, \delta)$-shuffle differential privacy (or $(\epsilon, \delta)$-SDP) if for any pair $\mathcal{U}_T$ and $\mathcal{U}^{'}_T$ that differ by one client, the following is satisfied for all $Z\subseteq Range(\cM_s)$:
\begin{equation}
	\mathbb{P}[\cM_s(\mathcal{U}_T)\in Z] \leq e^{\epsilon} \mathbb{P}[\cM_s(\mathcal{U}^{'}_T)\in Z] + \delta.
	\end{equation}
\end{definition}
Then, consider any phase $l$. Formally, the \textsc{Privatizer} $\cP$ is $(\epsilon, \delta)$-SDP if the following is satisfied for any pair of $U_l$, $U_l^{\prime} \subseteq \cU$ that differ by one client and for any possible output $z$ of $\cS\circ \cR$: 
\begin{equation*}
\mathbb{P}[(\cS\circ \cR)({U_l}) = z] \leq e^{\epsilon} \cdot \mathbb{P}[(\cS\circ \cR)({U_l^{\prime})}= z] + \delta.
\end{equation*} 
We present the concrete pseudocode of $\cR$, $\cS$, and $\cA$ for the shuffle DP model \textsc{Privatizer} $\cP$ in Algorithm~\ref{alg:shuffle_vec} (see Appendix~\ref{app:proof_privacy}), which builds on the vector summation protocol recently proposed in~\cite{cheu2021shuffle}.
Here, we provide a brief description of the process. Essentially, the noise added in the shuffle model \textsc{Privatizer} relies on the upper bound of $\ell_2$ norm of the input vectors. However, each component operates on each coordinate of the input vectors independently. Recall that the input of the shuffle model \textsc{Privatizer} is $\{\Vec{y}_l^u\}_{u\in U_l}$ and that each chosen action $x$ corresponds to a coordinate in the $s_l$-dimentional vector.
Consider the coordinate $j_x$ corresponding to action $x$, and the entry $y_l^u(x)$ at client $u$. 	
First, the local randomizer $\cR$ encodes the input $y_l^u(x)$ via a fixed-point encoding scheme \cite{cheu2019distributed} and ensures privacy by  injecting binomial noise. Specifically, given any scalar $w\in [0,1]$, it is first encoded as $\hat{w}=\bar{w}+\gamma_1$ using an accuracy parameter $g\in \mathbb{N}$, where $\bar{w}=\lfloor wg\rfloor$ and $\gamma_1 \sim \texttt{Ber}(wg- \bar{w})$  is a Bernoulli random variable. Then, a binomial noise $\gamma_2 \sim \texttt{Bin}(b,p)$ is generated, where $b\in \mathbb{N}$ and $p\in (0,1)$ controls the level of the privacy noise. The output of the local randomizer for each coordinate is simply a collection of $g+b$ bits, where $\hat{w}+\gamma_2$  bits are $1$'s and the rest are $0$'s. Combining these $g+b$ bits for each coordinate $j_x$ for $x\in \text{supp}(\pi_l)$ yields the final outputs of the local randomizer $\cR$ for the vector $\Vec{y}_l^u$. 
Note that the output bits for each coordinate are marked with the coordinate index so that they will not be mixed up in the following procedures. After receiving the bits from all participating clients, the shuffler $\cS$ simply permutes these bits  uniformly at random  and sends the output to the analyzer $\cA$ at the central server. The analyzer $\cA$ adds the received bits, removes the bias introduced by encoding and binomial noise (through simple shifting operations), and divides the result by $|U_l|$ for each coordinate. Finally, the analyzer $\cA$ outputs a random $s_l$-dimensional vector $\Tilde{y}_l$, whose expectation is the average of the input vectors. That is,  $\E[\Tilde{y}_l] = \frac{1}{|U_l|}\sum_{u\in U_l} \Vec{y}_l^u$ (which is proven in Appendix~\ref{app:shuffle_model}). In the shuffle model \textsc{Privatizer}, the three parameters $g$, $b$, and $p$ need to be properly chosen according to the privacy requirement. 
Then, the final privately aggregated data is the following:
\begin{equation}
	\Tilde{y}_l =  \mathcal{P}\left(\{\Vec{y}_l^{u}\}_{u\in U_l}\right) 
	=\cA(\mathcal{S} (\{\cR(\Vec{y}_l^{u})\}_{u\in U_l})).
\end{equation}

With the above definition, we present the privacy guarantee of DP-DPE in the shuffle DP model in Theorem~\ref{thm:sdp}.
\begin{theorem}\label{thm:sdp}
For any $\epsilon\in (0,15)$ and $\delta\in (0,1/2)$, the DP-DPE instantiation using the \textsc{Privatizer} specified in Algorithm~\ref{alg:shuffle_vec} guarantees $(\epsilon, \delta)$-SDP. 
\end{theorem}

\section{Main Results} \label{sec:analysis}
In this section, we study the performance of DP-DPE under different DP models in terms of regret and communication cost.
We start with the non-private DP-DPE algorithm (called DPE, with $\Tilde{y}_l = \frac{1}{|U_l|}\sum_{u\in U_l} \Vec{y}_{l}^u$ and $\sigma_{n}=0$ for all $l$) and present the main result in Theorem~\ref{thm:regret_approx}.
\begin{theorem}[DPE]\label{thm:regret_approx}
	Let $\beta=1/(kT)$ and $\sigma_n = 0$ in Algorithm~\ref{alg:dpe}. Then, the non-private DP-DPE algorithm achieves the following expected regret:
    \begin{equation}
	    \begin{aligned}
		\E[R(T)] = & O(\sqrt{dT\log (kT)}) +  O\left(\sigma T^{1-\alpha/2}\sqrt{\log(kT)}\right), 
		  \label{eq:regret_bound}
    \end{aligned}
	\end{equation}
	with a communication cost of $O(dT^{\alpha})$.
\end{theorem}
We present a proof sketch below and provide the detailed proof in Appendix~\ref{app:proof_of_regret}. 

\begin{proof}[Proof sketch]
We begin by considering a concentration inequality $P\left\{ \langle \Tilde{\theta}_l-\theta^*, x\rangle \geq W_l\right\}\leq 2\beta$, which indicates that in the $l$-th phase, the estimation error for the global reward of each action is bound by $W_l$ \emph{w.h.p.} Then, we show that the optimal action stays in the active set the whole time \emph{w.h.p.} and that the regret incurred by one pull is bounded by $4W_{l-1}$ in the $l$-th phase. Finally, summing up the regret over rounds in all phases, we derive the regret upper bound. The analysis of the communication cost is quite straightforward. In the $l$-th phase, only local average reward of each chosen action in this phase is communicated. Since the number of chosen actions is bounded by $(4d\log\log d+ 16)$ according to the near-$G$-optimal design~\cite[Proposition 3.7]{lattimore2020learning}, the communication cost is proportional to the total number of clients involved in the entire learning process. 
\end{proof}

\begin{remark}
	Theorem~\ref{thm:regret_approx} gives a problem-independent regret upper bound for DPE. We can observe an obvious tradeoff between regret and communication cost, captured by the value of $\alpha$. While a larger $\alpha$ leads to a smaller regret, it also incurs a larger communication cost. Setting $\alpha=2/3$ gives $O(T^{2/3})$ for both regret and communication cost. 
\end{remark}

\High{
\begin{remark}[(Sub-)optimality]\label{rk:sub-optimality} 
Note that one natural lower bound for our setting is $\Omega(\sqrt{dT})$, the one for the standard linear bandits with finite arms  \cite{lattimore2020bandit}, where there is no client-related uncertainty (i.e., $\sigma=0$). In this setting, 
the upper bound derived in Eq.~\eqref{eq:regret_bound} matches the existing lower bound up to a logarithmic term. As to the general case with $\sigma>0$, we can still see the (near)-optimality of our upper bound for the case with user-sampling parameter $\alpha>1$. When sampling fewer users with $\alpha\in(0,1)$, the second term of the regret upper bound in Eq.~\eqref{eq:regret_bound} that relies on $\alpha$ becomes dominant and cannot be ignored. However, the aforementioned lower bound $\Omega(\sqrt{dT})$ is derived under the standard linear bandit setting, which is irrelevant to the user sampling parameter $\alpha$. Therefore, we leave it as our future work to close this gap between this natural lower bound and the derived ($\alpha$-dependent) upper bound in Eq.~\eqref{eq:regret_bound}.  
\end{remark}
}
In the following, we present the performance of DP-DPE in terms of regret and communication cost under different DP models, i.e., instantiated with different \textsc{Privatizer}s introduced in Section~\ref{sec:DP-instantiation}. We use CDP-DPE, LDP-DPE, and SDP-DPE to denote the DP-DPE algorithm in the central, local, and shuffle DP models, respectively.
Let $S\triangleq 4d\log\log d+16$ denote the upper bound of $|\text{supp}(\pi_l)|$ in every phase $l$. 

\begin{theorem}[CDP-DPE] \label{thm:regret_cdp}
Consider the Gaussian mechanism with $\sigma_{nc} = \frac{2B\sqrt{2s_l\ln(1.25/\delta)}}{\epsilon | U_l|}$ in the central DP model. With $\sigma_n = 2\sigma_{nc}\sqrt{Sd}$ and $\beta=1/(kT)$, CDP-DPE achieves the following expected regret:
\begin{equation}
\begin{aligned}
    \E[R(T)] =& O(\sqrt{dT\log(kT)}) + O(\sigma T^{1-\alpha/2}\sqrt{\log (kT)}) \\
    &+ O\left(\frac{Bd^{3/2}T^{1-\alpha}\sqrt{\ln(1/\delta)\log(kT)}}{\epsilon}\right) 
    +O\left(\frac{Bd^{5/2}\sqrt{\ln(1/
    \delta)\log(kT)}}{\epsilon}\right),
\end{aligned}
\end{equation}
with a communication cost of $O(dT^{\alpha})$.
\end{theorem}

\begin{theorem}[LDP-DPE]  \label{thm:regret_ldp}
Consider the Gaussian mechanism with $\sigma_{nl} = \frac{2B\sqrt{2s_l\ln(1.25/\delta)}}{\epsilon }$ in the local DP model. With $\sigma_n = 2\sigma_{nl}\sqrt{\frac{Sd}{|U_l|}}$ in the $l$-th phase and $\beta=1/(kT)$,  LDP-DPE achieves the following expected regret:
\begin{equation}
\begin{aligned}
    \E[R(T)] = & O(\sqrt{dT\log (kT)}) +  O\left(\sigma T^{1-\alpha/2}\sqrt{\log(kT)}\right) \\
    & + O\left(\frac{Bd^{3/2}T^{1-\alpha/2}\sqrt{\ln(1/\delta)\log(kT)}}{\epsilon} \right)+O\left(\frac{Bd^{5/2}\sqrt{\ln(1/\delta)\log(kT)}}{\epsilon}\right),
\end{aligned}
\end{equation}
with a communication cost of $O(dT^{\alpha})$.
\end{theorem}

\begin{theorem}[SDP-DPE]  \label{thm:regret_sdp}
With $\sigma_n = 2\sigma_{ns}\sqrt{Sd} =O\left(\frac{B\sqrt{Sds_l}\log{(s_l/\delta)}}{\epsilon |U_l|}\right)$ in the $l$-th phase and $\beta=1/(kT)$, SDP-DPE achieves the following expected regret:
\begin{equation}
\begin{aligned}
    \E[R(T)] =& O(\sqrt{dT\log (kT)}) +  O\left(\sigma T^{1-\alpha/2}\sqrt{\log(kT)}\right) \\ 
    &+ O\left(\frac{Bd^{3/2}\ln(d/\delta)T^{1-\alpha}\sqrt{\log(kT)}}{\epsilon}\right) +O\left(\frac{Bd^{5/2}\ln(d/
    \delta) \sqrt{\log(kT)}}{\epsilon}\right),
\end{aligned}
\end{equation}
and the communication cost is $O(dT^{3\alpha/2})$ bits. 
\end{theorem}

For all Theorems~\ref{thm:regret_cdp}, \ref{thm:regret_ldp}, and \ref{thm:regret_sdp}, the first term is from action-related uncertainty, the second term is due to client-related uncertainty, and the third and forth terms are introduced by privacy guarantee.
We provide the detailed proofs of Theorems~\ref{thm:regret_cdp}, \ref{thm:regret_ldp}, and \ref{thm:regret_sdp} in Appendix~\ref{app:proof_of_regret} and make the following remarks. 

\begin{remark} [Privacy ``for-free''] \label{rmk:free_privacy}
	Comparing the above results with Theorem~\ref{thm:regret_approx} for the non-private case, we observe that the DP-DPE algorithm enables us to achieve privacy guarantees ``for free" in the central and shuffle DP models, in the sense that the additional regret due to privacy protection is only a lower-order additive term. Essentially, this is because the uncertainty introduced by privacy noise is dominated by the client-related uncertainty, which can be captured by our carefully designed confidence width $W_l$ in Eq.~(\ref{eq:W_l_DPE}) and our choice of $\sigma_n$ for different \textsc{Privatizer}s. See more discussions on achieving privacy ``for-free'' in Section~\ref{sec:forfree}.
\end{remark}

\begin{remark}[\High{Regret-privacy tradeoff}]	\label{rmk:tradeoff_shuffle}
\High{Consider the regret due to privacy protection by comparing the regret performance column in Table~1 of all the DP-DPE algorithms. We can see an additional term in regret performance associated with each DP-DPE algorithm. Specifically, while the local DP model ensures a stronger privacy guarantee compared to the central DP model, it introduces an additional regret of $\Tilde{O}(T^{1-\alpha/2})$  compared to $\Tilde{O}(T^{1-\alpha})$ in the central DP model.} 
The shuffle DP model, however, leads to a much better tradeoff between regret and privacy, achieving nearly the same regret guarantee as the central DP model, yet assuming a similar trust model to the local DP model (i.e., without a trustworthy central server).
\end{remark}

\begin{remark}[Communication cost]\label{rmk:communication_cost}
	Both CDP-DPE and LDP-DPE consume the same amount of communication resources as the non-private DP-DPE algorithm, measured by the number of real numbers \cite{wang2019distributed}. In contrast, SDP-DPE relies only on binary feedback from the clients, and thus, the communication cost is measured by the number of bits. It is worth noting that sending messages consisting of real numbers could be difficult in practice on finite computers~\cite{canonne2020discrete, kairouz2021distributed}, and hence in this case, it is desirable to use SDP-DPE, which incurs a communication cost of $O(dT^{3\alpha/2})$ bits.  
\end{remark}
\High{
\begin{remark}[Pure DP extension]\label{rmk:pure-dp} While
we use the Gaussian mechanism to ensure approximate DP (i.e., $(\epsilon, \delta)$-DP), we claim that our proposed scheme in this paper can be effectively integrated with the Laplace mechanism, which ensures a pure DP and achieves nearly the same regret performance. We provide how to modify the algorithm and derive the theoretical results for the Laplace mechanism in Appendix~\ref{app:laplace}.
\end{remark}
}
\section{Numerical Results}\label{sec:simulation}
\begin{figure*}[t!] 
	\subfigure[]{
	\label{fig:impact_of_epsilon}
	\includegraphics[width=0.3\textwidth]{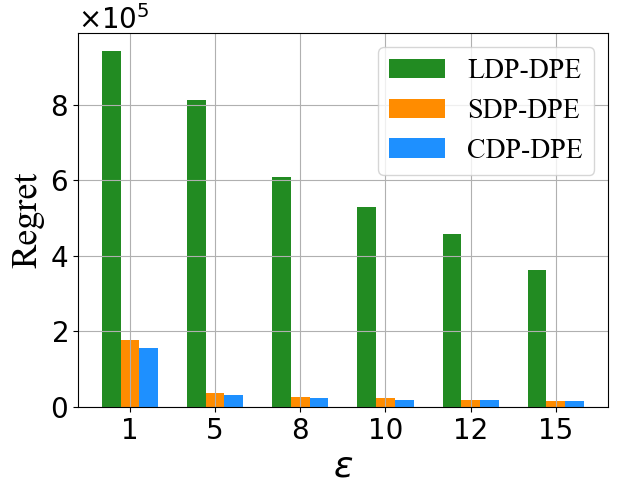}}
	\quad
	\subfigure[]{
	\label{fig:dp_performance}
	\includegraphics[width=0.3\textwidth]{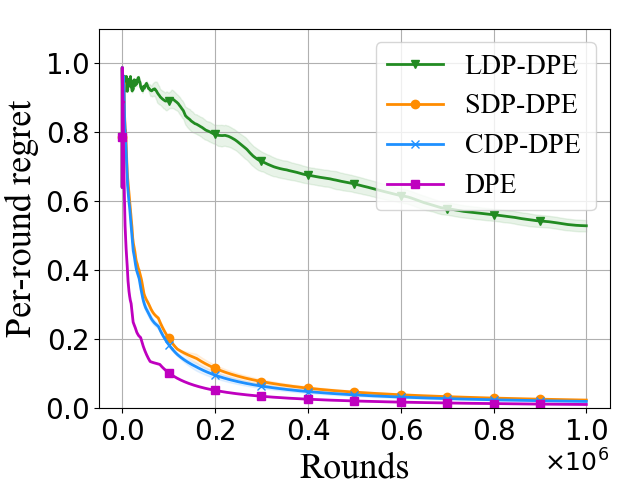}}
	\quad
	\subfigure[]{
	\label{fig:non_private}
	\includegraphics[width=0.3\textwidth]{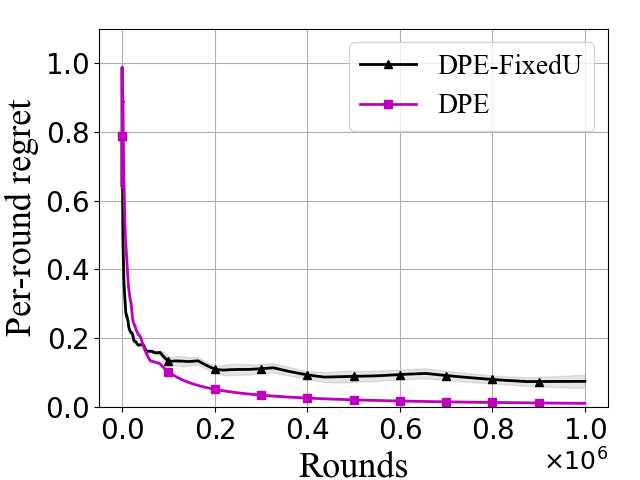}}
	\caption{Performance evaluation of DP-DPE. The shaded area indicates the standard deviation. (a) Final cumulative regret vs. the privacy budget $\epsilon$. (b) Per-round regret vs. time with privacy parameters $\epsilon=10$ and $\delta=0.1$. (c) Comparison between two non-private algorithms. \High{Here, we choose the number of clients in DPE-FixedU to be $U=97$ based on the calculation.}
 }
\end{figure*}

In this section, we conduct simulations to evaluate the performance of DP-DPE. The detailed setting of our simulations is as follows: $d=20, k=10^3, \sigma=0.1, |\cU|=10^5, \alpha=0.8$, and $T=10^6$. We perform $20$ independent runs for each set of simulations.

First, we study the regret performance of DP-DPE under different DP models. Recall that we use CDP-DPE, LDP-DPE, and SDP-DPE to denote DP-DPE in the central, local, and shuffle DP models, respectively. In Fig.~\ref{fig:impact_of_epsilon}, we present the cumulative regret at the end of $T$ rounds for the three algorithms under different values of privacy budget $\epsilon$. We can observe an obvious tradeoff between the privacy budget and the regret performance for all the DP models: the cumulative regret decreases as the privacy requirement becomes less stringent (i.e., a larger $\epsilon$). In addition, it also reflects the regret-privacy tradeoff across different DP models. That is, with the same privacy budget $\epsilon$, 
while LDP-DPE has the largest regret yet without requiring the clients to trust anyone else (neither the server nor a third party), CDP-DPE achieves the smallest regret but relies on the assumption that the clients trust the server. Interestingly, SDP-DPE achieves a regret fairly close to that of CDP-DPE, yet without the need to trust the server. This is well aligned with our theoretical results that SDP-DPE achieves a better regret-privacy tradeoff.

In addition, we are also interested in the regret loss due to privacy protection and how efficiently DP-DPE performs the global bandit learning. Fix the privacy parameters $\epsilon=10$ and $\delta=0.25$. In Fig.~\ref{fig:dp_performance}, we plot how the per-round regret of the three algorithms (i.e., CDP-DPE, LDP-DPE, and SDP-DPE) varies over time compared to the non-private DP-DPE algorithm (i.e., DPE). We observe that LDP-DPE incurs the largest regret while ensuring the strongest privacy guarantee (i.e., $(\epsilon,\delta)$-LDP). On the other hand, the regret performance of CDP-DPE and SDP-DPE is very close to that of DPE (that does not ensure any privacy guarantee), under the assumption of a trusted central server and a trusted third party shuffler, respectively. This observation, along with our theoretical results, shows that DP-DPE can indeed achieve privacy ``for-free'' under the central and shuffle models.

\begin{figure}[t!] 
\centering
	\includegraphics[width=0.4\textwidth]{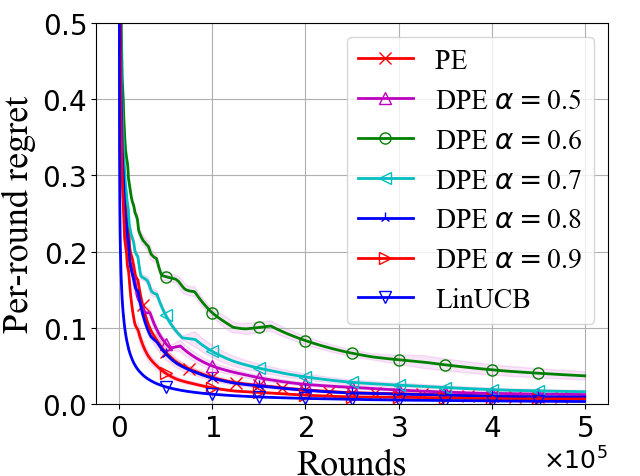}
	\caption{LinUCB vs. PE vs DPE with different values of $\alpha$.
	}
	\label{fig:Alg_compare}  
\end{figure}

\begin{table*}[!t]
\caption{Comparison of communication cost under LinUCB and PE with different values of $\alpha$.}
\label{tab:vs_LinUCB_PE}
\renewcommand\arraystretch{0.8}
\centering
	\scalebox{0.9}{
\begin{tabular}{c|c|c|c|c|c|c|c} 
    \toprule
    \multirow{2}*{Algorithms}  & \texttt{DPE} & \texttt{DPE} & \texttt{DPE} & \texttt{DPE} & \texttt{DPE} & \multirow{2}*{\texttt{LinUCB}} & \multirow{2}*{\texttt{PE}}\\ 
    ~  & $\alpha=0.5$& $\alpha=0.6$& $\alpha=0.7$& $\alpha=0.8$ & $\alpha=0.9$&~&~ \\
    \hline
     \multirow{2}*{Communication cost ($\times 10^4$)}
      & \multirow{2}*{$0.70$} & \multirow{2}*{$0.81$} & \multirow{2}*{$1.05$} & \multirow{2}*{$1.69$}  &
     \multirow{2}*{$3.27$}&
     \multirow{2}*{$5.00$} & \multirow{2}*{$5.00$}\\
   & ~ & ~   & ~& ~ & ~& ~ & ~\\
   \hline
   \multirow{2}*{\# of participating users ($\times 10^4$)}  &
     \multirow{2}*{$0.04$}   & \multirow{2}*{$0.10$} & \multirow{2}*{$0.23$} & \multirow{2}*{$0.55$} & \multirow{2}*{$1.34$}  &
     \multirow{2}*{$5.00$}&
     \multirow{2}*{$5.00$} \\
   & ~ & ~   & ~& ~ & ~& ~ & ~ \\
      \bottomrule
\end{tabular}
}
\end{table*}
Regarding the communication efficiency of our proposed algorithm, we also show that the exponentially-increasing client-sampling plays a key role in balancing the regret and the communication cost. To this end, we compare DPE 
with another non-private algorithm, called DPE-FixedU in Fig.~\ref{fig:non_private}.
DPE-FixedU is similar to DPE but samples only a fixed number $U$ of participating clients in each phase (i.e., the participating clients are different, but the number of clients in each phase is fixed, in contrast to our increasing sampling schedule). For a fair comparison, we choose the value of $U$ such that the communication cost is the same under DPE and DPE-FixedU\High{, i.e., $U=\lceil \frac{\sum_{l=1}^L |U_l| \cdot N_l}{\sum_{l=1}^L N_l}\rceil$}. 
The results show that DPE learns much faster than DPE-FixedU while incurring the same communication cost.

\High{Finally, as discussed in Section~\ref{sec:discuss_vs_ucb}, we also compare DPE with the the-state-of-the-art for standard linear bandit problem, i.e. LinUCB and PE, and present the regret comparison in Figure~\ref{fig:Alg_compare} and communication and sample  efficiency in Table~\ref{tab:vs_LinUCB_PE}. The results show that DPE can achieve a regret close to that of (adapted) LinUCB and PE by adjusting sampling parameter $\alpha$ while always consuming less communication cost and involving fewer users.
}
\section{Discussion}
\subsection{On Achieving Privacy ``for Free''}\label{sec:forfree}
\High{Following the remark on privacy ``for-free'' (Remark~\ref{rmk:free_privacy}), in this section, we first study differentially private linear bandits and then draw an interesting connection between bandit online learning and supervised learning.}
\subsubsection{Differentially Private Linear Bandits}\label{sec:dplb}
Motivated by the cellular configuration problem, we consider the distributed linear bandits with partial feedback in the main content and propose the DP-DPE algorithmic framework to address the newly introduced challenges. However, we highlight that our developed techniques with minor modifications can also achieve similar results in terms of regret and privacy for the standard linear bandits, where there is no client-related uncertainty ($\sigma=0$), 
i.e., $\theta_u = \theta^*$ in our notations. That is, we can design differentially private linear bandits where one can also achieve privacy ``for free" in the central and shuffle DP models (similar to Remarks~\ref{rmk:free_privacy}). We list the corresponding results in Table~\ref{tab:performance_lb}.
This might be of independent interest to the bandit learning community. We provide the detailed description of differentially private linear bandits in Appendix~\ref{app:dplb}. 

\begin{table*}[t]
\centering
\caption{Summary of the DP algorithms for the standard linear bandits}
\label{tab:performance_lb}
		\scalebox{0.95}{
	\begin{tabular}{ccc}
	\toprule
	Algorithm\footnotemark & Regret  & Privacy\\
	\midrule
	PE \cite{lattimore2020bandit}& $O\left( \sqrt{dT\log (kT)}\right)$   & None  \\
	CDP-PE & $ O\left( \sqrt{dT\log (kT)} + \frac{Bd^{3/2}\log(T)\sqrt{\ln(1/\delta)\log(kT)}}{\epsilon} \right)$  &  $(\epsilon, \delta)$-DP\\
	LDP-PE& $O(\sqrt{dT\log (kT)} +O\left( \frac{Bd^{3/2}\sqrt{\ln(1/\delta)T\log(kT)}}{\epsilon} \right)$ & $(\epsilon, \delta)$-LDP\\
	SDP-PE & $ O\left(\sqrt{dT\log (kT)} + \frac{Bd^{3/2}\log(T)\sqrt{\ln(1/\delta)\log(kT)}}{\epsilon} \right)$ &$(\epsilon, \delta)$-SDP\\
	\bottomrule
	\multicolumn{3}{l}{\footnotesize $^{10}$PE is the (non-private) phased elimination algorithm in \cite{lattimore2020bandit}; CDP-PE, LDP-PE, and SDP-PE represent } \\ 
	\multicolumn{3}{l}{\footnotesize the designed DP algorithms in the central, local, and shuffle models, respectively, which guarantee $(\epsilon, \delta)$-DP, }\\
 \multicolumn{3}{l}{\footnotesize $(\epsilon, \delta)$-LDP, and $(\epsilon, \delta)$-SDP, respectively.} \\ 
	\end{tabular}
		}
\end{table*}

\begin{remark}
\label{rem:context-lcb}
    We can achieve the above ``for-free'' results because the sensitive information in linear bandits are only rewards, which is in sharp contrast to linear \emph{contextual} bandits where both contexts and rewards need to be protected. In this case, the best known private regrets in the central, local and shuffle model are $\tilde{O}({\frac{\sqrt{T}}{\sqrt{\epsilon}}})$~\cite{shariff2018differentially}, $\tilde{O}({\frac{T^{3/4}}{\sqrt{\epsilon}}})$~\cite{zheng2020locally}, and $\tilde{O}(\frac{T^{3/5}}{\epsilon^{2/5}})$~\cite{chowdhury2022shuffle}, respectively.
\end{remark}

\subsubsection{Connection with Supervised Learning}
In addition, we draw an interesting connection of our novel bandit online learning problem to private (distributed) supervised learning problems, through which we provide more intuition on why DP-DPE can achieve privacy ``for free''. In particular, we compare our problem with differentially private stochastic convex optimization (DP-SCO) \cite{bassily2019private}, where the goal is to approximately minimize the population loss\footnote{The population loss for a solution $w$ is given by $\cL(w) \triangleq \mathbb{E}_{z \in \cD}[l(w,z)]$, where $w$ is the chosen solution (e.g., weights of a classifier), $z$ is a testing sample from the population distribution $\cD$, and $l$ is a convex loss function of $w$.} over convex and Lipschitz loss functions given $n$ \emph{i.i.d.} $d$-dimensional samples from a population distribution while protecting privacy under different trust models.  More specifically, via noisy stochastic gradient descent (SGD), the excess losses\footnote{The excess loss measures the gap between the chosen solution and the optimal solution in terms of the population loss. That is, the excess loss of $w$ is given by $\cL(w) - \min_{w' \in \cW} \cL(w')$, where $w$ is often the minimizer of the \emph{Empirical Risk Minimization (ERM)} problem: $\hat{\cL}(w) \triangleq \frac{1}{n}\sum_{i=1}^n l(w,z_i)$, where $\{z_i\}_{i=1}^n$ are \emph{i.i.d.} samples from the population distribution. To find the minimizer of ERM, we often resort to SGD.} in DP-SCO under various trust models are roughly as follows:
\begin{align}
	&\text{Central and Shuffle Model~\cite{bassily2019private,cheu2021shuffle}:} \quad \widetilde{O}\left(\frac{1}{\sqrt{n}} + \frac{\sqrt{d}}{n\epsilon}\right), \label{eq:sco_cs}\\
	&\text{Local Model~\cite{duchi2018minimax}:} \quad \widetilde{O}\left(\frac{1}{\sqrt{n}} + \frac{\sqrt{d}}{\sqrt{n}\epsilon}\right) \label{eq:sco_l}.
\end{align}
Recall our main results in Table~\ref{tab:performance_summary} as follows (ignoring all the logarithmic terms for clarity): 
\begin{align}
&\text{Central and Shuffle Model:} \quad \widetilde{O}\left( T^{1-\alpha/2} + \frac{d^{3/2} T^{1-\alpha}}{\epsilon}\right), \label{eq:dpe_cs}\\
	&\text{Local Model:} \quad \widetilde{O}\left( T^{1-\alpha/2} + \frac{d^{3/2} T^{1-\alpha/2}}{\epsilon}\right)\label{eq:dpe_l}.
\end{align}
Now, one can easily see that in both problems, privacy protection is achieved ``for free'' in the central and shuffle models, in the sense that the second term (i.e., the additional privacy-dependent term) is a lower-order term (with respect to $n$ or $T$) compared to the first term in both Eqs.~(\ref{eq:sco_cs}) and~(\ref{eq:dpe_cs}). On the other hand, under the much stronger local model, in both problems, the additional privacy-dependent term is of the same order as the first term in both Eqs.~(\ref{eq:sco_l}) and~(\ref{eq:dpe_l}).

We tend to believe that the above interesting connection is not a coincidence. Rather, it provides us with a sharp insight into our introduced DP-DLB formulation. In particular, we know that the first term $1/\sqrt{n}$ in DP-SCO comes from standard concentration results, i.e., how independent samples approximate the true population parameter. Similarly, in our problem, the first term $\sqrt{d}T^{1-\alpha/2}$ comes from the concentration due to client sampling, which is used to approximate the true unknown population parameter $\theta^*$. On the other hand, the second term in DP-SCO is privacy-dependent and comes from the average of noisy gradients.  Similarly, in our problem, the second term is due to the average of the local reward vectors with added noise for preserving privacy.

In addition to these useful insights, we believe that this interesting connection also opens the door to a series of important future research directions, in which one can leverage recent advances in DP-SCO to improve our main results (dependence on $d$, communication efficiency, etc.).

\subsection{Comparison with the-State-of-the-Art}\label{sec:discuss_vs_ucb}
\High{Some perceptive readers might think reducing the model to a problem where each user $u$ can observe \emph{i.i.d.} rewards with mean $\langle \theta^*, x\rangle$ by treating $\langle \theta_u - \theta^*, x\rangle$ as an additional noise to $\eta_t$. In this case, we may solve our problem with the existing solutions to the traditional linear bandits. However, they exhibit the following significant limitations. 

Note that the uncertainty introduced by the additional noise has to be addressed by sampling enough clients, e.g., one client per round. Considering DP, this problem essentially reduces to the differential private linear bandit (also discussed in our Section~\ref{sec:dplb}) with a larger noise variance, where the same results in terms of regret (order-wise) and privacy can be achieved. 
However, one new user is sampled in each round to collect reward observation, which requires exactly $T$ users in total to obtain the desired regret while ensuring the privacy guarantee. Instead, the DP-DPE framework in this work provides an approach where it collects feedback from multiple clients for the selected action in each round while each client serves for multiple rounds to maintain (or improve) sample efficiency. Specifically, it samples $\lceil 2^{\alpha l}\rceil$ clients for $2^l$ plays (rounds in the $l$-th phase), which is $O(T^{\alpha})$ users in total. In addition, by only collecting feedback after preprocessing reward observations at the end of each phase, this carefully designed DP-DPE algorithmic framework reduces the communication cost from exactly $T$ to $O(dT^{\alpha})$. We have to mention that choosing $\alpha<1$, however, will incur a larger privacy cost (see Table~\ref{tab:performance_summary}
). Therefore, there is a tradeoff between the regret penalty due to privacy and the communication and sampling efficiency, which can be balanced by tuning $\alpha$ properly. Meanwhile, we run simulations of the non-private algorithms: DPE, LinUCB in \cite{li2010contextual}, and PE in \cite{lattimore2020bandit}, and present the results in Figure~\ref{fig:Alg_compare} and Table~\ref{tab:vs_LinUCB_PE}. The results show that DPE can achieve similar regret performance (by adjusting parameter $\alpha$) to LinUCB and PE while improving user-sampling efficiency and communication efficiency significantly for each $\alpha\in (0,1)$. }

\subsection{Extensions to Non-linear Bandits}
In this work, we study the problem of global reward maximization with distribution feedback in the stochastic linear bandit model where \emph{direct reward observations} are not available.
Note that the same challenge (i.e., no direct /partial reward feedback) could also exist in other general bandit models, e.g., generalized linear bandits and kernelized bandits. 
We believe our algorithmic framework incorporating different DP models can be extended through careful accommodation for the parametric generalized linear bandits. Specifically, one may refer to \cite{filippi2010parametric} to update the estimator of $\tilde{\theta}_l$ and the confidence width $W_l$ for the upper/lower confidence bound (UCB/LCB) of each active arm used in the elimination rule in any particular phase $l$.  
However, our algorithmic framework may not be extended directly to the non-parametric kernelized bandits. We study the new challenges and present the solutions in our recent paper \cite{liprivate}.

\section{Related Work}\label{sec:related_work} 
The bandit models (including linear bandits) and their variants have proven to be useful for many real-world applications and have been extensively studied~(see, e.g., \cite{bubeck2012regret,slivkins2019introduction,lattimore2020bandit} and references therein).
Most of the existing studies assume that the exact reward feedback is available to the learning agent for updating the model. However, there is a key difference in the new linear bandit setting we consider: while an action is taken at a central server, it influences a large population of users that contribute to the global reward, which, unfortunately, is not fully observable. Instead, one can learn the global model at the server by randomly sampling a subset of users from the population and iteratively aggregating such partial distributed feedback. While this setting shares some similarities with distributed bandits, federated bandits, and multi-agent cooperative bandits,  our motivation and model are very different from theirs, which leads to different regret definitions (global regret vs. group regret; see Section~\ref{sec:problem_formulation}) and algorithmic solutions.
In the following, we discuss the most relevant work in the literature and highlight the key differences.

\smallskip

\textbf{Linear bandits.}
While the stochastic multi-armed bandits (MAB) model has been extensively studied for a wide range of applications, its modeling power is limited by the assumption that actions are independent. In contrast, the linear bandit model captures the correlation among actions via an unknown parameter~\cite{dani2008stochastic, rusmevichientong2010linearly, abbasi2011improved}. 
The best-known regret upper bound for stochastic linear bandits is $
O(d\sqrt{T\log(T)})$ in \cite{abbasi2011improved}, 
which holds for an almost arbitrary, even infinite, bounded subset of a finite-dimensional vector space.
For a special setting where the set of actions is finite and does not change over time, it is shown in \cite{lattimore2020bandit} that a \emph{phased elimination with G-optimal exploration} algorithm guarantees a regret upper bounded by $O(\sqrt{dT\log(kT)})$. This new bound is better by a factor of $\sqrt{d}$, which deserves the effort when $d\geq \log (k)$. 
However, none of these studies consider the scenario where an action influences a large population and the exact reward feedback is unavailable, which is a key challenge in our problem. Note that the linear bandits model we consider is different from the contextual linear bandits in \cite{li2010contextual, chu2011contextual} where the parameter is not shared by actions (although assuming linear reward function), and thus, the actions are not correlated through the parameter. 
\smallskip

\textbf{Differentially private online learning and bandits.} Since proposed in\cite{dwork2006calibrating}, differential privacy (DP) has become the \emph{de facto} privacy preserving model in many applications, including online learning \cite{jain2012differentially} and bandits problems \cite{mishra2015nearly}. Specifically, 
in \cite{tossou2016algorithms,ren2020multi,tenenbaum2021differentially}, MAB has been studied in the central, local, and shuffle DP models, respectively. \High{We refer interested readers to~\cite{chowdhury2022distributed} for state-of-the-art results on private MABs under all three models.}
In \cite{shariff2018differentially}, the authors explore DP in contextual linear bandits and introduce joint DP as ensuring the standard DP incurs a linear regret. As stronger privacy protection, local DP is also studied for contextual linear bandits \cite{zheng2020locally} and Bayesian optimization~\cite{zhou2020local}. Very recently, shuffle model for linear \emph{contextual} bandits have been studied in~\cite{chowdhury2022shuffle}. As already highlighted in Remark~\ref{rem:context-lcb}, the additional protection of context information leads to a higher cost of privacy compared to linear bandits considered in our paper, where only rewards are private information. \High{One concurrent work \cite{hanna2022differentially} with our conference paper study the standard linear bandits in all the three DP models as ours while ensuring pure DP. However, different from the unified algorithmic framework in this paper,  their algorithms in different DP models are independently designed, and  their shuffle model requires the shuffler to do more than shuffling.}

\smallskip
\textbf{Distributed bandits.}
Another line of related work is on multi-agent collaborative learning in the distributed bandits setting \cite{agarwal2021multi, cesa2016delay, martinez2019decentralized, dubey2020kernel,
dubey2020cooperative, wang2019distributed}. 
The most relevant work to ours is the distributed linear bandit problem studied in \cite{wang2019distributed}. Similarly, they design a distributed phased elimination algorithm where a central server aggregates data provided by the local clients and iteratively eliminates suboptimal actions. However, there are two key differences: 
i) they consider the standard group regret minimization problem with homogeneous clients that have the same unknown parameter; ii) the clients send the rewards to the central server without any data privacy protection.  

\smallskip
\textbf{Federated bandits.}  
Federated learning (FL) has received substantial attention since its introduction in \cite{mcmahan2017communication}. The main idea of FL is to enable collaborative learning among heterogeneous devices while preserving data privacy.
Very recently, bandit problems have also been studied in the federated setting, where the underlying problem is a bandit one, including \emph{federated multi-armed bandits}~\cite{shi2021fmab, shi2021federated, zhu2021federated}, 
\emph{federated linear bandits}  \cite{dubey2020differentially,huang2021federated}, and \emph{federated Bayesian optimization} \cite{dai2020federated}. 
Among all the above work, the two most relevant studies are \cite{huang2021federated} and \cite{dubey2020differentially}. While they both consider the case where all heterogeneous users share the same unknown parameter with heterogeneous decision sets,
in our problem of global reward maximization, the users have heterogeneous unknown local parameters.

In addition to the differences in model and problem formulation, we also highlight our main technical contributions compared to these works in the following. 
While a phased elimination algorithm is also employed in \cite{huang2021federated}, there are two key differences: i) They do not consider the correlation among the actions. That is, the linear bandits setting plays a different role in their work. Specifically, they consider a linear reward for contextual bandits while still studying multi-armed bandits with independent actions, each of which is associated with a distinct parameter vector. Differently, the linear bandits formulation in our work is used to capture the correlation among the actions and can be extended to the case with an infinite number of actions; ii) When aggregating users' data for learning the global parameter, we protect users' data privacy using rigorous differential privacy guarantees, which, however, is not considered in their design. While DP is also employed to protect users' data privacy in \cite{dubey2020differentially}, they require that both the Gram matrix of actions (of size $O(d^2)$) and reward vectors (of size $O(d)$) be periodically communicated using some DP mechanisms (e.g., the Gaussian mechanism). Instead, in our algorithm, only private average local reward for the chosen actions (of size $O(d\log\log d))$ would be communicated in each phase. Moreover, while they only consider a variant of the central DP model, our DP-DPE solution provides a unified algorithmic learning framework, which can be instantiated with different DP models. Specifically, DP-DPE with the shuffle model enables us to achieve a finer regret-privacy-communication tradeoff (see Table~\ref{tab:performance_summary}). That is, not only can it achieve nearly the same regret performance as the central model (yet without trusting the central server), but it requires the users to report feedback in bits only throughout the learning process.

\smallskip 

\High{Recently, we also extended our setup to the non-linear case by considering kernelized bandits~\cite{liprivate}.} 

\High{Despite the above work regarding federated bandits,} one may wonder whether we can follow the idea of federated learning to share clients' locally learned model parameters only. This way, one can avoid sharing raw data, which is another way of protecting clients' data privacy. However, we argue that the additional benefit is marginal. On the one hand, by employing different DP mechanisms, our proposed DP-DPE algorithms already ensure provable privacy guarantees. On the other hand, the communication cost of transmitting the (private) average rewards is nearly the same as that of transmitting the local model parameters. Specifically, in each phase, a client in our DP-DPE algorithm needs to send a $|\text{supp}(\pi_l)|$-dimensional vector in DP-DPE, compared to a $d$-dimensional vector when sending the local model parameters. Therefore, the difference is marginal since we have $|\text{supp}(\pi_l)|\leq 4d\log\log d + 16$. 

\smallskip
\High{
\textbf{Reinforcement learning.} Note that reinforcement learning (RL) \cite{agarwal2019reinforcement} is a generalization of bandits with a distinct new feature -- the agent's actions not only yield immediate rewards but also influence the environment's future state(s). In other words,  bandits is a special and simple case of RL  where the horizon length of each episode is one, and hence, the action will not impact the state for the next step as the episode just restarts. In this sense, our study in bandits (dealing with a stateless environment) could shed light on distributed RL, including efficient communication design and differentially private algorithmic framework design, which might be of independent interest to the RL community.  
}

\section{Conclusion} \label{sec:conclusion}
In this paper, we studied a new problem of global reward maximization with partial distributed feedback. This problem is motivated by several practical applications where the expected reward of an action represents the overall performance over a large population. In such scenarios, it is often difficult, if not impossible, to collect exact reward feedback. To that end, we proposed a differentially private distributed linear bandits formulation, where the learning agent samples clients and interacts with them by iteratively aggregating such partial distributed feedback in a privacy-preserving fashion. 
We then developed a unified algorithmic learning framework, called DP-DPE, which can be naturally integrated with different DP models, and systematically established the regret-communication-privacy tradeoff.

In this work, we assumed that actions are correlated through a common linear function with parameter $\theta^*$. One interesting direction for future work is to extend linear functions 
to general (possibly non-convex) functions via kernelized bandits. Moreover, our current privacy guarantee under the shuffle model is only approximated DP. One promising future direction is to explore pure DP in the shuffle model by building upon the recent advance in MAB~\cite{chowdhury2022distributed}.   
Finally, our work also raises several interesting questions that are worth investigating. For example, can we further improve the communication efficiency by using advanced shuffle protocols?  Can we generalize our formulation to studying reinforcement learning problems?


\printbibliography 
\newpage

\section{Proofs and Supplementary Materials for Section~\ref{sec:DP-instantiation}}\label{app:proof_privacy}
In this section, we provide the proof of the theorems in Section~\ref{sec:DP-instantiation}  to show that DP-DPE is differentially  private  under different models. Before considering different models, we describe the Gaussian Mechanism in the differential privacy literature. 

Let $f: \mathcal{U} \to \R^s$ be a vector-valued function operating on databases. The $\ell_2$-sensitivity of $f$, denoted $\Delta_2f$ is the maximum over all pairs $\mathcal{U}, \mathcal{U}'$ of neighboring datasets of $\Vert f(\mathcal{U})-f(\mathcal{U'})\Vert$. The Gaussian mechanism adds independent noise drawn from a Gaussian with mean zero and standard deviation slightly greater than $\Delta_2f\sqrt{\ln(1/\delta)}/\epsilon$ to each element of its output \cite{dwork2014analyze}.
\begin{theorem}\label{thm:gaussian_mech}
	({The Gaussian Mechanism} \cite{dwork2014algorithmic}).  
	Given any vector-valued function $f: \mathcal{U}^* \to \R^s$, define $\Delta_2\triangleq \max_{\mathcal{U}, \mathcal{U}' \in \mathcal{U}^*} \Vert f(\mathcal{U})-f(\mathcal{U'})\Vert_2$. Let $\sigma = \Delta_2 \sqrt{2\ln (1.25/\delta)}/\epsilon$. The Gaussian mechanism, which adds independently drawn random noise from $\mathcal{N}(0,\sigma^2)$ to each output of $f(\cdot)$, i.e. returning $f(\cU)+(\gamma_1, \dots, \gamma_s)$ with $\gamma_j \overset{\text{i.i.d.}}{\sim} \mathcal{N}(0,\sigma^2)$, ensures $(\epsilon, \delta)$-DP.
\end{theorem}

\subsection{The Central Model}
The \textsc{Privatizer} $\cP$ in the central model adds Gaussian noise to the averaged local performance of each action directly, i.e., at analyzer $\cA$ while doing identity mapping with $\cR$ and $\cS$. That is, 
\begin{equation}
	\begin{aligned}
		\Tilde{y}_l = &
		\cP\left(\{\Vec{y}_{l}^u\}_{u\in  U_l}\right) = \cA\left(\{\Vec{y}_{l}^u\}_{u\in  U_l}\right) \\
		= & \frac{1}{| U_l|}\sum_{u\in  U_l} \Vec{y}_{l}^u+(\gamma_1, \dots, \gamma_{s_l}),    
	\end{aligned}
\end{equation}
where $\gamma_j \overset{\emph{i.i.d.}}{\sim} \mathcal{N}(0,\sigma^2_{nc})$.
\begin{proof}[Proof of Theorem~\ref{thm:cdp}]
	Recall that $\mathcal{U}_T =( U_l)_{l=1}^L\subseteq \mathcal{U}$ represents the sequence of participating clients in the total $L$ phases.  Let $\mathcal{U}_T, \mathcal{U}_T^{\prime} \subseteq \mathcal{U}$ be the sequence of participating clients differing on a single client and $U_l$, $U_l^{\prime}$ be the sets of participating clients in $l$-th phase corresponding to $\mathcal{U}_T$ and $\mathcal{U}_T^{\prime}$ respectively.  Note that the $\ell_2$ sensitivity of the average $\frac{1}{| U_l|}\sum_{u\in  U_l}\Vec{y}_{l}^u$ is
	\begin{equation}
		\begin{aligned}
			\Delta_2
			&=  \max_{\mathcal{U}_T, \mathcal{U}_T^{\prime}} \left\Vert \frac{1}{| U_l|}\sum_{u\in  U_l}\Vec{y}_{l}^u - \frac{1}{| U_l|}\sum_{u\in U^{\prime}_l}\Vec{y}_{l}^u \right\Vert_2\\
			&\leq \frac{1}{| U_l|}\max_{u,u' \in \cU}\Vert \Vec{y}_{l}^u - \Vec{y}_{l}^{u'}\Vert_2 \\
			&\leq \frac{2}{| U_l|}\max_{u\in \cU}\Vert \Vec{y}_l^u\Vert_2
			\leq \frac{2B\sqrt{s_l}}{| U_l|},    
		\end{aligned}
	\end{equation}
	where the last step is because $s_l$ is the dimension of $y_l^u$ and then $\Vert \Vec{y}_l^u \Vert_2^2 \leq {s_l} \Vert y_l^u(x)\Vert_1^2 \leq {s_l} B^2 $.
	
	Let $\sigma_{nc} = \frac{2B\sqrt{2s_l\ln(1.25/\delta)}}{\epsilon | U_l|}$. According to Theorem~\ref{thm:gaussian_mech}, we have that the output $\Tilde{y}_l$ of the \textsc{Privatizer} in the central model is $(\epsilon,\delta)$-DP with respect to $\mathcal{U}_T$.

	Combining the result of Proposition 2.1 in \cite{dwork2014algorithmic}, we derive that DP-DPE with a \textsc{Privatizer} in the central model is $(\epsilon,\delta)$-DP. 
\end{proof}

\subsection{The Local Model}
The \textsc{Privatizer} $\cP$ in the local model applies Gaussian mechanism to the local average performance of each action ($\Vec{y}_{l}^u$) with $\cR$ while doing identity mapping with $\cS$ and pure averaging with $\cA$. That is, 
\begin{equation}
	\Tilde{y}_l = \mathcal{P}\left(\{\Vec{y}_{l}^u\}_{u\in  U_l}\right) = \frac{1}{| U_l|}\sum_{u\in  U_l} \cR(\Vec{y}_{l}^u) = \frac{1}{|U_l|}\sum_{u\in U_l} \left(\Vec{y}_{l}^u+\gamma_u\right), \notag
\end{equation}
where $\gamma_u \overset{\text{\emph{i.i.d.}}}{\sim}  \mathcal{N}(0,\sigma^2_{nl}I_{s_l})$.
\begin{proof}[Proof of Theorem~\ref{thm:ldp}]
	An algorithm is $(\epsilon, \delta)$-LDP if the output of the local randomizer $\cR$ is $(\epsilon, \delta)$-DP. 
	In the local model of \textsc{Privatizer}, we have 
	\begin{equation}
		\cR (\Vec{y}_{l}^u) = \Vec{y}_{l}^u+\gamma_{u}.
	\end{equation}
	For any input of local parameter estimator $\Vec{y}_{l}$, the $\ell_2$ sensitivity is
	\begin{equation}
		\Delta_2
		= \max_{\Vec{y}_{l}, \Vec{y}_{l}^{\prime}} \Vert\Vec{y}_{l} - \Vec{y}_{l}'\Vert_2
		\leq 2\max \Vert \Vec{y}_{l}^u\Vert_2
		\leq 2B\sqrt{s_l}.
	\end{equation}

	Let $\sigma_{nl} = \frac{2B\sqrt{2s_l\ln(1.25/\delta)}}{\epsilon}$. According to Theorem~\ref{thm:gaussian_mech}, we have that the output of the local randomizer $\cR$ of the \textsc{Privatizer} in the local model is $(\epsilon,\delta)$-DP. That is,  the DP-DPE algorithm instantiated with the \textsc{Privatizer} in the local model is $(\epsilon,\delta)$-LDP.
\end{proof}

\subsection{The Shuffle Model} \label{app:shuffle_model}
In the shuffle model, the \textsc{Privatizer} $\cP$ operates under the cooperation of the local randomizer $\cR$ at each client, the shuffler $\cS$, and the analyzer $\cA$ at the central server. First, we present the implementation of each component of the \textsc{Privatizer} $\cP$, including $\cR$, $\cS$, and $\cA$, in the shuffle model in Algorithm~\ref{alg:shuffle_vec}.

\begin{algorithm}[ht]
	\caption{ $\cM:$ $(\epsilon,\delta)$-SDP vector average mechanism for a set $U$ of clients} 
	\label{alg:shuffle_vec}
	\begin{algorithmic}[1]
		\STATE \textbf{Input:} $\{y_u\}_{u\in U}$, where each $y_u\in \R^s$, $\Vert y_u\Vert_2\leq \Delta_2$
		\STATE Let 
		\begin{equation}
			\begin{cases}
				\hat{\epsilon}=\frac{\epsilon}{18\sqrt{\log(2/\delta)}} \\
				g=\max\{\hat{\epsilon}\sqrt{|U|}/(6\sqrt{5\ln{((4s)/\delta)}}), \sqrt{s}, 10\}\\
				b=\lceil \frac{180g^2\ln{(4s/\delta)}}{\hat{\epsilon}^2|U|}\rceil\\
				p = \frac{90g^2\ln{(4s/\delta)}}{b\hat{\epsilon}^2|U|} \label{eq:gbp_set}
			\end{cases}    
		\end{equation}
		\item[] \deemph{// Local Randomizer}
		\item[] \textbf{function} 
		$\cR(y_u)$ \\
		\begin{ALC@g}
			\FOR{coordinate $j\in [s]$} 
			\STATE Shift data to enforce non-negativity: $w_{u,j} = (y_u)_j + \Delta_2, \forall u\in U$
			\item[] //randomizer for each entry
			\STATE Set $\Bar{w}_{u,j} \gets \lfloor w_{u,j}g/(2\Delta_2)\rfloor$  \hfill \deemph{//$\max|(y_u)_j+\Delta_2|\leq 2\Delta_2$}
			\STATE Sample rounding value $\gamma_1 \sim \textbf{Ber}(w_{u,j}g/(2\Delta_2) - \Bar{w}_{u,j})$
			\STATE Sample privacy noise value $\gamma_2 \sim \textbf{Bin}(b,p)$
			\STATE Let $\phi_j^u$ be a multi-set of $(g+b)$ bits associated with the $j$-th coordinate of client $u$, where $\phi_j^u$ consists of $\Bar{w}_{u,j}+\gamma_1+\gamma_2$ copies of 1 and $g+b-(\Bar{w}_{i,j}+\gamma_1+\gamma_2)$ copies of 0
			\ENDFOR
			\STATE Report $\{(j,\phi^u_j)\}_{j\in [s]}$ to the shuffler
		\end{ALC@g}
		\item[] \textbf{end function}
		\item[] \deemph{// Shuffler}
		\item[] \textbf{function} 
		$\mathcal{S}(\{(j, \Vec{\phi}_j)\}_{j\in[s]})$  \quad \deemph{//$\Vec{\phi}_j = (\phi^u_j)_{u\in U}$}\\
		\begin{ALC@g}
			\FOR{each coordinate $j\in [s]$}
			\STATE Shuffle and output all  $(g+b)|U|$ bits in $\Vec{\phi}_j$ 
			\ENDFOR
		\end{ALC@g}
		\item[] \textbf{end function} \\
		\item[] \deemph{// Analyzer}
		\item[] \textbf{function} 
		$\mathcal{A}(\mathcal{S}(\{(j, \Vec{\phi}_j)\}_{j\in[s]})$
		\begin{ALC@g}
			\FOR{coordinate $j\in [s]$}
			\STATE Compute $z_j \gets \frac{2\Delta_2}{g|U|} ((\sum_{i=1}^{(g+b)|U|} (\Vec{\phi}_{j})_i)-b|U|p) $   \quad\quad \deemph{// $(\Vec{\phi})_i$ denotes the $i$-th bit in $\Vec{\phi}_j$}
			\STATE Re-center: $o_j \gets z_j - \Delta_2$
			\ENDFOR
			\STATE  Output the estimator of vector average ${o}=(o_j)_{j\in [s]}$
		\end{ALC@g}
		\textbf{end function}
	\end{algorithmic}
\end{algorithm}

The \textsc{Privatizer} $\cP$ in the shuffle model adds binary bits to  the local average performance of each played action (after being converted to binary representation) at each client, i.e., at local randomizer $\cR$, and then shuffles all bits reported from all participating clients via a shuffler $\cS$ before sending them to the central server, where the analyzer $\cA$ output an unbiased and private estimator of the average of local parameters.   That is, 
\begin{equation}
	\Tilde{y}_l =  \mathcal{P}\left(\{\Vec{y}_l^{u}\}_{u\in U_l}\right)  = \cA(\mathcal{S} (\{\cR(\Vec{y}_l^{u})\}_{u\in U_l})).
\end{equation}

Before proving Theorem~\ref{thm:sdp}, we first show that the shuffle protocol in Algorithm~\ref{alg:shuffle_vec} is $(\epsilon,\delta)$-SDP. 

In this proof, we use $(\cdot)_j$ to denote the $j$-th element of an vector. 
\begin{theorem} \label{thm:vector_shuffle}
	For any $\epsilon \in (0,15), \delta \in (0,1)$, Algorithm~\ref{alg:shuffle_vec} is $(\epsilon,\delta)$-SDP, unbiased, and has error distribution which is sub-Gaussian with variance $\sigma_{ns}^2 = O\left(\frac{\Delta_2\ln(s/\delta)}{\epsilon^2 |U|^2}\right)$ and independent of the inputs.
\end{theorem}
\begin{proof}
	The proof for the privacy part follows from the SDP guarantee of vector summation protocol in \cite{cheu2021shuffle}. In the following, we try to show the remaining part of the above theorem.
	
	Consider an arbitrary coordinate $j\in [s]$. Note that the sum of the messages produced by $\cR$ at client $u$ is:  $ \bar{w}_{u,j}+\gamma_1+\gamma_2$. Since $\gamma_1$ is drawn from \texttt{Ber}$(w_{u,j}g/(2\Delta_2)-\Bar{w}_{u,j})$, which has expectation $\E[\gamma_1] = w_{u,j}g/(2\Delta_2)-\Bar{w}_{u,j}$ and is $1/2$-sub-Gaussian according to Hoeffding Lemma. Meanwhile, $\gamma_2\sim \texttt{Bin}(b,p)$ indicates $\E[\gamma_2]=bp$ and $\sqrt{b}/2$-sub-Gaussian.
	Recall that $z_j = \frac{2\Delta_2}{g|U|} ((\sum_{i=1}^{(g+b)U} (\Vec{\phi}_{j})_i)-b|U|p)
	=\frac{2\Delta_2}{g|U|} (\sum_{u\in U} (\bar{w}_{u,j}+\gamma_1+\gamma_2)-b|U|p)$ and  $o_j = z_j - \Delta_2$. We have
	$$
	\begin{aligned}
		&\E[o_j] =\E[z_j-\Delta_2] \\
		=&\E\left[\frac{2\Delta_2}{g|U|} \left(\sum_{u\in U} (\bar{w}_{u,j}+\gamma_1+\gamma_2)-b|U|p\right)-\Delta_2\right]\\
		=&\frac{2\Delta_2}{g|U|} \left(\sum_{u\in U} (\bar{w}_{u,j}+\E[\gamma_1]+\E[\gamma_2])-b|U|p\right)-\Delta_2\\
		=& \frac{2\Delta_2}{g|U|} \left(\sum_{u\in U} \left(\frac{w_{u,j}g}{2\Delta_2}+bp\right) -b|U|p\right)-\Delta_2\\
		= &\frac{1}{|U|} \sum_{u\in U} w_{u,j} -\Delta_2\\
		=& \frac{1}{|U|} \sum_{u\in U}\left( (y_u)_j+\Delta_2\right) -\Delta_2 \\
		=&\frac{1}{|U|} \sum_{u\in U} (y_u)_j,
	\end{aligned}$$
	which indicates that the output is unbiased estimator of the average. In addition, according to the property of sub-Gaussian, we have the output $o_j$ satisfies
	
	$$
	\begin{aligned}
		&\text{Var}[o_j] = \text{Var}[z_j-\Delta_2] \\ = &\text{Var}\left[\frac{2\Delta_2}{g|U|} \left(\sum_{u\in U} (\bar{w}_{u,j}+\gamma_1+\gamma_2)\right)\right]\\
		=& \text{Var}\left[\frac{2\Delta_2}{g|U|} \left(\sum_{u\in U} (\gamma_1 + \gamma_2)\right)\right]\\
		\leq & \frac{4\Delta^2_2}{g^2|U|^2}\left(\frac{|U|}{4}+\frac{b|U|}{4}\right) \\
		\leq & \frac{\Delta^2_2}{g^2|U|^2}\left({|U|}+|U| \cdot (\frac{180g^2\ln{(4s/\delta)}}{\hat{\epsilon}^2|U|}+1)\right) \\
		= & \frac{\Delta^2_2}{g^2|U|^2}\left(2|U|+\frac{180g^2\ln{(4s/\delta)}}{\hat{\epsilon}^2}\right) \\
		\overset{(a)}{\leq } & \frac{180\Delta_2^2 \ln{(4s/\delta)}}{|U|^2\hat{\epsilon}^2} + \frac{180\Delta_2^2 \ln{(4s/\delta)}}{|U|^2\hat{\epsilon}^2} \\
		=  & \frac{360\Delta_2^2 \ln{(4s/\delta)}}{|U|^2\hat{\epsilon}^2} \\
		= &O\left(\frac{\Delta_2^2 \ln^2{(s/\delta)}}{|U|^2\epsilon} \right),
	\end{aligned}
	$$
	where $(a)$ is derived from our choice of $g$ in Eq.~\eqref{eq:gbp_set}. The output $o_j$ is $O\left(\frac{\Delta_2 \ln{(s/\delta)}}{|U|\epsilon}\right)$-sub-Gaussian. Then, the output vector $o=\{o_j\}_{j\in [s]}$ is a $s$-dimensional $O\left(\frac{\Delta_2 \ln{(s/\delta)}}{|U|\epsilon}\right)$-sub-Gaussian vector according to the definition of the sub-Gaussian vector.  
\end{proof}
Now, we are ready to prove Theorem~\ref{thm:sdp}. 
\begin{proof}[Proof of Theorem~\ref{thm:sdp}] From Theorem~\ref{thm:vector_shuffle}, we have the shuffle protocol in Algorithm~\ref{alg:shuffle_vec} guarantees $(\epsilon, \delta)$-SDP with inputs $\{y_u\}_{u\in U}$.  
	In the DP-DPE algorithm, we apply the shuffle protocol in Algorithm~\ref{alg:shuffle_vec} as a \textsc{Privatizer} with inputs $\{\Vec{y}_l^u\}_{u\in U_l}$ and $\Delta_2=\max \Vert \Vec{y}_l^u \Vert_2=B\sqrt{s_l}$ in each phase $l$. By admitting new clients in each phase, the DP-DPE algorithm is $(\epsilon,\delta)$-SDP. 
\end{proof}


\section{Proofs of Theorems in Section~\ref{sec:analysis}}
\label{app:proof_of_regret}
In Section~\ref{sec:analysis}, we present the performance analysis  of the DP-DPE algoritm in different DP models. In this section, we provide complete proofs for each of the theorems in Section~\ref{sec:analysis}.

\subsection{Proof of Theorem~\ref{thm:regret_approx}}\label{app:proof_of_regret_dpe}

\vskip 0.3 in
Before proving Theorem~\ref{thm:regret_approx}, we first provide the key concentration under DPE in the following Theorem~\ref{thm:concentration}.
\begin{theorem}\label{thm:concentration}
	For any phase $l$, under DPE with $\sigma_n=0$, the following concentration inequalities hold, for any $x\in \cD$,
	\begin{equation}
		\begin{aligned}
			&P\left\{ \langle \Tilde{\theta}_l-\theta^*, x\rangle \geq W_l\right\}\leq 2\beta, \quad
			\text{and,} \\
			\quad 
			& P\left\{ \langle \theta^* - \Tilde{\theta}_l, x\rangle \geq W_l\right\}\leq 2\beta. \label{eq:concentration_ineq}
		\end{aligned}
	\end{equation} 
\end{theorem}
\begin{proof} We prove the first concentration inequality in \eqref{eq:concentration_ineq} in the following, and the second inequality can be proved symmetrically.  
	Note that for any action $x\in \cD$, the gap between the estimated reward with parameter $\Tilde{\theta}_l$ and the true reward with $\theta^*$ satisfies
	\begin{equation*}
		\begin{aligned}
			&\langle \Tilde{\theta}_l-\theta^*, x\rangle \\
			=&\left\langle \Tilde{\theta}_l- \frac{1}{|U_l|}\sum_{u\in U_l}\theta_u + \frac{1}{|U_l|}\sum_{u\in U_l}\theta_u -\theta^*, x\right\rangle \\
			=& \left\langle  \Tilde{\theta}_l-\frac{1}{|U_l|}\sum_{u\in U_l} \theta_u, x \right\rangle  +  \frac{1}{|U_l|}  \sum_{u\in U_l} \left\langle \theta_u - \theta^*, x\right\rangle.     
		\end{aligned}
	\end{equation*}
	Then, with $\sigma_n=0$, we have
	\begin{equation}
		\begin{aligned}
			& P\left\{ \langle \Tilde{\theta}_l-\theta^*, x\rangle \geq W_l\right\} \\
			= & P\left\{ \left\langle  \Tilde{\theta}_l-\frac{1}{|U_l|}\sum_{u\in U_l} \theta_u, x \right\rangle + \frac{1}{{|U_l|}}\sum_{u\in U_l}  \langle \theta_u - \theta^*, x\rangle \geq W_l\right\} \\
			\leq & P\left\{ \left\langle  \Tilde{\theta}_l-\frac{1}{|U_l|}\sum_{u\in U_l} \theta_u, x \right\rangle \geq \sqrt{\frac{4d}{h_l|U_l|}\log \left(\frac{1}{\beta}\right)}\right\} \\ 
			&+ P\left\{ \frac{1}{{|U_l|}}\sum_{u\in U_l}  \langle \theta_u - \theta^*, x\rangle \geq \sqrt{\frac{2\sigma^2}{|U_l|}\log \left(\frac{1}{\beta}\right)} \right\}. \label{eq:gap_bound}
		\end{aligned}
	\end{equation}
	In the following, we try to bound the above two terms, respectively. 
	Under the non-private DP-DPE algorithm, the output of  $\cP$ is the exact average of local performance, i.e., $\Tilde{y}_l = 
	\cP\left(\{\Vec{y}_{l}^u\}_{u\in U_l}\right)
	= \frac{1}{|U_l|}\sum_{u\in U_l} \Vec{y}_{l}^u$. Then, the estimated model parameter
	\begin{equation}
		\begin{aligned}
			&\Tilde{\theta}_l = V_l^{-1}G_l\\
			& =  V_l^{-1}\sum_{x\in \text{supp}(\pi_l)}  x T_l(x)  \Tilde{y}_l(x) \\
			& =  V_l^{-1}\sum_{x\in \text{supp}(\pi_l)}  x T_l(x)  \frac{1}{|U_l|}\sum_{u\in U_l}y_l^u(x) \\
			& =  \frac{1}{|U_l|}\sum_{u\in U_l} V_l^{-1}\sum_{x\in \text{supp}(\pi_l)}  x T_l(x) y_l^u(x) \\
			& =  \frac{1}{|U_l|}\sum_{u\in U_l} V_l^{-1}\sum_{x\in \text{supp}(\pi_l)}  x \sum_{t\in \cT_l(x)} (x^\top \theta_u + \eta_{u,t}) \\
		\end{aligned}
	\end{equation}
	can be further expressed as
	\begin{equation}
		\begin{aligned}
			&\Tilde{\theta}_l = V_l^{-1}G_l\\
			& =  \frac{1}{|U_l|}\sum_{u\in U_l} V_l^{-1}\left(\sum_{x\in \text{supp}(\pi_l)} T_l(x) x  x^\top \theta_u\right) \\
			&+ \frac{1}{|U_l|}\sum_{u\in U_l} V_l^{-1}\left(\sum_{x\in \text{supp}(\pi_l)}\sum_{t\in \cT_l(x)}\eta_{u,t}x \right)\\
			& =  \frac{1}{|U_l|}\sum_{u\in U_l} V_l^{-1}   \left( V_l\theta_u  + \sum_{t \in\cT_l}  \eta_{u,t}x_t\right)\\
			& =  \frac{1}{|U_l|}\sum_{u\in U_l} \theta_u   +  \frac{1}{|U_l|}\sum_{u\in U_l}V^{-1} \sum_{t \in\cT_l}  \eta_{u,t}x_t. \label{eq:estimation_average_relation}
		\end{aligned}
	\end{equation}
	For any $x$ in $\cD$,
	\begin{equation}
		\left\langle x, V_l^{-1}  \sum_{t \in \cT_l}  \eta_{u,t} x_t\right\rangle =  \sum_{t \in \cT_l}  \langle x, V_l^{-1}  x_t\rangle \eta_{u,t}. \label{eq:one_client_subGaussian}
	\end{equation}
	
	Note that $ \eta_{u,t}$ is \emph{i.i.d.} $1$-sub-Gaussian over $t$ and that the chosen action $x_{t}$ at $t$ is deterministic in the $l$-th phase under the DP-DPE algorithm. Combining the following result,
	$$\sum_{t\in \cT_l}  \langle x, V_l^{-1}  x_t\rangle^2 = x^{\top} V_l^{-1} \left(\sum_{t\in \cT_l}x_tx_t^{\top}\right) V_l^{-1} x=
	\Vert x\Vert_{V_l^{-1}}^2,$$
	where the second equality is due to $V_l = 
	\sum_{t\in \cT_l} x_{t}x_{t}^{\top}$,
	we derive that the LHS of Eq.~\eqref{eq:one_client_subGaussian} is $\Vert x\Vert_{V_l^{-1}}$-sub-Gaussian. Besides, we have $\Vert x\Vert_{V_l^{-1}}^2 
	\leq \frac{\Vert x\Vert_{V_l(\pi_l)^{-1}}^2}{h_l}
	\leq \frac{g(\pi_l)}{h_l}
	\leq  \frac{2d}{h_l}$ by the near-G-optimal design.
	According to the property of a sub-Gaussian random variable, we can obtain 
	\begin{equation}
		\begin{aligned}
			&P\left\{\frac{1}{|U_l|} \sum_{u\in U_l} \left\langle x, V_l^{-1}  \sum_{t \in \cT_l}  \eta_{u,t} x_t\right\rangle \geq \sqrt{\frac{4d}{h_l|U_l|}\log \left(\frac{1}{\beta}\right)}\right\} \\
			&\leq  \exp\left\{-\frac{ |U_l|\frac{4d}{h_l|U_l|}\log (1/\beta)}{2\cdot \frac{2d}{h_l} }\right\}
			= \beta.      \label{eq:W_l1_prob_bnd}
		\end{aligned}
	\end{equation}
	
	Combining the result in Eq.~\eqref{eq:estimation_average_relation}, we have
	\begin{equation*}
		\begin{aligned}
			&P\left\{\left\langle  \Tilde{\theta}_l-\frac{1}{|U_l|}\sum_{u\in U_l} \theta_u, x \right\rangle \geq \sqrt{\frac{4d}{h_l|U_l|}\log \left(\frac{1}{\beta}\right)}\right\} \\
			= &P\left\{\frac{1}{|U_l|} \sum_{u\in U_l} \left\langle x, V_l^{-1}  \sum_{t \in \cT_l}  \eta_{u,t} x_t\right\rangle \geq \sqrt{\frac{4d}{h_l|U_l|}\log \left(\frac{1}{\beta}\right)}\right\} \leq \beta.       
		\end{aligned}
	\end{equation*}

	For the second term of Eq.~\eqref{eq:gap_bound}, we know that
	$\langle \theta_u - \theta^*, x\rangle =\langle \xi_u, x \rangle$ is $\Vert x\Vert_2\sigma$-sub-Gaussian. Similarly, according to the sub-Gaussian property, we have
	\begin{equation}
		\begin{aligned}
			&P\left\{ \frac{1}{{|U_l|}}\sum_{u\in U_l}  \langle \theta_u - \theta^*, x\rangle \geq \sqrt{\frac{2\sigma^2}{|U_l|}\log \left(\frac{1}{\beta}\right)} \right\} \\
			&\leq \exp\left\{-\frac{{|U_l|} \cdot \frac{2\sigma^2}{{|U_l|}}\log (\frac{1}{\beta}) }{2\sigma^2\Vert x\Vert_2^2}\right\} \label{eq:W_l2_prob_bnd}
			\leq  \beta.    
		\end{aligned}
	\end{equation}
	Therefore, we have 
	$$P\left\{ \langle \Tilde{\theta}_l-\theta^*, x\rangle \geq W_l\right\}  \leq 2\beta. $$
	The symmetrical argument completes the proof.
\end{proof}

To prove Theorem~\ref{thm:regret_approx}, we first present two observations with high probability based on Theorem~\ref{thm:concentration}, then analyze the regret in a particular phase $l>2$, and finally combine all phases to get the total regret. 

Under DPE algorithm, define a ``good" event at $l$-th phase as $\mathcal{E}_l$: 
\begin{equation}
	\begin{aligned}\notag
		&\left\{\langle \theta^*- \Tilde{\theta}_l, x^*\rangle \leq W_l \right\},
		\quad \text{and} \\
		&\left\{ \langle \Tilde{\theta}_l-\theta^*, x\rangle \leq W_l  \right\}, \quad \forall x\in \cD\backslash \{x^*\} .
	\end{aligned}
\end{equation}
According to Theorem~\ref{thm:concentration}, it is not difficult to derive $P(\mathcal{E}_l)\geq 1-2k\beta$ via union bound. In addition, under event $\mathcal{E}_l$, we have the following two observations:
\begin{itemize}
	\item[\textbf{1.}] If the optimal action $x^*\in \cD_l$, then $x^* \in \cD_{l+1}$. 
	\item[\textbf{2.}] For any $x\in \cD_{l+1}$, we have $\langle \theta^*, x^* - x\rangle \leq 4W_l$. 
\end{itemize}
\begin{proof}
	
	\smallskip
	\textbf{Observation 1:}
	Let $b\in \argmax_{x\in\cD_l} \langle \Tilde{\theta}_l, x \rangle $. If $x^*=b$, then $x^* \in \cD_{l+1}$ according to the elimination step in Algorithm~\ref{alg:dpe}. If $x^* \neq b$, then under event $\mathcal{E}_l$, we have 
	\begin{equation}
		\begin{aligned}
			\langle \Tilde{\theta}_l, b - x^*\rangle 
			&=   \langle \Tilde{\theta}_l, b\rangle - \langle \Tilde{\theta}_l, x^*\rangle \\
			&\leq   \langle \theta^*, b\rangle + W_l-  \langle \theta^*, x^*\rangle + W_l\\
			&=   \langle \theta^*, b-x^*\rangle + 2W_l\\
			&\leq  2W_l,
		\end{aligned}
	\end{equation}
	which means that $x^*$ is not eliminated at the end of the $l$-th phase, i.e., $x^* \in \cD_{l+1}$. 
	
	\textbf{Observation 2:} For any $x\in \cD_{l+1}$, we have $\langle \Tilde{\theta}_l, b-x\rangle \leq 2W_l $. Then, we have the following steps:
	\begin{equation}
		\begin{aligned}
			2W_l 
			& \geq \langle \Tilde{\theta}_l, b-x\rangle  \\
			& \geq \langle \Tilde{\theta}_l, x^* - x\rangle \\
			&\geq \langle \theta^*, x^*\rangle - W_l -\langle \theta^*, x\rangle - W_l \\
			&=   \langle \theta^*, x^*-x\rangle - 2W_l,
		\end{aligned}
	\end{equation}
	where the second inequality is from Observation 1. 
	Then, we derive  Observation 2. 
\end{proof}
\smallskip

Now, we are ready to prove Theorem~\ref{thm:regret_approx}.
\begin{proof}
	\textbf{1) Regret in a specific phase $l\geq 2$.} 
Let $r_l$ denote the incurred regret in the $l$-th phase, i.e., $r_l \triangleq  \sum_{t\in \cT_l} \langle \theta^*, x^* - x_{t} \rangle $. For any phase $l=1,\dots, L-1$, under event $\cE_{l}$, we have the following result
\begin{equation}
	\begin{aligned}
		r_{l+1} &=  \sum_{t \in \cT_{l+1}}  \langle \theta^*, x^* - x_{t} \rangle \\
		&\leq \sum_{t\in\cT_{l+1}} 4W_{l}\\
		&=  \sum_{t\in\cT_{l+1}} 4\left(\sqrt{\frac{4d}{h_l|U_l|}\log \left(\frac{1}{\beta}\right)}+\sqrt{\frac{2\sigma^2}{|U_l|}\log \left(\frac{1}{\beta}\right)}\right)\\
		& = \underbrace{4T_{l+1}\sqrt{\frac{4d}{h_l|U_l|}\log \left(\frac{1}{\beta}\right)}}_{\text{\textcircled{1}}}+\underbrace{4T_{l+1}\sqrt{\frac{2\sigma^2}{|U_l|}\log \left(\frac{1}{\beta}\right)}}_{\text{\textcircled{2}}}.
	\end{aligned}
\end{equation}
We derive an upper bound for each of the two terms in the above equation. Note that the total number of pulls in the $(l+1)$-th phase is $T_{l+1} = \sum_{x\in \text{supp}(\pi_{l+1})} T_{l+1}(x)$, which satisfies
$$h_1\cdot 2^l =h_{l+1} \leq T_{l+1}\leq h_{l+1} + |\text{supp}(\pi_{l+1})| \leq  h_1\cdot 2^l +S,$$
where $S\triangleq 4d\log\log d + 16 \geq |\text{supp}(\pi_{l+1})|$. In addition, we have $h_l=h_1\cdot 2^{l-1}$ and $2^{\alpha {l}} \leq |U_l|\leq 2^{\alpha l}+1$. Then, 
for \textcircled{1}, we have 
\begin{equation}
	\begin{aligned}
		\text{\textcircled{1}}& \leq 4(h_1\cdot2^l+S)\sqrt{\frac{8 d}{h_1\cdot 2^{(1+\alpha)l}}\log \left(\frac{1}{\beta}\right)}\\
		& = 8\sqrt{2d\log \left(\frac{1}{\beta}\right)} \left(\sqrt{h_1\cdot 2^{(1-\alpha)l}}+ \frac{S}{\sqrt{h_1\cdot 2^{(\alpha+1)l}}}\right).
	\end{aligned}
\end{equation}
As to the second term \textcircled{2}, we have 
\begin{equation}
	\begin{aligned}
		\text{\textcircled{2}}
		& \leq  4(h_1\cdot2^l+S)\sqrt{\frac{2\sigma^2}{2^{\alpha l}}\log \left(\frac{1}{\beta}\right)}\\
		& = 4\sigma\sqrt{2\log \left(\frac{1}{\beta}\right)} \left(h_1\sqrt{2^{(2-\alpha) l}}+ \frac{S}{\sqrt{2^{\alpha l}}}\right).
	\end{aligned}
\end{equation}
Then, for any $l=2,\cdots, L$, the regret in the $l$-th phase $r_l$ is upper bounded by 
\begin{equation}
	\begin{aligned}
		r_{l}
		&\leq 8\sqrt{2d\log \left(\frac{1}{\beta}\right)} \left(\sqrt{h_1\cdot 2^{(1-\alpha){(l-1)}}}+ \frac{S}{\sqrt{h_1\cdot 2^{(\alpha+1){(l-1)}}}}\right) \\
		&+ 4\sigma\sqrt{2\log \left(\frac{1}{\beta}\right)} \left(h_1\sqrt{2^{(2-\alpha) (l-1)}}+ \frac{S}{\sqrt{2^{\alpha (l-1)}}}\right).\label{eq:regret_in_l}    
	\end{aligned}
\end{equation}

\smallskip
\textbf{2) Total regret:}

Define $\mathcal{E}_g$ as the event where the ``good" event occurs in every phase, i.e., $\mathcal{E}_g \triangleq \bigcap_{l=1}^L \mathcal{E}_l$. Based on to Theorem~\ref{thm:concentration}, it is not difficult to obtain  $P\{\mathcal{E}_g\} \geq 1- 2k\beta L$ by applying union bound. At the same time, let $R_g$ be the regret under event $\mathcal{E}_g$, and $R_b$ be the regret if event $\mathcal{E}_g$ does not hold. Then, the expected total regret in $T$ is $\E[R(T)] = P(\mathcal{E}_g)R_g + (1-P(\mathcal{E}_g)) R_b$.

Under event $\mathcal{E}_g$, the regret in the $l$-th phase $r_l$ satisfies Eq.~\eqref{eq:regret_in_l} for any $l\geq 2$. Combining $r_1 \leq 2T_1 \leq 2(h_1+S)$ (since $\langle \theta^*, x^*-x \rangle \leq 2$ for all $x \in \cD$), we have
\begin{equation}
	\begin{aligned}
		&R_g = \sum_{l=1}^L r_l  \\
		& \leq 2(h_1+S)\\
		&+ \sum_{l=2}^L 8\sqrt{2d\log \left(\frac{1}{\beta}\right)} \left(\sqrt{h_1 2^{(1-\alpha){(l-1)}}}+ \frac{S}{\sqrt{h_12^{(\alpha+1){(l-1)}}}}\right) \\
		& + \sum_{l=1}^L 4\sigma\sqrt{2\log \left(\frac{1}{\beta}\right)} \left(h_1\sqrt{2^{(2-\alpha)(l-1)}}+ \frac{S}{\sqrt{2^{\alpha (l-1)}}}\right), \\
	\end{aligned}
\end{equation}
which is
\begin{equation}
	\begin{aligned}
		&R_g 
		\leq 2(h_1+S)\\
		&+ 8\sqrt{2d\log(1/\beta)} \left(C_0\sqrt{h_1 2^{(L-1)(1-\alpha)}}+\frac{S}{\sqrt{h_1}(\sqrt{2^{1+\alpha}}-1)}\right) \\
		&+ 4\sigma\sqrt{2\log(1/\beta)} \left(\frac{h_1\sqrt{2^{2-\alpha}}}{\sqrt{2^{2-\alpha}}-1}\cdot \sqrt{2^{(L-1)(2-\alpha)}}+C_1S\right)\\
		&= 2(h_1+S)\\
		&+ 8\sqrt{2d\log(1/\beta)}\cdot C_0\sqrt{h_12^{(L-1)(1-\alpha)}} + \frac{8S\sqrt{2d\log (1/\beta)}}{\sqrt{h_1}(\sqrt{2}-1)}\\
		&+ 4\sigma\sqrt{2\log(1/\beta)} \left(4h_1 \sqrt{2^{(L-1)(2-\alpha)}}+C_1S\right),
	\end{aligned}
\end{equation}
where $C_0=\frac{\sqrt{2^{1-\alpha}}}{\sqrt{2^{1-\alpha}}-1}$ and $C_1 = \sum_{l=2}^{\infty} \frac{1}{\sqrt{2^{\alpha(l-1)}}}$.
Note that $ h_L \leq T_L \leq T$, which indicates $2^{L-1} \leq T/h_1$, and $L\leq \log (2T/h_1)$.
Then, the above inequality becomes
\begin{equation*}
	\begin{aligned}
		R_g &= \sum_{l=1}^L r_l\\
		&\leq 2(h_1+S)+ 8\sqrt{2d\log(1/\beta)  } \cdot C_0h_1(\sqrt{T/h_1})^{1-\alpha} \\
		&+ 20S\sqrt{2d/h_1\log(1/\beta)}\\
		&+4\sigma\sqrt{2\log(1/\beta)} \left(4h_1\sqrt{(T/h_1)^{2-\alpha}}+C_1S\right)\\
		&\leq 2(h_1+S) + 8C_0\sqrt{2dh_1^{\alpha}T^{1-\alpha}\log(1/\beta)  }\\
		&+20S\sqrt{2d/h_1\log(1/\beta)}\\
		&+16\sigma\sqrt{2h_1^{\alpha}\log(1/\beta)}\cdot T^{1-\alpha/2} \\
		&+4C_1\sigma S\sqrt{2\log(1/\beta)}.
	\end{aligned}
\end{equation*}
On the other hand, $R_b \leq 2T$ 
since $\langle \theta^*, x^*-x \rangle \leq 2$ for all $x \in \cD$. 
Choose $\beta = \frac{1}{kT}$ in Algorithm~\ref{alg:dpe}. Finally, we have the following results:
\begin{equation*}
	\begin{aligned}
		&\E[R(T)] \\
		&= P(\mathcal{E}_g)R_g + (1-P(\mathcal{E}_g)) R_b \\
		& \leq R_g + 2k\beta L \cdot 2T\\
		& \leq 2(h_1+S)+ 8C_0\sqrt{2dh_1^{\alpha}T^{1-\alpha}\log(kT)  }\\
		&+20S\sqrt{2d/h_1\log(kT)}+16\sigma\sqrt{2h_1^{\alpha}\log(kT)}\cdot T^{1-\alpha/2}\\
		&+4C_1\sigma S \sqrt{2\log(kT)} + 4\log (2T/h_1)\\
		&= O(\sqrt{dT^{1-\alpha}\log (kT)}) +O(\sigma T^{1-\alpha/2}\sqrt{\log (kT)})\\
		&+O(d^{3/2}\sqrt{\log(kT)}),
	\end{aligned}
\end{equation*}
where the last equality is from $h_1=2$ and ignoring the logarithmic term regarding $d$ in $S$.

\smallskip

\textbf{3) Communication cost.} 
Notice that the communicating data in each phase is the local average performance $y_l^u(x)$ for each chosen action $x$ in the support set $\text{supp}(\pi_l)$. 
Therefore, the total communication cost is \begin{equation*}
	C(T) =\sum_{l=1}^L s_l |U_l|\leq \sum_{l=1}^L (4d\log \log d+16)\cdot 2^{\alpha l} =O(dT^\alpha).
\end{equation*}
\end{proof}
\subsection{Proof of Theorem~\ref{thm:regret_dp_dpe}} \label{app:proof_regret_dp_dpe}

\begin{table}[t]
\centering
\caption{$\sigma_n$ in Algorithm~\ref{alg:dpe} under different DP models}
\label{tab:parameter_setting_dpe}
\begin{tabular}{lcc}
	\toprule
	Algorithm & $\sigma_n$& Notes\\
	\midrule
	DPE& $0$  &  --  \\
	CDP-DPE & $\sigma_n = 2\sigma_{nc}\sqrt{Sd}$  & $\sigma_{nc} = \frac{2B\sqrt{2s_l\ln(1.25/\delta)}}{\epsilon |U_l|}$\\
	LDP-DPE& $\sigma_n = 2\sigma_{nl}\sqrt{\frac{Sd}{|U_l|}}$ & $\sigma_{nl} = \frac{2B\sqrt{2s_l\ln(1.25/\delta)}}{\epsilon }$\\
	SDP-DPE & $\sigma_n = 2\sigma_{ns}\sqrt{Sd}$ &$\sigma_{ns} = O\left(\frac{B\sqrt{d}\ln(d/\delta)}{\epsilon |U_l|}\right)$\\
	\bottomrule
\end{tabular}
\end{table}
To be clear, we list the parameter setting of Theorem~\ref{thm:regret_dp_dpe} 
in Table~\ref{tab:parameter_setting_dpe}. To prove it, we first show that the following concentration inequalities hold for the DP-DPE algorithm under the three DP models with the setting in Table~\ref{tab:parameter_setting_dpe}.  

\begin{theorem}\label{thm:concentration_dp}
Set DP-DPE in the central, local, and shuffle models (i.e., CDP-DPE, LDP-DPE, and SDP-DPE, respectively) based on Table~\ref{tab:parameter_setting_dpe}. In any phase $l$, all of CDP-DPE, LDP-DPE, and SDP-DPE satisfy the following concentration,
for any $x\in \cD_l$,
\begin{equation}
\begin{aligned}
	&P\left\{ \langle \Tilde{\theta}_l-\theta^*, x\rangle \geq W_l\right\}\leq 3\beta,  \quad \text{and} \\
	&P\left\{ \langle \theta^* - \Tilde{\theta}_l, x\rangle \geq W_l\right\}\leq 3\beta.    \label{eq:concentration_ineq_dp}
\end{aligned}
\end{equation} 
\end{theorem}
\begin{proof}
We prove the first concentration inequality in \eqref{eq:concentration_ineq_dp} for CDP-DPE, LDP-DPE, and SDP-DPE, respectively, in the following, and the second inequality can be proved symmetrically.
\smallskip

According to Line~\ref{alg_lse} in Algorithm~\ref{alg:dpe}, we have
\begin{equation*}
\begin{aligned}
	\Tilde{\theta}_l &= V_l^{-1}G_l\\
	& =  V_l^{-1}\sum_{x\in \text{supp}(\pi_l)}  x T_l(x)  \Tilde{y}_l(x) \\
	& =  V_l^{-1}\sum_{x\in \text{supp}(\pi_l)}  x T_l(x)  \frac{1}{|U_l|}\sum_{u\in U_l}y_l^u(x) \\
	&+  V_l^{-1}\sum_{x\in \text{supp}(\pi_l)}  x T_l(x)  \underbrace{\left( \Tilde{y}_l(x)-\frac{1}{|U_l|}\sum_{u\in U_l}y_l^u(x)\right)}_{\gamma_p(x)}\\
	& \overset{(a)}{=}  V_l^{-1}\sum_{x\in \text{supp}(\pi_l)}  x T_l(x)  \frac{1}{|U_l|}\sum_{u\in U_l}y_l^u(x) \\
	&+  V_l^{-1}\sum_{x\in \text{supp}(\pi_l)} xT_l(x)\gamma_p(x)\\
	& \overset{(b)}{=}  \frac{1}{|U_l|}\sum_{u\in U_l} \theta_u  +  \frac{1}{|U_l|}\sum_{u\in U_l}V^{-1} \sum_{t \in\cT_l}  \eta_{u,t}x_t \\&+  V_l^{-1}\sum_{x\in \text{supp}(\pi_l)} xT_l(x)\gamma_p(x),
\end{aligned}
\end{equation*}
where we denote the noise introduced  for privacy-preserving associated with action $x$ by $\gamma_p(x)\triangleq  \Tilde{y}_l(x)-\frac{1}{|U_l|}\sum_{u\in U_l}y_l^u(x)$, which varies according to the specified DP models, and $(b)$ is derived from Eq.~\eqref{eq:estimation_average_relation}.
Then, for any action $x'\in \cD$, the gap between the estimated reward with parameter $\Tilde{\theta}_l$ and the true reward with $\theta^*$ satisfies
$$
\begin{aligned}
& \langle \Tilde{\theta}_l-\theta^*, x'\rangle \\
&=\left\langle \Tilde{\theta}_l- \frac{1}{|U_l|}\sum_{u\in U_l}\theta_u + \frac{1}{|U_l|}\sum_{u\in U_l}\theta_u -\theta^*, x'\right\rangle \\ 
&= \left\langle  \Tilde{\theta}_l-\frac{1}{|U_l|}\sum_{u\in U_l} \theta_u, x' \right\rangle  +  \frac{1}{|U_l|}  \sum_{u\in U_l} \left\langle \theta_u - \theta^*, x'\right\rangle \\
& =  \frac{1}{|U_l|}\sum_{u\in U_l}  \left\langle   V_l^{-1}  \sum_{t \in \cT_l}  \eta_{u,t} x_t, x'\right\rangle \\
& + \left\langle V_l^{-1}\sum_{x\in \text{supp}(\pi_l)} xT_l(x)\gamma_p(x), x'\right\rangle \\
&+  \frac{1}{|U_l|}  \sum_{u\in U_l} \left\langle \theta_u - \theta^*, x'\right\rangle.
\end{aligned}$$
Then, we have
\begin{equation}
\begin{aligned}
	& P\left\{ \langle \Tilde{\theta}_l-\theta^*, x\rangle \geq W_l\right\} \\
	\leq & P\left\{ \frac{1}{|U_l|}\sum_{u\in U_l}  \left\langle x',   V_l^{-1}  \sum_{t \in \cT_l}  \eta_{u,t} x_t\right\rangle \geq \sqrt{\frac{4d}{h_l|U_l|}\log \left(\frac{1}{\beta}\right)}\right\} \\
	+& P\left\{ \frac{1}{{|U_l|}}\sum_{u\in U_l}  \langle \theta_u - \theta^*, x\rangle \geq \sqrt{\frac{2\sigma^2}{|U_l|}\log \left(\frac{1}{\beta}\right)} \right\} \\
	+&P\left\{ \left\langle x', V_l^{-1}\sum_{x\in \text{supp}(\pi_l)} xT_l(x)\gamma_p(x)\right\rangle\geq \sqrt{2\sigma_n^2\log \left(\frac{1}{\beta}\right)} \right\}.  
\end{aligned}
\end{equation}
We have shown that either the first or the second probability in the above equation is less than $\beta$ in Eq.~\eqref{eq:W_l1_prob_bnd} and Eq.~\eqref{eq:W_l2_prob_bnd}, respectively. In the following, we try to show that the third term is less than $\beta$ under each of the three DP models.

\smallskip

\textbf{i) CDP-DPE.}
Under the DP-DPE algorithm with the central model \textsc{Privatizer}, the output of  the \textsc{Privatizer} $\cP$ is, $\Tilde{y}_l
= \frac{1}{|U_l|}\sum_{u\in U_l} \Vec{y}_{l}^u +(\gamma_1, \dots, \gamma_{s_l})$, where $\gamma_j \overset{\emph{i.i.d.}}{\sim} \mathcal{N}(0,\sigma^2_{nc})$. Then, $\gamma_p(x) \sim \cN(0, \sigma_{nc}^2)$ in the central model and is \emph{i.i.d.} across actions $x\in \text{supp}(\pi_l)$.  
Note that 
\begin{equation}
\begin{aligned}
	&\left\langle x', V_l^{-1}\sum_{x\in \text{supp}(\pi_l)} xT_l(x)\gamma_{p}(x)\right\rangle \\
	&= \sum_{x\in \text{supp}(\pi_l)}  \left\langle x', V_l^{-1} x\right\rangle  T_l(x)\gamma_{p}(x), \label{eq:gaussian_sum}    
\end{aligned}
\end{equation}
and $\gamma_{p}(x) \overset{\emph{i.i.d.}}{\sim} \cN(0, \sigma_{nc}^2)$. The variance (denoted by $\sigma_{\text{sum}}^2$) of the above sum of \emph{i.i.d.} Gaussian variables is 
$$
\begin{aligned}
\sigma_{\text{sum}}^2 &= \sum_{x\in \text{supp}(\pi_l)}  \left\langle x', V_l^{-1} x\right\rangle^2  T_l(x)^2 \sigma_{nc}^2\\
&\overset{(a)}{\leq} T_l  \cdot   x'^{\top} V_l^{-1} \left(\sum_{x\in \text{supp}(\pi_l)} T_l(x)x x^{\top}\right)V_l^{-1}x' \sigma_{nc}^2 \\\
&= T_l\Vert x'\Vert_{V_l^{-1}}^2 \sigma_{nc}^2,
\end{aligned}
$$
where $(a)$ is from $T_l(x) \leq T_l$ for any $x$ in the support set $\text{supp}(\pi_l)$. 
Therefore, the LHS of Eq.~\eqref{eq:gaussian_sum} is a Gaussian variable with variance 
$$
\begin{aligned}
\sigma_{\text{sum}}^2 &\leq T_l\Vert x'\Vert_{V_l^{-1}}^2\sigma_{nc}^2\leq T_l\cdot \frac{2d}{h_l} \cdot \sigma_{nc}^2 \\
&\overset{(a)}{\leq} \left(1+\frac{S}{h_l}\right)2d\sigma_{nc}^2\leq 4Sd\sigma_{nc}^2 = \sigma_n^2,    
\end{aligned}$$ 
where $(a)$ is due to $T_l=\sum_{x\in\text{supp}(\pi_l)}T_l(x) \leq h_l + |\text{supp}(\pi_l)| \leq h_l+S$, and the last step is based on our setting in Table~\ref{tab:parameter_setting_dpe}.
Combining the tail bound for Gaussian variables, we have
$$
\begin{aligned}
&P\left\{ \left\langle x', V_l^{-1}\sum_{x\in \text{supp}(\pi_l)} xT_l(x)\gamma_p(x)\right\rangle\geq \sqrt{2\sigma_n^2\log \left(\frac{1}{\beta}\right)} \right\} \\
&\leq \exp \left\{-\frac{2\sigma_n^2\log \left(\frac{1}{\beta}\right)} {2\sigma_{\text{sum}}^2 }\right\}\leq \beta.
\end{aligned}
$$
Hence, the first concentration inequality in Eq.~\eqref{eq:concentration_ineq_dp} holds for CDP-DEP algorithm.

\smallskip
\textbf{ii) LDP-DPE. }
Under the DP-DPE algorithm with the local model \textsc{Privatizer}, the output of  $\cP$ is, $\Tilde{y}_l
= \frac{1}{|U_l|}\sum_{u\in U_l} ( \Vec{y}_{l}^u+(\gamma_{u,1}, \dots, \gamma_{u,s_l}) ) $, where $\gamma_{u,j} \overset{\emph{i.i.d.}}{\sim} \mathcal{N}(0,\sigma^2_{nl})$. Let $j_x$ denote the index corresponding to the action $x$ in the support set $\text{supp}(\pi_l)$, i.e., $\Tilde{y}_l(x)=\frac{1}{|U_l|}\sum_{u\in U_l} (y_l^u(x)+\gamma_{u,j_x})$ and $\gamma_p(x) = \frac{1}{|U_l|}\sum_{u\in U_l}\gamma_{u,j_x}$ in the local model. Then,  
\begin{equation}
\begin{aligned}
	&\left\langle x', V_l^{-1}\sum_{x\in \text{supp}(\pi_l)} xT_l(x)\gamma_{p}(x)\right\rangle \\
	&= \frac{1}{|U_l|}\sum_{u\in U_l}\sum_{x\in \text{supp}(\pi_l)}  \left\langle x', V_l^{-1} x\right\rangle  T_l(x)\gamma_{u,j_x}.\label{eq:gamma_p_relationship} 
\end{aligned}
\end{equation}
Consider the sum at client $u$ first, i.e., 
\begin{equation}
\sum_{x\in \text{supp}(\pi_l)}  \left\langle x', V_l^{-1} x\right\rangle  T_l(x)\gamma_{u,j_x}. \label{eq:gaussian_sum_u}
\end{equation}
Recall that $\gamma_{u,j_x} \overset{\emph{i.i.d.}}{\sim} \cN(0, \sigma_{nl}^2)$. The variance (denoted by $\sigma_{u,\text{sum}}^2$) of the above sum of \emph{i.i.d.} Gaussian variables at client $u$ is 
$$
\begin{aligned}
\sigma_{u,\text{sum}}^2 &= \sum_{x\in \text{supp}(\pi_l)}  \left\langle x', V_l^{-1} x\right\rangle^2  T_l(x)^2 \sigma_{nl}^2\\
&\overset{(a)}{\leq} T_l  \cdot   x'^{\top} V_l^{-1} \left(\sum_{x\in \text{supp}(\pi_l)} T_l(x)x x^{\top}\right)V_l^{-1}x' \sigma_{nl}^2 \\
&= T_l\Vert x'\Vert_{V_l^{-1}}^2 \sigma_{nl}^2,    
\end{aligned}
$$
where $(a)$ is from $T_l(x) \leq T_l$ for any $x$ in the support set $\text{supp}(\pi_l)$. 
Therefore, the term in Eq.~\eqref{eq:gaussian_sum_u} is a Gaussian variable with variance 
$$
\begin{aligned}
\sigma_{u,\text{sum}}^2 &\leq T_l\Vert x'\Vert_{V_l^{-1}}^2\sigma_{nl}^2\leq T_l\cdot \frac{2d}{h_l} \cdot \sigma_{nl}^2 \\
&\overset{(a)}{\leq} \left(1+\frac{S}{h_l}\right)2d\sigma_{nl}^2\leq 4Sd\sigma_{nl}^2 = |U_l|\sigma_n^2,   
\end{aligned}$$ 
where $(a)$ is due to $T_l=\sum_{x\in\text{supp}(\pi_l)}T_l(x) \leq h_l + |\text{supp}(\pi_l)| \leq h_l+S$, and the last step is based on our setting in Table~\ref{tab:parameter_setting_dpe}.
Combining the tail bound for Gaussian variables, we have
$$
\begin{aligned}
&P\left\{ \frac{1}{|U_l|}\sum_{\mathclap{u\in U_l}}\sum_{{x\in \text{supp}(\pi_l)}} \left\langle x', V_l^{-1} xT_l(x)\gamma_{j_x}\right\rangle\geq \sqrt{2\sigma_n^2\log \left(\frac{1}{\beta}\right)} \right\}\\ 
&\leq \exp \left\{-\frac{|U_l|2\sigma_n^2\log \left(\frac{1}{\beta}\right)} {2\sigma_{u,\text{sum}}^2 }\right\}  \leq \beta.   \label{eq:concentration_privacy_CDP}
\end{aligned}
$$ Finally, based on Eq.~\eqref{eq:gamma_p_relationship}, we have
$$
\begin{aligned}
&P\left\{ \left\langle x', V_l^{-1}\sum_{x\in \text{supp}(\pi_l)} xT_l(x)\gamma_p(x)\right\rangle\geq \sqrt{2\sigma_n^2\log \left(\frac{1}{\beta}\right)} \right\} \\
&\leq \beta.    
\end{aligned}
$$
Hence, the first concentration inequality in Eq.~\eqref{eq:concentration_ineq_dp} holds for LDP-DEP algorithm.

\smallskip

\textbf{iii) SDP-DPE.}
Under the DP-DPE algorithm with the shuffle model \textsc{Privatizer}, the output of  $\cP$ is, $\Tilde{y}_l =   (\cA \circ \cS\circ \cR^{|U_l|} )(\{\Vec{y}_l^{u}\}_{u\in U_l}) = \cA(\mathcal{S} (\{\cR(\Vec{y}_l^{u})\}_{u\in U_l}))$, where $\cA, \cS$ and $\cR$ follow Algorithm~\ref{alg:shuffle_vec}. From Theorem~\ref{thm:vector_shuffle}, we know that the output of $(\cA \circ \cS\circ \cR^{|U_l|})(\{\Vec{y}_l^{u}\}_{u\in U_l})$ is an unbiased estimator of the average of the $|U_l|$ input vectors $\{\Vec{y}_l^{u}\}_{u\in U_l}$ and that the error distribution is sub-Gaussian with variance $\sigma_{ns}^2 = O\left(\frac{B^2s_l\ln^2(s_l/\delta)}{\epsilon^2 |U_l|^2}\right)$. Then, $\gamma_p(x) = \Tilde{y}_l(x) - \frac{1}{|U_l|}\sum_{u\in U_l} y_l^u(x)$ is $\sigma_{ns}$-sub-Gaussian with $\E[\gamma_p(x)] = 0$ in the shuffle model.
Besides, $\gamma_p(x)$ is \emph{i.i.d.} over each coordinate corresponding to each action $x$ in the support set $\text{supp}(\pi_l)$. Note that 
\begin{equation}
\begin{aligned}
	&\left\langle x', V_l^{-1}\sum_{x\in \text{supp}(\pi_l)} xT_l(x)\gamma_{p}(x)\right\rangle \\
	&= \sum_{x\in \text{supp}(\pi_l)}  \left\langle x', V_l^{-1} x\right\rangle  T_l(x)\gamma_{p}(x). \label{eq:sub-gaussian_sum}  
\end{aligned}
\end{equation}
The variance (denoted by $\sigma_{\text{sub-G}}^2$) of the above sum of \emph{i.i.d.} sub-Gaussian variables is 
$$
\begin{aligned}
\sigma_{\text{sub-G}}^2 &= \sum_{x\in \text{supp}(\pi_l)}  \left\langle x', V_l^{-1} x\right\rangle^2  T_l(x)^2 \sigma_{ns}^2\\
&\overset{(a)}{\leq} T_l  \cdot   x'^{\top} V_l^{-1} \left(\sum_{x\in \text{supp}(\pi_l)} T_l(x)x x^{\top}\right)V_l^{-1}x' \sigma_{ns}^2 \\
&= T_l\Vert x'\Vert_{V_l^{-1}}^2 \sigma_{ns}^2,     
\end{aligned}
$$
where $(a)$ is from $T_l(x) \leq T_l$ for any $x$ in the support set $\text{supp}(\pi_l)$. 
Therefore, the LHS of Eq.~\eqref{eq:sub-gaussian_sum} is $\sigma_{\text{sub-G}}$-sub-Gaussian variable with variance proxy
$$
\begin{aligned}
\sigma_{\text{sub-G}}^2 &\leq T_l\Vert x'\Vert_{V_l^{-1}}^2\sigma_{ns}^2\leq T_l\cdot \frac{2d}{h_l} \cdot \sigma_{ns}^2 \\
&\overset{(a)}{\leq} \left(1+\frac{S}{h_l}\right)2d\sigma_{ns}^2\leq 4Sd\sigma_{ns}^2 = \sigma_n^2,    
\end{aligned}$$ 
where $(a)$ is due to $T_l=\sum_{x\in\text{supp}(\pi_l)}T_l(x) \leq h_l + |\text{supp}(\pi_l)| \leq h_l+S$, and the last step is based on our setting in Table~\ref{tab:parameter_setting_dpe}.
Combining the property for sub-Gaussian variables, we have
$$
\begin{aligned}
& P\left\{ \left\langle x', V_l^{-1}\sum_{x\in \text{supp}(\pi_l)} xT_l(x)\gamma_p(x)\right\rangle\geq \sqrt{2\sigma_n^2\log \left(\frac{1}{\beta}\right)} \right\} \\
&\leq \exp \left\{-\frac{2\sigma_n^2\log \left(\frac{1}{\beta}\right)} {2\sigma_{\text{sum}}^2 }\right\}\leq \beta.   
\end{aligned}
$$
Hence, the first concentration inequality in Eq.~\eqref{eq:concentration_ineq_dp} holds for SDP-DEP algorithm.

\smallskip 
With the symmetrical arguments under the three DP models, we derive the results in Eq.~\eqref{eq:concentration_ineq_dp}.
\smallskip
\end{proof}

To prove Theorem~\ref{thm:regret_dp_dpe}, 
we follow a similar line  to the proof of Theorem~\ref{thm:regret_approx} by first analyzing the regret incurred in a particular phase with high probability, then summing up the regret over all phases, and finally analyzing the communication cost associated with each instantiation of the DP-DPE algorithm. 

Before getting into the proof for each DP-DPE instantiation, we present two important observations under the ``good" event at $l$-th phase. Recall the ``good" event definition at the $l$-th phase $\mathcal{E}_l$: 
$$
\begin{aligned}
&\left\{\langle \theta^*- \Tilde{\theta}_l, x^*\rangle \leq W_l\right\} \quad \text{and}, \\
&\left\{\langle \Tilde{\theta}_l-\theta^*, x\rangle \leq W_l  \right\}, \quad \forall x\in \cD\backslash \{x^*\} .
\end{aligned}
$$
According to Theorem~\ref{thm:concentration_dp}, it is not difficult to derive $P(\mathcal{E}_l)\geq 1-3k\beta$ for the DP-DPE algorithm. In addition, under event $\mathcal{E}_l$, we can still have the following two observations for all the DP-DPE instantiations:
\begin{itemize}
\item[\textbf{1.}] If the optimal action $x^*\in \cD_l$, then $x^* \in \cD_{l+1}$. 
\item[\textbf{2.}] For any $x\in \cD_{l+1}$, we have $\langle \theta^*, x^* - x\rangle \leq 4W_l$. 
\end{itemize}
For any $l=1,\dots, L-1$, according to the second observation under event $\cE_{l}$, we have the regret incurred in the $(l+1)$-th phase satisfies
\begin{equation}
\begin{aligned}\notag
&r_{l+1} \triangleq \sum_{t \in \cT_{l+1}}  \langle \theta^*, x^* - x_{t} \rangle \\
&\leq \sum_{t\in\cT_{l+1}} 4W_{l}\\
&=  \!\sum_{\mathclap{t\in\cT_{l+1}}} 4\left(\sqrt{\frac{4d}{h_l|U_l|}}+\sqrt{\frac{2\sigma^2}{|U_l|}}+ \sqrt{2\sigma_n^2}\right)\sqrt{\log \left(\frac{1}{\beta}\right)}\\
& = \underbrace{4T_{l+1}\sqrt{\frac{4d}{h_l|U_l|}\log \left(\frac{1}{\beta}\right)}}_{\text{\textcircled{1}}}+\underbrace{4T_{l+1}\sqrt{\frac{2\sigma^2}{|U_l|}\log \left(\frac{1}{\beta}\right)}}_{\text{\textcircled{2}}}\\
&+ \underbrace{4T_{l+1}\sqrt{2\sigma_n^2\log \left(\frac{1}{\beta}\right)}}_{\text{\textcircled{3}}}.
\end{aligned}
\end{equation}
We have shown that \textcircled{1} is bounded by 
\begin{equation}
\begin{aligned}\notag
\text{\textcircled{1}}& \leq  8\sqrt{2d\log \left(\frac{1}{\beta}\right)} \left(\sqrt{h_1\cdot 2^{(1-\alpha)l}}+ \frac{S}{\sqrt{h_1\cdot 2^{(\alpha+1)l}}}\right),
\end{aligned}
\end{equation}
and that \textcircled{2} is bounded by
\begin{equation}
\begin{aligned}\notag
\text{\textcircled{2}}
& \leq  4\sigma\sqrt{2\log \left(\frac{1}{\beta}\right)} \left(h_1\sqrt{2^{(2-\alpha) l}}+ \frac{S}{\sqrt{2^{\alpha l}}}\right).
\end{aligned}
\end{equation}
Regarding the third term \textcircled{3}, it varies according to different DP models. In the following, we analyze the term \textcircled{3} in the central, local, and shuffle model respectively. 

\smallskip

\textbf{i) CDP-DPE. } In the central model, $\sigma_n = 2\sigma_{nc}\sqrt{Sd}$ where  $\sigma_{nc} = \frac{2B\sqrt{2s_l\ln(1.25/\delta)}}{\epsilon | U_l|}$. Let $\sigma_0 \triangleq \frac{2B\sqrt{2\ln(1/\delta)}}{\epsilon}$. Then, $\sigma_n = \frac{2\sigma_0\sqrt{s_ldS}}{|U_l|}\leq \frac{ 2\sigma_0 S\sqrt{d}}{|U_l|}$ since $s_l\leq S$ for any $l$. Combining $|U_l|\geq 2^{\alpha l}$, we have 
\begin{equation}
\begin{aligned}
\text{\textcircled{3}}
& \leq   4(h_1\cdot2^l+S)\sqrt{\frac{8\sigma_0^2S^2d}{2^{2\alpha l}}\log \left(\frac{1}{\beta}\right)}\\
& \leq  8\sigma_{0}S\sqrt{2d\log\left(\frac{1}{\beta}\right)} \left(h_1\cdot 2^{(1-\alpha)l}+ \frac{S}{2^{\alpha l}}\right).\notag
\end{aligned}
\end{equation}

\smallskip
\textbf{ii) LDP-DPE. } In the local model, $\sigma_n = 2\sigma_{nl}\sqrt{\frac{Sd}{|U_l|}}$ where $\sigma_{nl} =
\frac{2B\sqrt{2s_l\ln(1.25/\delta)}}{\epsilon}$. Substituting $\sigma_0$, we have $\sigma_n = \frac{2\sigma_0\sqrt{s_ldS}}{\sqrt{|U_l|}}\leq \frac{2\sigma_0 S\sqrt{d}}{\sqrt{|U_l|}}$ and then derive
\begin{equation}
\begin{aligned}
\text{\textcircled{3}}
& \leq   4(h_1\cdot2^l+S)\sqrt{\frac{8\sigma_0^2S^2d}{2^{\alpha l}}\log \left(\frac{1}{\beta}\right)}\\
& \leq  8\sigma_{0}S\sqrt{2d\log\left(\frac{1}{\beta}\right)} \left(h_1\cdot\sqrt{2^{(2-\alpha)l}}+ \frac{S}{\sqrt{2^{\alpha l}}}\right).\notag
\end{aligned}
\end{equation}

\smallskip
\textbf{iii) SDP-DPE.} In the shuffle model, $\sigma_n = 2\sigma_{ns}\sqrt{Sd}$, where $\sigma_{ns} = O\left(\frac{B\sqrt{s_l}\ln(s_l/\delta)}{\epsilon |U_l|}\right)$. Let $\sigma_{ns} =\frac{C_s B\sqrt{s_l}\ln(s_l/\delta)}{\epsilon |U_l|} $  and $\sigma_0^{\prime} \triangleq \frac{C_sB\ln(S/\delta)}{\epsilon}$. Then, we have $\sigma_n \leq \frac{2\sigma_0^{\prime}S\sqrt{d}}{|U_l|}$ and
\begin{equation}
\begin{aligned}
\text{\textcircled{3}}
& \leq   4(h_1\cdot2^l+S)\sqrt{\frac{8{\sigma_0^{\prime}}^2S^2d}{2^{2\alpha l}}\log \left(\frac{1}{\beta}\right)}\\
& \leq  8\sigma_0^{\prime}S\sqrt{2d\log\left(\frac{1}{\beta}\right)} \left(h_1\cdot 2^{(1-\alpha)l}+ \frac{S}{2^{\alpha l}}\right).\notag
\end{aligned}
\end{equation}

\begin{proof}[Proof of Theorem~\ref{thm:regret_dp_dpe} [CDP-DPE]
Under the ``good" event, the regret in the $(l+1)$-th phase satisfies
\begin{equation}
\begin{aligned}
r_{l+1} &\leq \text{\textcircled{1}} + \text{\textcircled{2}} + \text{\textcircled{3}}\\
& \leq  8\sqrt{2d\log \left(\frac{1}{\beta}\right)} \left(\sqrt{h_1\cdot 2^{(1-\alpha)l}}+ \frac{S}{\sqrt{h_1\cdot 2^{(\alpha+1)l}}}\right) \\
& + 4\sigma\sqrt{2\log \left(\frac{1}{\beta}\right)} \left(h_1\sqrt{2^{(2-\alpha) l}}+ \frac{S}{\sqrt{2^{\alpha l}}}\right)\\
& + 8\sigma_{0}S\sqrt{2d\log\left(\frac{1}{\beta}\right)} \left(h_1\cdot 2^{(1-\alpha)l}+ \frac{S}{2^{\alpha l}}\right).
\end{aligned}
\end{equation}
Assume the ``good" event hold in every phase, i.e., under event $\cE_g = \bigcap_{l=1}^L\cE_l$. We have $P\{\cE_g\}\geq 1-3k\beta L$ by applying union bound and the total regret $R_g$ under $\cE_g$ satisfies
\begin{equation*}
\begin{aligned}
R_g &=  \sum_{l=1}^L r_l \\
\leq & 2(h_1+S) + \sum_{l=2}^{L} r_{l}\\
\leq & 2(h_1 + S) \\
&+ \!\sum_{l=2}^L 8\sqrt{2d\log \left(\frac{1}{\beta}\right)} \left(\sqrt{h_1 2^{(1-\alpha)(l-1)}}+ \frac{S}{\sqrt{h_1 2^{(\alpha+1)(l-1)}}}\right) \\
& + \sum_{l=2}^L 4\sigma\sqrt{2\log \left(\frac{1}{\beta}\right)} \left(h_1\sqrt{2^{(2-\alpha) (l-1)}}+ \frac{S}{\sqrt{h_12^{\alpha (l-1)}}}\right)\\
& + \sum_{l=2}^L 8\sigma_{0}S\sqrt{2d\log\left(\frac{1}{\beta}\right)} \left(h_1\cdot 2^{(1-\alpha)(l-1)}+ \frac{S}{2^{\alpha (l-1)}}\right)\\
\overset{(a)}{\leq} & 2(h_1+S)+ 8\sqrt{2d\log(1/\beta)}\cdot C_0\sqrt{h_12^{(L-1)(1-\alpha)}}\\
&+ \frac{8S\sqrt{2d\log (1/\beta)}}{\sqrt{h_1}(\sqrt{2}-1)}\\
&+ 4\sigma\sqrt{2\log(1/\beta)} \left(4h_1 \sqrt{2^{(L-1)(2-\alpha)}}+C_1S\right)\\
& + 8\sigma_0 S\sqrt{2d\log(1/\beta)}\left(\frac{h_1\cdot 2^{(1-\alpha)(L-1)}}{2^{1-\alpha}-1} +C_2 S\right)\\
\overset{(b)}{\leq} & 2(h_1+S)+ 8\sqrt{2d\log(1/\beta)}\cdot C_0\sqrt{h_1}(\sqrt{T/h_1})^{1-\alpha} \\
&+ 20S\sqrt{2d/h_1\log (1/\beta)}\\
&+ 4\sigma\sqrt{2\log(1/\beta)} \left(4h_1\sqrt{(T/h_1)^{2-\alpha}}+C_1S\right)\\
& + 8\sigma_0 S\sqrt{2d\log(1/\beta)}\left(\frac{h_1\cdot (T/h_1)^{1-\alpha}}{2^{1-\alpha}-1} +C_2 S\right)\\
\overset{(c)}{\leq} & 2(h_1+S)+ 8C_0\sqrt{2dh_1^{\alpha}T^{1-\alpha}\log(kT)} \\
&+ 20S\sqrt{2d/h_1\log (kT)}\\
&+ 16\sigma\sqrt{2h_1^{\alpha}\log(kT)} T^{1-\alpha/2}+ 4C_1S\sigma\sqrt{2\log(kT)}\\
& + 8\sigma_0 h_1^{\alpha} S\sqrt{2d\log(kT)}\left(\frac{T^{1-\alpha}}{2^{1-\alpha}-1} +C_2 S\right)\\
= & O\left(\sqrt{dT^{1-\alpha}\log(kT)}\right) + O\left(\sigma T^{1-\alpha/2}\sqrt{\log(kT)}\right) \\
&+O\left(\sigma_0 T^{1-\alpha}S\sqrt{d\log(kT)}\right)+O\left(\sigma_0S^2\sqrt{d\log(kT)}\right)\\
\overset{(d)}{=} & 
O\left(\sigma T^{1-\alpha/2}\sqrt{\log(kT)}\right) \\&+O\left(\frac{Bd^{3/2} T^{1-\alpha}\sqrt{\ln(1/\delta)\log(kT)}}{\epsilon}\right)\\
&+O\left(\frac{Bd^{5/2}\sqrt{\ln(1/\delta)\log(kT)}}{\epsilon}\right),
\end{aligned}
\end{equation*}
where in $(a)$ $C_0=\frac{\sqrt{2^{1-\alpha}}}{\sqrt{2^{1-\alpha}}-1}$, $C_1 = \sum_{l=2}^{\infty} \frac{1}{\sqrt{2^{\alpha(l-1)}}}$ and $C_2 = \sum_{l=2}^{\infty}\frac{1}{2^{\alpha(l-1)}}$
,  $(b)$ is from $h_1\cdot 2^{L-1} = h_L \leq T_L \leq T$, $(c)$ is by setting $\beta=\frac{1}{kT}$, and $(d)$ is derived by substituting $\sigma_0 = \frac{2B\sqrt{2\ln(1.25/\delta)}}{\epsilon}$ and $S=4d\log\log d+16$. 

Finally, the expected regret of CDP-DPE algorithm is upper bounded by 
\begin{equation*}
\begin{aligned}
&\E[R(T)] = P(\mathcal{E}_g)R_g + (1-P(\mathcal{E}_g)) \cdot 2T \\
& \leq R_g + 2\beta k L \cdot 2T \\
&=O\left(\sigma T^{1-\alpha/2}\sqrt{\log(kT)}\right) \\
&+O\left(\frac{Bd^{3/2} T^{1-\alpha}\sqrt{\ln(1/\delta)\log(kT)}}{\epsilon}\right)\\
&+O\left(\frac{Bd^{5/2}\sqrt{\ln(1/\delta)\log(kT)}}{\epsilon}\right).
\end{aligned}
\end{equation*}

\medskip

\noindent\textbf{Communication cost.} Notice that the communicating data in each phase is the local average performance $y_l^u(x)$ for each chosen action $x$ in the support set $\text{supp}(\pi_l)$. 
Therefore, the total communication cost is 
\begin{equation*}
C(T) =\sum_{l=1}^L s_l |U_l|\leq \sum_{l=1}^L (4d\log \log d+16)\cdot 2^{\alpha l} =O(dT^\alpha).
\end{equation*}
\end{proof}
\begin{proof}[Proof of Theorem~\ref{thm:regret_dp_dpe}[LDP-DPE]
Under the ``good" event, the regret in the $(l+1)$-th phase satisfies
\begin{equation*}
\begin{aligned}
r_{l+1} &\leq \text{\textcircled{1}} + \text{\textcircled{2}} + \text{\textcircled{3}}\\
& \leq  8\sqrt{2d\log \left(\frac{1}{\beta}\right)} \left(\sqrt{h_1\cdot 2^{(1-\alpha)l}}+ \frac{S}{\sqrt{h_1\cdot 2^{(\alpha+1)l}}}\right) \\
& + 4\sigma\sqrt{2\log \left(\frac{1}{\beta}\right)} \left(h_1\sqrt{2^{(2-\alpha) l}}+ \frac{S}{\sqrt{2^{\alpha l}}}\right)\\
& +    8\sigma_{0}S\sqrt{2d\log\left(\frac{1}{\beta}\right)} \left(h_1\cdot\sqrt{2^{(2-\alpha)l}}+ \frac{S}{\sqrt{2^{\alpha l}}}\right).
\end{aligned}
\end{equation*}
Assume the ``good" event hold in every phase, i.e., under event $\cE_g = \bigcap_{l=1}^L\cE_l$. We have $P\{\cE_g\}\geq 1-3k\beta L$ by applying union bound. Then, the total regret satisfies
\begin{equation*}
\begin{aligned}
R_g& =  \sum_{l=1}^L r_l \\
\leq & 2(h_1 + S) \\
&+ \!\sum_{l=2}^L 8\sqrt{2d\log \left(\frac{1}{\beta}\right)} \left(\sqrt{h_1 2^{(1-\alpha)(l-1)}}+ \frac{S}{\sqrt{h_12^{(\alpha+1)(l-1)}}}\right) \\
& + \!\sum_{l=2}^L 4\sigma\sqrt{2\log \left(\frac{1}{\beta}\right)} \left(h_1\sqrt{2^{(2-\alpha) (l-1)}}+ \frac{S}{\sqrt{2^{\alpha (l-1)}}}\right)\\
& + \!\sum_{l=2}^L 8\sigma_0S\sqrt{d}\sqrt{2\log \left(\frac{1}{\beta}\right)} \left(h_1\sqrt{2^{(2-\alpha) (l-1)}}+ \frac{S}{\sqrt{2^{\alpha (l-1)}}}\right)\\
\overset{(a)}{\leq} & 2(h_1+S)\\
&+ 8\sqrt{2d\log(1/\beta)}\cdot C_0\sqrt{h_12^{(L-1)(1-\alpha)}} + \frac{8S\sqrt{2d\log (1/\beta)}}{\sqrt{h_1}(\sqrt{2}-1)}\\
&+ 4\sigma\sqrt{2\log(1/\beta)} \left(4h_1 \sqrt{2^{(L-1)(2-\alpha)}}+C_1S\right)\\
&+ 8\sigma_0S\sqrt{d}\sqrt{2\log(1/\beta)} \left(4h_1 \sqrt{2^{(L-1)(2-\alpha)}}+C_1S\right),
\end{aligned}
\end{equation*}
where in $(a)$ $C_0=\frac{\sqrt{2^{1-\alpha}}}{\sqrt{2^{1-\alpha}}-1}$ and $C_1 = \sum_{l=2}^{\infty} \frac{1}{\sqrt{2^{\alpha(l-1)}}}$. Then, 
\begin{equation*}
\begin{aligned}
R_g 
&\leq 2(h_1+S)\\
&+ 8\sqrt{2d\log(1/\beta)}\cdot C_0\sqrt{h_1}(\sqrt{T/h_1})^{1-\alpha} \\
&+ 20S\sqrt{2d/h_1\log (1/\beta)}\\
&+ (4\sigma+8\sigma_0 S\sqrt{d})\sqrt{2\log(1/\beta)} \left(4h_1\sqrt{(T/h_1)^{2-\alpha}}+C_1S\right)\\
\overset{(a)}{\leq} & 2(h_1+S)\\
&+ 8C_0\sqrt{2dh_1^{\alpha}T^{1-\alpha}\log(kT)}\\
&+ 20S\sqrt{2d/h_1\log (kT)}\\
&+ (16\sigma+32\sigma_0 S\sqrt{d})\sqrt{2h_1^{\alpha}\log(kT)} T^{1-\alpha/2}\\
&+4C_1S\sigma\sqrt{2\log(kT)}+8C_1\sigma_0S^2\sqrt{2d\log(kT)}\\
= & O\left(\sqrt{dT^{1-\alpha}\log(kT)}\right) \\&+ O\left(\sigma T^{1-\alpha/2}\sqrt{\log(kT)}\right) \\& +O\left(\sigma_0 T^{1-\alpha/2}S\sqrt{d\log(kT)}\right)\\&+O\left(\sigma_0S^2\sqrt{d\log(kT)}\right)\\
= & 
O\left(\sigma T^{1-\alpha/2}\sqrt{\log(kT)}\right) \\&+O\left(\frac{Bd^{3/2} T^{1-\alpha/2}\sqrt{\ln(1/\delta)\log(kT)}}{\epsilon}\right)\\&+O\left(\frac{Bd^{5/2}\sqrt{\ln(1/\delta)\log(kT)}}{\epsilon}\right),
\end{aligned}
\end{equation*}
where the first inequality is from $h_1\cdot 2^{L-1} = h_L \leq T_L \leq T$, $(a)$ is by setting $\beta=\frac{1}{kT}$, and the last equation is derived by substituting $\sigma_0 = \frac{2B\sqrt{2\ln(1.25/\delta)}}{\epsilon}$ and $S=4d\log\log d+16$. 

\smallskip 

Finally, the expected regret of LDP-DPE algorithm is upper bounded by 
\begin{equation}
\begin{aligned}
&\E[R(T)] = P(\mathcal{E}_g)R_g + (1-P(\mathcal{E}_g)) \cdot 2T \\
& \leq R_g + 2\beta k L \cdot 2T \\
&=O\left(\sigma T^{1-\alpha/2}\sqrt{\log(kT)}\right) \\&+O\left(\frac{Bd^{3/2} T^{1-\alpha/2}\sqrt{\ln(1/\delta)\log(kT)}}{\epsilon}\right)\\&+O\left(\frac{Bd^{5/2}\sqrt{\ln(1/\delta)\log(kT)}}{\epsilon}\right).
\end{aligned}
\end{equation}

\medskip

\noindent\textbf{Communication cost.} In the local model, the communicating data in each phase is the private local average reward $(y_l^u(x)+\gamma_{u,j_x})$ for each chosen action $x$ in the support set $\text{supp}(\pi_l)$. It is still an $s_l$-dimensional vector from each client. 
Therefore, the total communication cost is 
\begin{equation*}
C(T) =\sum_{l=1}^L s_l |U_l|\leq \sum_{l=1}^L (4d\log \log d+16)\cdot 2^{\alpha l} =O(dT^\alpha).
\end{equation*}

\end{proof}
\begin{proof}[Proof of Theorem~\ref{thm:regret_dp_dpe}[SDP-DPE]
Under the ``good" event, the regret in the $(l+1)$-th phase satisfies
\begin{equation}
\begin{aligned}
r_{l+1} &\leq \text{\textcircled{1}} + \text{\textcircled{2}} + \text{\textcircled{3}}\\
& \leq  8\sqrt{2d\log \left(\frac{1}{\beta}\right)} \left(\sqrt{h_1 2^{(1-\alpha)l}}+ \frac{S}{\sqrt{h_12^{(\alpha+1)l}}}\right) \\
& + 4\sigma\sqrt{2\log \left(\frac{1}{\beta}\right)} \left(h_1\sqrt{2^{(2-\alpha) l}}+ \frac{S}{\sqrt{2^{\alpha l}}}\right)\\
& +  8\sigma_0^{\prime}S\sqrt{2d\log\left(\frac{1}{\beta}\right)} \left(h_1 2^{(1-\alpha)l}+ \frac{S}{2^{\alpha l}}\right).
\end{aligned}
\end{equation}
Assume the ``good" event hold in every phase, i.e., under event $\cE_g = \bigcap_{l=1}^L\cE_l$. We have $P\{\cE_g\}\geq 1-3k\beta L$ by applying union bound. Notice that the regret $r_{l+1}$ under SDP-DPE share the same form as CDP-DPE except replacing $\sigma_0$ in CDP-DPE with $\sigma_0^{\prime}$. Hence, we have the total regret satisfies
\begin{equation*}
\begin{aligned}
R_g 
= & O\left(\sqrt{dT^{1-\alpha}\log(kT)}\right) \\
&+ O\left(\sigma T^{1-\alpha/2}\sqrt{\log(kT)}\right) \\
& +O\left(\sigma_0^{\prime} T^{1-\alpha}S\sqrt{d\log(kT)}\right)\\
&+O\left(\sigma_0^{\prime}S^2\sqrt{d\log(kT)}\right)\\
= & 
O\left(\sigma T^{1-\alpha/2}\sqrt{\log(kT)}\right) \\
&+O\left(\frac{Bd^{3/2} T^{1-\alpha}\ln(d/\delta)\sqrt{\log(kT)}}{\epsilon}\right)\\
&+O\left(\frac{Bd^{5/2}\ln(d/\delta)\sqrt{\log(kT)}}{\epsilon}\right),
\end{aligned}
\end{equation*}
where the last step is derived by substituting $\sigma_0^{\prime} = \frac{C_sB\ln(S/\delta)}{\epsilon}$ and $S=O(d)$. 

\smallskip 

Finally, the expected regret of SDP-DPE algorithm is upper bounded by
\begin{equation}
\begin{aligned}
&\E[R(T)] = P(\mathcal{E}_g)R_g + (1-P(\mathcal{E}_g)) \cdot 2T \\
& \leq R_g + 2\beta k L \cdot 2T \\
&=O\left(\sigma T^{1-\alpha/2}\sqrt{\log(kT)}\right) \\
&+O\left(\frac{Bd^{3/2} T^{1-\alpha}\ln(d/\delta)\sqrt{\log(kT)}}{\epsilon}\right)\\
&+O\left(\frac{Bd^{5/2}\ln(d/\delta)\sqrt{\log(kT)}}{\epsilon}\right).
\end{aligned}
\end{equation}

\medskip

\noindent\textbf{Communication cost.} The communicate cost in the shuffle model is slightly different from the central model and the local model because it communicates $(g+b)s_l$ bits from each participating client in the $l$-th phase instead of an $s_l$-dimensional real vector. Based on our setting (Eq.~\eqref{eq:gbp_set} in Algorithm~\ref{alg:shuffle_vec}), we have $(g+b) =\max\left\{O\left(\frac{\sqrt{|U_l|}}{\ln(d)}\right), O\left(\sqrt{d}+\frac{d\ln(d)}{|U_l|}\right), O\left(\frac{\ln(d)}{|U_l|}\right)\right\}$. Combining $s_l\leq S \approx O(d)$, the total communication cost is
$$\sum_{l=1}^L(g+b)s_l|U_l| = O\left(\sum_{l=1}^L \frac{2^{\frac{3}{2}\alpha l}S}{\ln (d)}\right) = O\left(dT^{(3/2)\alpha}\right).$$

\end{proof}

\section{Pure DP-DPE Algorithm Design and Analysis
}\label{app:laplace}
Besides the Gaussian mechanism used in this work, another differential privacy mechanism, the Laplace mechanism, also works well with our algorithmic framework. Different from Gaussian, which ensures approximate DP (i.e., $(\epsilon, \delta)$-DP), Laplace provides pure DP ($\epsilon$-DP or $(\epsilon,0)$-DP). We claim that \emph{our proposed scheme in this paper can be effectively integrated with the Laplace mechanism, which ensures a pure DP and achieves nearly the same regret performance}. 

\medskip

In this appendix, we provide how to modify the algorithm and derive the theoretical results for the Laplace mechanism. 
We first present the standard Laplace mechanism below for reference.

\begin{definition}\label{def:laplace_mech}
({The Laplace Mechanism} \cite{dwork2014algorithmic}).  
Given any vector-valued function $f: \mathbb{N}^{|\mathcal{X}| }\to \mathbb{R}^s$, the Laplace mechanism is definied as:
$$\mathcal{M}_L(x, f(\cdot), \epsilon) = f(x)+(\gamma_1, \dots, \gamma_s),$$
where $\gamma_j$ are i.i.d. random variables drawn from Lap$(\Delta_1)/\epsilon$.
\end{definition}
\begin{theorem}\cite{dwork2014algorithmic}\label{thm:dp}
The Laplace mechanism preserves $(\epsilon,0)$-differential privacy. 
\end{theorem}

Take the central DP model as an example. When adding noise in CDP-DPE (ref.~Eq.~(5) of our manuscript), we let $\gamma_j \overset{\emph{i.i.d.}}{\sim} \text{Lap}(b)$, where $b \triangleq 2Bs_l/(\epsilon |U_l|)$ and $\Delta_1 = 2Bs_l/|U_l|$ is the $\ell_1$-sensitivity of $\frac{1}{|U_l|}\sum_{u\in U_l} \Vec{y}_{l}^u$. From Theorem~\ref{thm:dp}, it is not difficult to show $\epsilon$-DP. 

Then, we need to adjust the confidence width correspondingly as follows:
\begin{equation}
\begin{aligned}
\!W_l\triangleq &\left(\underbrace{\sqrt{\frac{2d}{|U_l|h_l}}}_{\text{action-related}} + \underbrace{\frac{\sigma}{\sqrt{|U_l|}}}_{\text{client-related}} \right) \sqrt{2\log \left(\frac{1}{\beta}\right)} \\
&+\underbrace{\left(2s_ldb + 2db\sqrt{2\log \left(\frac{1}{\beta}\right)}\right)}_{W_{l,p}: \text{ privacy-related}}, \!\label{eq:W_l_pure}
\end{aligned}
\end{equation}
where we denote the last privacy-related term as $W_{l,p} \triangleq 2s_ldb + 2db\sqrt{2\log \left(\frac{1}{\beta}\right)}$.

According to Lemma 4 in \cite{hanna2022differentially}, we can still derive the corresponding concentration inequality as Eq.~\eqref{eq:concentration_privacy_CDP} when using Gaussian mechanism in CDP-DPE, i.e., 
$$
\begin{aligned}
&P\left\{ \left\langle x', V_l^{-1}\sum_{x\in \text{supp}(\pi_l)} xT_l(x)\gamma_p(x)\right\rangle \geq W_{l,p}\right\} \\
& \leq \exp \left\{ \frac{s_l}{2}-\frac{2s_ldb +2db\log \left(\frac{1}{\beta}\right)} {2db }\right\}\\
& \leq \exp \left\{-\frac{s_l}{2} - \log \left(\frac{1}{\beta}\right)\right\}\leq \beta.
\end{aligned}
$$

By replacing $W_l$ in the $l$-th phase for any $1\leq l\leq L$ as identified in Eq.~\eqref{eq:W_l_pure}, we have the regret in the $(l+1)$-th phase under the ``good event'' (Eq.~(38) in the manuscript) becomes 
\begin{equation}
\begin{aligned}
r_{l+1}
& \leq  8\sqrt{2d\log \left(\frac{1}{\beta}\right)} \left(\sqrt{h_1\cdot 2^{(1-\alpha)l}}+ \frac{S}{\sqrt{h_1\cdot 2^{(\alpha+1)l}}}\right) \\
& + 4\sigma\sqrt{2\log \left(\frac{1}{\beta}\right)} \left(h_1\sqrt{2^{(2-\alpha) l}}+ \frac{S}{\sqrt{2^{\alpha l}}}\right)\\
& + \left(\frac{8BSd}{\epsilon}\sqrt{2\log\left(\frac{1}{\beta}\right)} + \frac{4BSd}{\epsilon}\right) \left(h_1\cdot 2^{(1-\alpha)l}+ \frac{S}{2^{\alpha l}}\right).
\end{aligned}
\end{equation}
Finally, we get the expected regret of the CDP-DPE algorithm is upper bounded by 
\begin{equation}
\begin{aligned}
\mathbb{E}[R_g] =&  O\left(\sigma T^{1-\alpha/2}\sqrt{\log(kT)}\right) \\
&+ O\left(\frac{Bd^2 T^{1-\alpha}\sqrt{\log(kT)}}{\epsilon}\right)\\
&+O\left(\frac{Bd^3\sqrt{\log(kT)}}{\epsilon}\right),
\end{aligned}
\end{equation}
which is almost the same as the result in our work using the Gaussian mechanism, except for an extra $\sqrt{d}$ ratio in the last two terms. This difference is quite intuitive since the Laplace mechanism depends on the $\ell_1$-sensitivity of the raw data ($\frac{1}{|U_l|}\sum_{u\in U_l} \Vec{y}_{l}^u$)  while the Gaussian mechanism depends on its $\ell_2$-sensitivity.

Recall that the above algorithm design and results are for DP-DPE employing the Laplace mechanism in the central DP model. One can apply the local model and shuffle model correspondingly in DP-DPE and derive the corresponding results when ensuring pure DP.

\section{Differentially Private Linear Bandits}\label{app:dplb}
In this section, we consider the standard stochastic linear bandits \cite{lattimore2020bandit} 
and provide differentially private algorithms in the central, local, and shuffle DP models.  

\subsection{Model and Algorithmic Framework}
\smallskip
\noindent\textbf{Stochastic linear bandits.} 
In the stochastic linear bandits, there is no client-related uncertainty, and any user/client $u$ can provide direct (noisy) reward observations (i.e., $\theta_u = \theta^*$ in our notations). Specifically, at each round $t$, the learning agent selects an action $x_t$ from the decision set $\cD \subseteq \{x\in \R^d: \Vert x\Vert_2^2\leq 1\}$ with $|\cD|=k$ and receives a reward with mean $\langle \theta^*, x_t \rangle$, where  $\theta^*\in \R^d$ with $\Vert\theta^*\Vert_2\leq 1$ is unknown to the agent. The goal of the agent is to maximize the cumulative reward in $T$ rounds by selecting $x_t$ sequentially. Without knowing $\theta^*$, the agent learns it gradually by collecting the noisy reward observation $y_t = \langle \theta^*, x_t \rangle +\eta_t$ at each $t\in [T]$ from a client.  The noise $\eta_t$ is assume to be conditionally $1$-sub-Gaussian and \emph{i.i.d.} over time. Moreover, we assume that the reward observations are bounded, i.e., $|y_t|\leq B$ for all $t\in [T]$.
Let $x^* \in \argmax_{x\in\cD} \langle \theta^*, x\rangle$	be an optimal action.
Then, the objective of maximizing the cumulative reward is  equivalent to minimizing the regret defined 
as follows:
\begin{equation}
R(T) \triangleq T \langle  \theta^*, x^* \rangle - \sum_{t=1}^T \langle  \theta^*,x_{t} \rangle.
\end{equation}

\noindent\textbf{Privacy. } 
To protect clients' privacy involved in their reward observations, we still consider differential privacy (DP) guarantee in the three trust models in Section~\ref{sec:DP-instantiation} when collecting clients' observations $y_t$ for all $t$.

\smallskip
\noindent\textbf{DP algorithmic framework.} 
To ensure DP in the standard stochastic linear bandits, we only need to address the challenges (\textcircled{c} and \textcircled{d}) mentioned Section~\ref{sec:challenges}.
Following a similar way of ensuring DP with a general \textsc{Privatizer} $\cP=(\cR, \cS, \cA)$, we slightly modify our DP-DPE framework in Section~\ref{sec:alg_design} and design a new differentially private phased elimination algorithmic framework (DP-PE) for the standard linear bandits. We present the detailed pseudo-code of DP-PE in Algorithm~\ref{alg:dp-pe}. 

\begin{algorithm}[ht]
\caption{Differentially Private Phased Elimination (DP-PE)}
\label{alg:dp-pe}
\begin{algorithmic}[1]
\STATE \textbf{Input:} $\cD\subseteq \R^d$, $\phi\in (0,1)$, $\alpha\in (0,1)$, $\beta\in (0,1)$, and $\sigma_n$ 
\STATE \textbf{Initialization:} $l=1$, $t_1=1$, $\cD_1=\cD$, and $h_1= 4d\log \log d +16$
\WHILE{$t_l\leq T$} 
\STATE Find a distribution $\pi_l(\cdot)$ over $ \cD_l$ such that  $ g(\pi_l) \triangleq \max_{x\in \cD_l} \Vert x \Vert_{V(\pi_l)^{-1}}^2 \leq 2d$ and $|\text{supp}(\pi_l)|\leq 4d\log\log d +16$, where $V(\pi_l) \triangleq \sum_{x\in\cD_l} \pi_l(x)xx^{\top}$			
\FOR{each action $x\in \text{supp}(\pi_l)$} 
\STATE Select $T_l(x)=\lceil h_l \pi_l(x)\rceil$ clients, denoted as $U_l(x)$ 
\FOR{each client $u$ in $U_l(x)$}
\STATE Play the action $x$ and observes the reward $y_u(x)$
\item[] \deemph{\# The local randomizer $\cR$ at each client:}
\STATE Run the local randomizer $\mathcal{R}$ and send the output $\cR(y_u(x))$ to $\cS$ 
\STATE If the total number of action pulls reaches $T$, exit
\ENDFOR
\item[] \deemph{\# Computation $\cS$ at a trusted third party:}
\STATE Run the computation function $\cS$ and send the output $\cS(\{\cR(y_u(x))\}_{u\in U_l(x)})$ to the analyzer $\cA$
\item[] \deemph{\# The analyzer $\cA$ at the server:}
\STATE Generate the privately aggregated statistics: $\Tilde{y}_l(x) = \cA\left(\cS\left(\{\cR(y_u)\}_{u\in U_l(x)}\right)\right)$ 
\ENDFOR
\STATE Compute the following quantities:
\begin{equation}
\begin{cases}
V_l = \sum_{x\in \text{supp}(\pi_l)} T_l(x)xx^\top \\
G_l = \sum_{x\in \text{supp}(\pi_l)}   x \tilde{y}_l(x)\\
\Tilde{\theta}_l = V_l^{-1}G_l \label{eq:theta_est}
\end{cases}    
\end{equation}
\STATE Find low-rewarding actions with confidence width $W_l$: \label{alg_line_elimination_lb}%
$$E_l = \left\{x\in \cD_l: \max_{b\in \cD_l} \langle \Tilde{\theta}_l, b-x\rangle > 2W_l\right\}$$ \label{alg_elimination_lb}
\STATE Update: $\cD_{l+1} = \cD_{l}\backslash E_l$, $h_{l+1} = 2h_l$, 
$t_{l+1} = t_l+T_l$, and $l = l+1$ 
\ENDWHILE
\end{algorithmic}
\end{algorithm}

The DP-PE algorithm runs in phases and maintains a set of active actions $\cD_l$, which is updated at the end of each phase. At a high level, each phase consists of the following steps. First, compute a near-$G$-optimal design $\pi_l(\cdot)$(i.e., a distribution) over a set of possibly optimal actions $\cD_l$. For each action $x$ in the support set of $\pi_l$, send $x$ to $T_l(x)$ clients, denoted as $U_l(x)$, where the action $x$ is played and a reward $y_u(x)$ is observed at each client $u\in U_l(x)$. Before being used to estimate $\theta^*$, the reward observations $y_u(x)$ at all clients $u \in U_l(x)$ for each chosen action $x$ is processed by a $\textsc{Privatizer}$ $\cP$ to ensure differential privacy as in the DP-DPE algorithm. We still consider a \textsc{Privatizer}~$\cP = (\cR, \cS, \cA)$ as a process completed by the clients, the server, and/or a trusted third party. As instantiations of $\cP$, we also consider the central, local and shuffle models and provide the detailed implementations of $\cR, \cS, \cA$ in Section~\ref{sec:DP-instantiation-brief}. In all the DP models, the final output $\tilde{y}_l(x)$ of $\cP$ for each action $x$ is a private sum  of its reward observations. With the aggregated statistics $\tilde{y}_l(x)$ for each action $x\in\text{supp}(\pi_l)$, the agent computes the least-square estimator $\tilde{\theta}_l$ according to Eq.~\eqref{eq:theta_est}. Finally, low-rewarding actions are eliminated from $\cD_l$ (Line~\ref{alg_elimination_lb}) based on the following confidence width:
\begin{equation}
\begin{aligned}
W_l\triangleq &\left(\underbrace{\sqrt{\frac{2d}{h_l}}}_{\text{action-related}}  +\underbrace{\sigma_{n}}_{\text{privacy noise}}\right) \sqrt{2\log \left(\frac{1}{\beta}\right)},
\end{aligned}
\end{equation}
where $\sigma_n$ is determined by the privacy noise added in the DP model. 

\subsection{DP-PE Instantiations with different DP Models}\label{sec:DP-instantiation-brief}

We now briefly explain how to instantiate the \textsc{Privatizer} $\cP = (\cR,\cS,\cA)$ in DP-PE using the three representative DP trust models: the central, local, and shuffle models.
In addition, we also present the formal definition of the privacy guarantees regarding $\cP$ under each trust model, which further implies the respective privacy guarantee of DP-PE according to the post-processing property of DP \cite[Proposition~2.1]{dwork2014algorithmic}. 

\subsubsection{The Central Model}
In the central model, each client trusts the server, and the outputs of the server on two neighboring datasets (differing by only one client) should be indistinguishable~\cite{dwork2006calibrating}.

Consider a particular phase $l$ and an action $x$ in supp$(\pi_l)$. The \textsc{Privatizer} $\cP$ is $(\epsilon, \delta)$-differentially-private (or $(\epsilon, \delta)$-DP) if the following is satisfied for any pair of $U_l(x), U_l^{\prime}(x)\subseteq \cU$ that differ by at most one client and for any output $\tilde{y}$ of $\cA$:
\begin{equation*}
\mathbb{P}[\cP\left(\{y_u(x)\}_{u\in U_l(x)}\right)= \tilde{y}] \leq e^{\epsilon} \cdot \mathbb{P}[\cP\left(\{y_u(x)\}_{u\in U_l^{\prime}(x)}\right)= \tilde{y}] + \delta.
\end{equation*}

To achieve this, the \textsc{Privatizer} functions as follows: while both $\cR$ and $\cS$ are simply identity mappings, $\cA$ adds well-tuned Gaussian noise to the sum of $T_l(x)$ reward observations from clients $U_l(x)$ for each action $x$ in $\text{supp}(\pi_l)$ for privacy. That is, 
\begin{equation}
\begin{aligned}
\Tilde{y}_l(x) &= \cP\left(\{y_u(x)\}_{u\in U_l(x)}\right) = \cA\left(\{y_u(x)\}_{u\in U_l(x)}\right)\\
&= \sum_{u\in  U_l(x)} y_u(x)+\gamma_x, \quad \forall x\in \text{supp}(\pi_l), \label{eq:privatizer_cdp_lb}   
\end{aligned}
\end{equation}
where $\gamma_x {\sim} \cN(0,\sigma^2_{nc})$ is \emph{i.i.d.} across actions, and the variance $\sigma^2_{nc}$ depends on the sensitivity of $\sum_{u\in U_l(x)}y_u(x)$, which is $2B$. 
Combining the Gaussian mechanism in Theorem~\ref{thm:gaussian_mech} with the post-processing property of DP in \cite{dwork2014algorithmic}, it is not difficult to obtain the following DP guarantee.
\begin{theorem} \label{thm:lb_cdp}
The DP-PE instantiation using the \textsc{Privatizer} in Eq.~\eqref{eq:privatizer_cdp_lb} with $\sigma_{nc} = \frac{2B\sqrt{2\ln(1.25/\delta)}}{\epsilon} 
$ guarantees $(\epsilon, \delta)$-DP.
\end{theorem}
At the same time, with the above DP guarantee, we can show that the regret under DP-PE in the central model satisfies the following result. 
\begin{theorem}[CDP-PE]\label{thm:regret_cdp_lb}
With $\sigma_n = 2d\sigma_{nc}\sqrt{s_l}/h_l$ in each phase $l$ and $\beta=1/(kT)$, the  DP-PE algorithm with the central model \textsc{Pivatizer} achieves expected regret
\begin{equation*}
\begin{aligned}
\E[R(T)] &= O(\sqrt{dT\log (kT)} \\
&+O\left( \frac{Bd^{3/2}\log(T)\sqrt{\ln(1/\delta)\log(kT)}}{\epsilon} \right) \\
&+ O\left(\frac{Bd^{3/2}\sqrt{\ln(1/\delta)\log(kT)}}{\epsilon}\right).
\end{aligned}
\end{equation*}
\end{theorem}

\subsubsection{The Local Model}
In the local DP model, since clients do not trust the server, each client with a local randomizer $\cR$ is responsible for privacy protection by injecting Gaussian noise; $\cS$ is an identity mapping; $\cA$ simply sums up the (private) rewards from $U_l(x)$ corresponding to action $x$. That is,
\begin{equation}
\Tilde{y}_l(x) =\sum_{u\in U_l(x)} \cR(y_u(x)) = \sum_{u\in U_l(x)} \left(y_u(x)+ \gamma_{u,x}\right), \label{eq:privatizer_ldp_lb}
\end{equation}
where $\gamma_{u,x}{\sim}  \mathcal{N}(0,\sigma^2_{nl})$ is \emph{i.i.d.} across clients, and the variance $\sigma^2_{nl}$ is chosen according to the sensitivity of $y_u(x)$, which is also $2B$.
Consider any phase $l$ and any $x$ in $\text{supp}(\pi_l)$. 
Let $Y_u$ be the set of all possible values of the reward observation $y_u(x)$ at client $u$ for any action $x$. The \textsc{Privatizer} $\cP$ 
is $(\epsilon, \delta)$-local-differentially-private (or $(\epsilon, \delta)$-LDP) if the following is satisfied for any client $u$, for any pair of $y_u(x),y_u^{\prime}(x) \in Y_u$, and for any output $o \in \{\cR(y)|y \in Y_u\}$: 
\begin{equation*}
\mathbb{P}[\cR(y) = o ] \leq e^{\epsilon} \cdot \mathbb{P}[\cR(y^{\prime} ) = o] + \delta.
\end{equation*}
With the above definition, we present the privacy guarantee of DP-PE in the local DP model in Theorem~\ref{thm:ldp_lb}.
\begin{theorem} \label{thm:ldp_lb}
The DP-PE instantiation using the \textsc{Privatizer} in Eq.~(\ref{eq:privatizer_ldp_lb}) with $\sigma_{nl} = 
\frac{2B\sqrt{2\ln(1.25/\delta)}}{\epsilon}$ guarantees $(\epsilon, \delta)$-LDP.
\end{theorem}
\begin{theorem}[LDP-PE]
With $\sigma_n = 2d\sigma_{nc}\sqrt{2s_l/h_l}$ in each phase $l$ and $\beta=1/(kT)$, the  DP-DPE algorithm with the local model \textsc{Pivatizer} achieves expected regret
\begin{equation*}
\begin{aligned}
\E[R(T)] &=O(\sqrt{dT\log (kT)} \\
&+O\left( \frac{Bd^{3/2}\sqrt{\ln(1/\delta)T\log(kT)}}{\epsilon} \right) \\
&+O\left(\frac{Bd^{2}\sqrt{\ln(1/\beta)\log(kT)}}{\epsilon}\right).
\end{aligned}
\end{equation*}
\end{theorem}
\subsubsection{The Shuffle Model}
In the shuffle model, without a trusted agent, we instantiate DP-PE by building on the scalar sum protocol $\cP_{\mathrm{1D}}$ recently developed in \cite{cheu2021shuffle}. 
Consider a particular phase $l$ and any action $x$ in supp$(\pi_l)$. 
Specifically, each local randomizer $\cR$ encodes its inputs (the reward observation $y_u(x)$) by adding random bits; the analyzer $\cA$ outputs the random number whose expectation is the sum of inputs ($\sum_{u\in U_l(x)}y_u(x)$); in addition, a third-party shuffler~$\cS$ is utilized to uniformly at random permute clients' messages (in bits) to hide their sources $u$. We present the concrete pseudocode of $\cR$, $\cS$, and $\cA$ for the shuffle model \textsc{Privatizer} in Algorithm~\ref{alg:scalar_sum}. Finally, the private sum $\tilde{y}_l(x)$ is
\begin{equation*}
\Tilde{y}_l(x) =  \mathcal{P}\left(\{y_u(x)\}_{u\in U_l(x)}\right) 
=\cA\left(\cS \left(\{\cR(y_u(x))\}_{u\in U_l(x)}\right)\right)
\end{equation*}

Similar to the shuffle model in DP-DPE, we let $(\cS\circ \cR)({U_l(x)})\triangleq \cS (\{\cR(y_u(x)) \}_{u\in U_l(x)})$ denote the composite mechanism. 
Formally, the \textsc{Privatizer} $\cP$ is $(\epsilon, \delta)$-shuffle-differentially-private (or $(\epsilon,\delta)$-SDP) if the following is satisfied for any pair of $U_l(x)$, $U_l^{\prime}(x) \subseteq \cU$ that differ by one client and for any possible output $z$ of $\cS\circ \cR$: 
\begin{equation*}
\mathbb{P}[(\cS\circ \cR)({U_l}(x)) = z] \leq e^{\epsilon} \cdot \mathbb{P}[(\cS\circ \cR)({U_l^{\prime}(x))}= z] + \delta.
\end{equation*}


\begin{algorithm}[ht] 
\caption{ $\cP:$ A shuffle protocol for summing $n$ scalars \cite{cheu2021shuffle}} 
\label{alg:scalar_sum}
\begin{algorithmic}[1]
\STATE \textbf{Input:} Scalar database $Y = (y_1, \dots, y_n) \in [-B,B]^n$; $g, b\in \mathbb{N}; p\in (0,1/2)$; $\epsilon, \delta$
\STATE Let 
\begin{equation}
\begin{cases}
g\geq 2B\sqrt{n}\\
b=\lceil \frac{180g^2\ln{(2/\delta)}}{(\epsilon/2)^2n}\rceil\\
p = \frac{90g^2\ln{(2/\delta)}}{b(\epsilon/2)^2n} \label{eq:gbp_set_lb}
\end{cases}    
\end{equation}
\item[] \deemph{// Local Randomizer}
\item[] \textbf{function} 
$\cR(y)$ \\
\begin{ALC@g}
\STATE Shift data to enforce non-negativity: $y \gets y+B$
\STATE Set $\Bar{y} \gets \lfloor yg/(2B)\rfloor$  
\STATE Sample rounding value $\gamma_1 \sim \textbf{Ber}(yg/(2B) - \Bar{y})$
\STATE Set $\hat{y}\gets \bar{y}+\gamma_1$
\STATE Sample privacy noise value $\gamma_2 \sim \textbf{Bin}(b,p)$
\STATE Let $\phi$ be a multi-set of $(g+b)$ bits, containing $\hat{y}+\gamma_2$ copies of $1$ and $g+b-(\hat{y}+\gamma_2)$ copies of $0$
\STATE Report $\phi$ to the shuffler
\end{ALC@g}
\item[] \textbf{end function}
\item[] \deemph{// Shuffler}
\item[] \textbf{function} 
$\cS(\phi_1, \dots, \phi_n)$  \deemph{// each $\phi_i$ consists of $(g+b)$ bits}\\
\begin{ALC@g}
\STATE Shuffle and output all  $(g+b)n$ bits 
\end{ALC@g}
\item[] \textbf{end function} \\
\item[] \deemph{// Analyzer}
\item[] \textbf{function} 
$\mathcal{A}(\cS(\phi_1, \dots, \phi_n))$
\begin{ALC@g}
\STATE Compute $z \gets \frac{2B}{g} (\mathrm{sum}(\cS(\phi_1, \dots, \phi_n))-nbp)$ 
\STATE Re-center: $o \gets z - nB$ 
\STATE  Output the estimator $o$
\end{ALC@g}
\textbf{end function}
\end{algorithmic}
\end{algorithm}
Before showing the privacy guarantee of the shuffle model in Algorithm~\ref{alg:scalar_sum}, we provide a lemma derived directly from the original results in \cite{cheu2021shuffle}.
\begin{lemma}[Lemma 3.1 in \cite{cheu2021shuffle}]\label{lem:p1d_sdp}
Fix any number of users $n$, $\hat{\epsilon}<15$, and $0<\delta<1/2$. Let $g\geq B\sqrt{n}$, $b>\frac{180g^2\ln(1/\delta)}{\hat{\epsilon}^2 n}$, and $p=\frac{90g^2\ln(2/\delta)}{b\hat{\epsilon}^2 n}$. Then, 
\begin{itemize}
\item[1.] 
the \textsc{Privatizer} $\cP$ in Algorithm~\ref{alg:scalar_sum} is $\left(\hat{\epsilon}\left(\frac{2}{g}+1\right), \delta\right)$-SDP;
\item[2.] for any $Y\in [-B, B]^n$, $\cP$ is an unbiased estimator of $\sum_{i=1}^n y_i$, and the error is sub-Gaussian with variance $\sigma_{ns}^2=O\left(\frac{B^2\log(1/\delta)}{\hat{\epsilon}^2}\right)$.
\end{itemize} 
\end{lemma}

Set the parameters $p,b, g$ according to Eq.~\eqref{eq:gbp_set_lb} with $\epsilon = 2\hat{\epsilon}$ in the \textsc{Privatizer} specified in Algorithm~\ref{alg:scalar_sum}. Then, we derive the following privacy guarantee. 
\begin{theorem} \label{thm:sdp_lb}
For any $\epsilon\in(0,30)$ and $\delta\in (1,1/2)$, the DP-PE instantiation using the \textsc{Privatizer} specified in Algorithm~\ref{alg:scalar_sum}  guarantees $(\epsilon, \delta)$-SDP.
\end{theorem}
\begin{theorem}[SDP-PE]
With $\sigma_n = 2d\sigma_{ns}\sqrt{s_l}/h_l=O\left(\frac{Bd\sqrt{s_l\ln(1/\delta)}}{\epsilon h_l}\right)$ in each phase $l$ and $\beta=1/(kT)$, the  DP-PE algorithm with the shuffle model \textsc{Pivatizer} specified in Algorithm~\ref{alg:scalar_sum} achieves expected regret
\begin{equation*}
\begin{aligned}
\E[R(T)] &= 
O(\sqrt{dT\log (kT)} \\
&+O\left( \frac{Bd^{3/2}\log(T)\sqrt{\ln(1/\delta)\log(kT)}}{\epsilon} \right)\\
&+ O\left(\frac{Bd^{3/2}\sqrt{\ln(1/\delta)\log(kT)}}{\epsilon}\right).
\end{aligned}
\end{equation*}
\end{theorem}

\subsubsection{Discussions}
Regarding the results derived in this subsection, we make the following remarks. 

\begin{remark}[Privacy ``for free"] 
In \cite{lattimore2020bandit}, we know that the non-private phased elimination algorithm achieves $O(\sqrt{dT\log(kT)})$ regret. From the above Theorems \ref{thm:lb_cdp}, \ref{thm:ldp_lb} and Theorem~\ref{thm:ldp_lb}, we derive that the DP-PE algorithm enables us to achieve privacy guarantee ``for free" in the central and shuffle model as in the DP-DPE algorithm. 
\end{remark}
\begin{remark}[Extended to other privacy noise]
To be consistent with our main content, we employ the Gaussian mechanism here to achieve the corresponding  approximate DP in the central and local DP models. However, a Laplacian mechanism can also be employed in the central and the local models to achieve pure $(\epsilon, 0)$-DP  privacy guarantee.
\end{remark}
\begin{remark}[Phase length initialization $h_1$] 
In the standard linear bandits without client-related uncertainty, we initialize the phase length to be  the upper bound of the support set size in every phase, i.e., $h_1=4d\log\log d+16$, in order to derive a better regret cost (i.e., a lower order of $d$) due to privacy guarantees, especially for the local DP model.   
\end{remark}

\subsection{Proofs for the Results in Section~\ref{app:dplb}}
To prove the regret upper bound of the DP-PE algorithm under three DP models, we follow a similar line to the proof of Theorem~\ref{thm:regret_dp_dpe}. 
First, we present the concentration result for the DP-PE algorithm under the three DP models with the setting in Table~\ref{tab:parameter_setting_lb}. 
\begin{table*}[t]
\centering
\caption{Setting}
\label{tab:parameter_setting_lb}
\begin{tabular}{lcc}
\toprule
Algorithm & $\sigma_n$& Notes\\
\midrule
CDP-PE & $\sigma_n = 2d\sigma_{nc}{\sqrt{s_l}}/{h_l}$  & $\sigma_{nc} = \frac{2B\sqrt{2\ln(1.25/\delta)}}{\epsilon }$\\
LDP-PE& $\sigma_n = 2d\sigma_{nl}\sqrt{{2s_l}/{h_l}}$ & $\sigma_{nl} = \frac{2B\sqrt{2\ln(1.25/\delta)}}{\epsilon }$\\
SDP-PE & $\sigma_n = 2d\sigma_{ns}\sqrt{s_l}/h_l$ &$\sigma_{ns} = O\left(\frac{B\sqrt{\ln(1/\delta)}}{\epsilon}\right)$\\
\bottomrule
\end{tabular}
\end{table*}
\begin{theorem}
Set DP-PE in the central, local, and shuffle models (i.e., CDP-PE, LDP-PE, and SDP-PE respectively) based on Table~\ref{tab:parameter_setting_lb}. In any phase $l$, all of CDP-PE, LDP-PE, and SDP-PE satisfies the following concentration, for any particular $x\in \cD_l$
\begin{equation}
\begin{aligned}
&P\left\{ \langle \Tilde{\theta}_l-\theta^*, x\rangle \geq W_l\right\}\leq 2\beta,  \quad \text{and,}\\
&P\left\{ \langle \theta^* - \Tilde{\theta}_l, x\rangle \geq W_l\right\}\leq 2\beta.\label{eq:concentration_ineq_lb}
\end{aligned}
\end{equation} 
\end{theorem}
\begin{proof}
We prove the first concentration inequality in Eq.~\eqref{eq:concentration_ineq_lb} for CDP-PE, LDP-PE, and SDP-PE in the following, and the second inequality can be proved symmetrically. 
According to Eq.~\eqref{eq:theta_est} in Algorithm~\ref{alg:dp-pe}, we have
\begin{equation*}
\begin{aligned}
&\Tilde{\theta}_l = V_l^{-1}G_l\\
= &  V_l^{-1}\sum_{x\in \text{supp}(\pi_l)}  x \Tilde{y}_l(x) \\
= & V_l^{-1}\sum_{x\in \text{supp}(\pi_l)}  x \sum_{u\in U_l(x)}y_u(x) \\
&+ V_l^{-1}\sum_{x\in \text{supp}(\pi_l)}  x \underbrace{\left(\tilde{y}_l(x) -\sum_{u\in U_l(x)}y_u(x)\right)}_{\gamma_p(x)}\\
\overset{(a)}{=} &  V_l^{-1}\sum_{x\in \text{supp}(\pi_l)}  x \sum_{u\in U_l(x)}y_u(x) +  V_l^{-1}\sum_{x\in \text{supp}(\pi_l)}  x \gamma_p(x)\\
= & V_l^{-1}\sum_{{x\in \text{supp}(\pi_l)}}  \!x \sum_{t\in \cT_l(x)}(x^{\top}\theta^*+\eta_t) +  \!V_l^{-1}\!\sum_{x\in \text{supp}(\pi_l)}  x \gamma_p(x)\\
=  & V_l^{-1}\sum_{x\in \text{supp}(\pi_l)}  T_l(x)x x^{\top}\theta^* + V_l^{-1}\sum_{x\in \text{supp}(\pi_l)}  x \sum_{t\in \cT_l(x)}\eta_t \\
&+ V_l^{-1}\sum_{x\in \text{supp}(\pi_l)}  x  \gamma_p(x)\\
=  & \theta^* + V_l^{-1}\sum_{t\in \cT_l}  x_t \eta_t + V_l^{-1}\sum_{x\in \text{supp}(\pi_l)}  x  \gamma_p(x),
\end{aligned}
\end{equation*}
where we represent the noise introduced for protecting privacy associated with action $x$ by $\gamma_p(x)\triangleq \tilde{y}_l(x)-\sum_{u\in U_l(x)}y_u(x)$, which varies according to the specified DP models, and 
the last step is due to our decision $x_t = x $ for any $x$ in $\cT_l(x)$. 
Then, the difference between estimation and the true reward of each action $x'\in\cD_l$ is
\begin{equation}
\begin{aligned}
&\langle \tilde{\theta}-\theta^*, x'\rangle\\
= & \left\langle x', V_l^{-1}  \sum_{t \in \cT_l}   x_t\eta_t\right\rangle  +  \left\langle x', V_l^{-1}\sum_{x\in \text{supp}(\pi_l)}  x  \gamma_p(x)\right\rangle. \label{eq:est_error}    
\end{aligned}
\end{equation}
Note that $ \eta_t$ is \emph{i.i.d} $1$-sub-Gaussian over $t$ and that the chosen action $x_{t}$ at $t$ is deterministic in the $l$-th phase under the DP-PE algorithm. Combining the following result,
$$\sum_{t\in \cT_l}  \langle x, V_l^{-1}  x_t\rangle^2 = x^{\top} V_l^{-1} \left(\sum_{t\in \cT_l}x_tx_t^{\top}\right) V_l^{-1} x=
\Vert x\Vert_{V_l^{-1}}^2,$$
where the second equality is due to $V_l = 
\sum_{t\in \cT_l} x_{t}x_{t}^{\top}$,
we derive that the first term of the RHS of Eq.~\eqref{eq:est_error} is $\Vert x\Vert_{V_l^{-1}}$-sub-Gaussian. Besides, we have $\Vert x\Vert_{V_l^{-1}}^2 
\leq \frac{\Vert x\Vert_{V_l(\pi_l)^{-1}}^2}{h_l}
\leq \frac{g(\pi_l)}{h_l}
\leq  \frac{2d}{h_l}$ by the near-G-optimal design. Based on the property of sub-Gaussian, we obtain 
\begin{equation}
\begin{aligned}
&P\left\{ \left\langle x', V_l^{-1}  \sum_{t \in \cT_l}   x_t\eta_t\right\rangle \geq \sqrt{\frac{4d}{h_l}\log\left(\frac{1}{\beta}\right)}\right\}\\
&\leq \exp\left\{-\frac{\frac{4d}{h_l}\log(1/\beta)}{2\Vert x'\Vert_{V_l^{-1}}^2}\right\} 
= \beta.      \label{eq:W_l1_prob_bnd_lb}
\end{aligned}
\end{equation}
Based on the union bound, we have
\begin{equation}
\begin{aligned}
& P\left\{ \langle \Tilde{\theta}_l-\theta^*, x\rangle \geq W_l\right\} \\
\leq & P\left\{ \sum_{t \in \cT_l}\left\langle x',   V_l^{-1}    x_t\right\rangle \eta_t \geq \sqrt{\frac{4d}{h_l}\log \left(\frac{1}{\beta}\right)}\right\} \\
& +P\left\{  \left\langle x', V_l^{-1}\sum_{x\in \text{supp}(\pi_l)}  x  \gamma_p(x)\right\rangle \geq \sqrt{2\sigma_n^2\log \left(\frac{1}{\beta}\right)} \right\}\\
\leq & \beta +P\left\{  \left\langle x', V_l^{-1}\sum_{x\in \text{supp}(\pi_l)}  x  \gamma_p(x)\right\rangle \geq \sqrt{2\sigma_n^2\log \left(\frac{1}{\beta}\right)} \right\}.
\end{aligned}
\end{equation}
To derive the concentration in Eq.~\eqref{eq:concentration_ineq_lb}, it remains to show that the second term is less than $\beta$ under each of the three DP models. Due to different $\gamma_p(x)$ in different DP models, we analyze the second term respectively. 

\smallskip
\textbf{i) CDP-PE.} In the central model, the private output of the \textsc{Privatizer} $\cP$  is $\tilde{y}_l(x) = \sum_{u\in  U_l(x)} y_u(x)+\gamma_x$, where $\gamma_x \overset{i.i.d.}{\sim} \cN(0, \sigma_{nc}^2)$. Then, $\gamma_p(x) \sim \cN(0, \sigma_{nc}^2)$ in the central model and is \emph{i.i.d.} across actions $x\in \text{supp}(\pi_l) $. Note that 
\begin{equation*}
\left\langle x', V_l^{-1}\sum_{x\in \text{supp}(\pi_l)} x\gamma_p(x)\right\rangle 
= \sum_{x\in \text{supp}(\pi_l)}  \left\langle x', V_l^{-1} x\right\rangle \gamma_{p}(x), 
\end{equation*}
and that $\gamma_p(x) \overset{\emph{i.i.d.}}{\sim} \cN(0, \sigma_{nc}^2)$. The variance (denoted by $\sigma_{\text{sum}}^2$) of the above sum of \emph{i.i.d.} Gaussian variables is 
$$
\begin{aligned}
\sigma_{\text{sum}}^2 &= \sum_{x\in \text{supp}(\pi_l)}  \left\langle x', V_l^{-1} x\right\rangle^2  \sigma_{nc}^2\\
&\overset{(a)}{\leq} \sum_{x\in \text{supp}(\pi_l)} \left(\max_{x\in \cD_l} \Vert x\Vert_{V_l^{-1}}^2\right)^2 \sigma_{nc}^2 \leq  \frac{s_l \cdot 4d^2 \cdot \sigma_{nc}^2}{h_l^2} = \sigma_n^2,   
\end{aligned}
$$
where $(a)$ is from $\langle x', V_l^{-1} x\rangle \leq \max_{x\in \cD_l} \Vert x\Vert_{V_l^{-1}}^2$ for the positive definite matrix $V_l$. Combining the tail bound for Gaussian variables, we have 
$$
\begin{aligned}
&P\left\{ \left\langle x', V_l^{-1}\sum_{x\in \text{supp}(\pi_l)} x\gamma_p(x)\right\rangle\geq \sqrt{2\sigma_n^2\log\left(\frac{1}{\beta}\right)}\right\} \\
&\leq \exp \left\{-\frac{2\sigma_n^2\log(1/\beta)}{2\sigma_{\text{sum}}^2 }\right\} \leq  \beta.   
\end{aligned}
$$
Hence, the first concentration inequality in Eq.~\eqref{eq:concentration_ineq_lb} holds for CDP-PE algorithm. 

\smallskip

\textbf{ii) LDP-PE.} In the local model, the output of the \textsc{Privatizer} $\cP$ is $\tilde{y}_l(x)= \sum_{u\in U_l}y_u(x)+\gamma_{u,x}$, where $\gamma_{u,x} \sim \cN(0, \sigma_{nl}^2)$ is \emph{i.i.d.} across clients. Then $\gamma_p(x) = \sum_{u\in U_l} \gamma_{u,x}$ in the local model, and 
\begin{equation*}
\begin{aligned}
&\left\langle x', V_l^{-1}\sum_{x\in \text{supp}(\pi_l)} x\gamma_{p}(x)\right\rangle \\
&= \sum_{x\in \text{supp}(\pi_l)}  \sum_{u\in U_l(x)} \left\langle x', V_l^{-1} x\right\rangle  \gamma_{u,x}.    
\end{aligned}
\end{equation*}
The variance (denoted by $\sigma_{\text{sum}}^2$) of the sum of the above $T_l(x)\times |\text{supp}(\pi_l)|$ \emph{i.i.d.} Gaussian variables satisfies
$$
\begin{aligned}
\sigma_{\text{sum}}^2 &= \sum_{u\in U_l(x)}\sum_{x\in \text{supp}(\pi_l)}  \left\langle x', V_l^{-1} x\right\rangle^2  \sigma_{nl}^2\\
&\leq  \frac{T_l(x) s_l \cdot 4d^2 \cdot \sigma_{nl}^2}{h_l^2} \leq \frac{8s_ld^2\sigma_{nl}^2}{h_l} =\sigma_n^2,    
\end{aligned}
$$
where the last step is from $T_l(x)\leq T_l=\sum_{x\in \text{supp}(\pi_l) T_l(x)}\leq h_l + s_l \leq h_l+h_1\leq 2h_l$. Then, we have
$$
\begin{aligned}
&P\left\{ \sum_{u\in U_l(x)}\sum_{x\in \text{supp}(\pi_l)}\left\langle x', V_l^{-1} x\gamma_{x}\right\rangle\geq \sqrt{2\sigma_n^2\log \left(\frac{1}{\beta}\right)}\right\} \\
&\leq \exp \left\{-\frac{2\sigma_n^2\log \left(\frac{1}{\beta}\right)}{2\sigma_{\text{sum}}^2 }\right\} \leq \beta.    
\end{aligned}
$$
Hence, the first concentration inequality in Eq.~\eqref{eq:concentration_ineq_lb} holds for LDP-PE algorithm.

\textbf{iii) SDP-PE. } In the shuffle model, the output of the \textsc{Privatizer} $\cP$ is, $\Tilde{y}_l =  \cA\left(\cS \left(\{\cR(y_u(x))\}_{u\in U_l(x)}\right)\right)$, where $\cA, \cS$ and $\cR$ follow Algorithm~\ref{alg:scalar_sum}. From Lemma~\ref{lem:p1d_sdp}, we know that the output of the $\cP\left(\{y_u(x)\}_{u\in U_l(x)}\right) $ is an unbiased estimator of $\sum_{u\in U_l(x)}{y}_u(x)$ and that the error distribution is sub-Gaussian with variance $\sigma_{ns}^2 = O\left(\frac{B^2\ln (1/\delta)}{\epsilon^2 }\right)$. 
Then, $\gamma_p(x) = \Tilde{y}_l(x) - \sum_{u\in U_l(x)} y_u(x)$ is $\sigma_{ns}$-sub-Gaussian with $\E[\gamma_p(x)] = 0$ in the shuffle model. Besides, $\gamma_p(x)$ is \emph{i.i.d.} over each coordinate corresponding to each action $x$ in the support set supp$(\pi_l)$. Recall that 
\begin{equation*}
\left\langle x', V_l^{-1}\sum_{x\in \text{supp}(\pi_l)} x\gamma_p(x)\right\rangle 
= \sum_{x\in \text{supp}(\pi_l)}  \left\langle x', V_l^{-1} x\right\rangle \gamma_{p}(x). 
\end{equation*}
The variance (denoted by $\sigma_{\text{sum}}^2$) of the above sum of \emph{i.i.d.} sub-Gaussian variables is 
$$
\begin{aligned}
\sigma_{\text{sum}}^2 &= \sum_{x\in \text{supp}(\pi_l)}  \left\langle x', V_l^{-1} x\right\rangle^2  \sigma_{ns}^2\\
&\overset{(a)}{\leq} \sum_{x\in \text{supp}(\pi_l)} \left(\max_{x\in \cD_l} \Vert x\Vert_{V_l^{-1}}^2\right)^2 \sigma_{ns}^2 \leq  \frac{s_l \cdot 4d^2 \cdot \sigma_{ns}^2}{h_l^2} = \sigma_n^2,    
\end{aligned}
$$
where $(a)$ is from $\langle x', V_l^{-1} x\rangle \leq \max_{x\in \cD_l} \Vert x\Vert_{V_l^{-1}}^2$ for the positive definite matrix $V_l$. Combining the tail bound for sub-Gaussian variables, we have 
$$
\begin{aligned}
&P\left\{ \left\langle x', V_l^{-1}\sum_{x\in \text{supp}(\pi_l)} x\gamma_p(x)\right\rangle\geq \sqrt{2\sigma_n^2\log\left(\frac{1}{\beta}\right)}\right\} \\
&\leq \exp \left\{-\frac{2\sigma_n^2\log(1/\beta)}{2\sigma_{\text{sum}}^2 }\right\} \leq  \beta.   
\end{aligned}
$$
Hence, the first concentration inequality in Eq.~\eqref{eq:concentration_ineq_lb} holds for SDP-PE algorithm.

By now, we complete the proof for Eq.~\eqref{eq:concentration_ineq_lb} with the symmetrical argument.
\end{proof}

\subsubsection{Proof of Theorem~\ref{thm:regret_cdp_lb}}
In the following, we start with analyzing regret in a specific phase under CDP-PE and then combine all phases together to get the total regret incurred by the CDP-PE algorithm.

\begin{proof}
\textbf{1) Regret in a specific phase.}
Based on the concentration in Theorem~\ref{thm:concentration_dp}, we define a ``good" event at $l$-th phase as $\mathcal{E}_l$: 
$$
\begin{aligned}
&\left\{\langle \theta^*- \Tilde{\theta}_l, x^*\rangle \leq W_l\right\} \quad \text{and}, \\
&\left\{\langle \Tilde{\theta}_l-\theta^*, x\rangle \leq W_l  \right\}, \quad \forall x\in \cD\backslash \{x^*\} .
\end{aligned}
$$
It is not difficult to derive $P(\mathcal{E}_l)\geq 1-2k\beta$ via union bound.

Under event $\mathcal{E}_l$, we have the following two observations:
\begin{itemize}
\item[\textbf{1.}] If the optimal action $x^*\in \cD_l$, then $x^* \in \cD_{l+1}$. 
\item[\textbf{2.}] For any $x\in \cD_{l+1}$, we have $\langle \theta^*, x^* - x\rangle \leq 4W_l$. 
\end{itemize}
For any particular $l$, under event $\cE_l$, we have the regret incurred in the $(l+1)$-th phase satisfies (according to the second observation)
\begin{equation}
\begin{aligned}
r_{l+1} &=  \sum_{t \in \cT_{l+1}}  \langle \theta^*, x^* - x_{t} \rangle \\
&\leq \sum_{t\in\cT_{l+1}} 4W_{l}\\
& = 4T_{l+1}W_l\\
& \leq \underbrace{4T_{l+1}\sqrt{\frac{4d}{h_l}\log \left(\frac{1}{\beta}\right)}}_{\text{\textcircled{1}}} + \underbrace{4T_{l+1}\sqrt{2\sigma_n^2\log\left(\frac{1}{\beta}\right)}}_{\text{\textcircled{2}}}.
\end{aligned}
\end{equation}
Note that $T_l=\sum_{x\in \text{supp}(\pi_l)}T_l(x) \leq h_l+s_l$. For \textcircled{1}, we have 
\begin{equation}
\begin{aligned}
\text{\textcircled{1}}& \leq \ 4(h_1\cdot 2^l+s_{l+1})\sqrt{\frac{4d}{h_1\cdot 2^{l-1}}\log \left(\frac{1}{\beta}\right)}\\
& = 8\sqrt{2d\log \left(\frac{1}{\beta}\right)} \left(\sqrt{h_1\cdot 2^l}+ \frac{s_{l+1}}{\sqrt{h_1\cdot 2^l}}\right)\\
& \leq 8\sqrt{2d\log \left(\frac{1}{\beta}\right)} \left(\sqrt{h_1\cdot 2^l}+ \frac{S}{\sqrt{h_1\cdot 2^l}}\right),
\end{aligned}
\end{equation}
where $S\triangleq 4d\log\log d+16\geq s_l$.
As to the second term \textcircled{2}, we have $\sigma_n^2 = \frac{4s_ld^2\sigma_{nc}^2}{h_l^2}$, and then
\begin{equation*}
\begin{aligned}
\text{\textcircled{2}}
& \leq   4(h_1\cdot 2^l+s_{l+1})\sqrt{\frac{8s_ld^2\sigma_{nc}^2\log (1/\beta)}{h_1^2\cdot 2^{2(l-1)}}}\\
& = 16d\sigma_{nc}\sqrt{2s_l\log \left(\frac{1}{\beta}\right)} \left(1+ \frac{s_{l+1}}{h_1\cdot 2^{l}}\right)\\
& = 16d\sigma_{nc}\sqrt{2S\log \left(\frac{1}{\beta}\right)} \left(1+ \frac{S}{h_1\cdot 2^{l}}\right).
\end{aligned}
\end{equation*}
Then, for any $l\geq 2$, the regret in the $l$-th phase $r_l$ is upper bounded by 
\begin{equation}
\begin{aligned}
r_l 
\leq &8\sqrt{2d\log \left(\frac{1}{\beta}\right)} \left(\sqrt{h_1\cdot 2^{l-1}}+ \frac{S}{\sqrt{h_1\cdot 2^{l-1}}}\right)\\
&+16d\sigma_{nc}\sqrt{2S\log \left(\frac{1}{\beta}\right)} \left(1+ \frac{S}{h_1\cdot 2^{l-1}}\right) \label{eq:regret_phase_l_lb}.
\end{aligned}    
\end{equation}

\textbf{2) Total regret.}
Define $\mathcal{E}_g$ as the event where the ``good" event occurs in every phase, i.e., $\mathcal{E}_g \triangleq \bigcap_{l=1}^L \mathcal{E}_l$. It is not difficult to obtain  $P\{\mathcal{E}_g\} \geq 1- 2k\beta L$ by applying union bound. At the same time, let $R_g$ be the regret under event $\mathcal{E}_g$, and $R_b$ be the regret if event $\mathcal{E}_g$ does not hold. Then, the expected total regret in $T$ is $\E[R(T)] = P(\mathcal{E}_g)R_g + (1-P(\mathcal{E}_g)) R_b$.

Under event $\mathcal{E}_g$, the regret in the $l$-th phase $r_l$ satisfies Eq.~\eqref{eq:regret_phase_l_lb} for any $l\geq 2$. Combining $r_1 \leq 2T_1 \leq 4h_1$ ( since $\langle \theta^*, x^*-x \rangle \leq 2$ for all $x \in \cD$), we have
\begin{equation*}
\begin{aligned}
R_g &= \sum_{l=1}^L r_l  \\
& \leq 2(h_1+S)\\
&+ \sum_{l=2}^L 8\sqrt{2d\log \left(\frac{1}{\beta}\right)} \left(\sqrt{h_1 2^{l-1}}+ \frac{S}{\sqrt{h_12^{l-1}}}\right)\\
& + \sum_{l=2}^L 16d\sigma_{nc}\sqrt{2S\log \left(\frac{1}{\beta}\right)} \left(1+ \frac{S}{h_1\cdot 2^{l-1}}\right) \\
&\leq 2(h_1+S)+ 8\sqrt{2d\log(1/\beta)}\left(4\sqrt{h_12^{L-1}}+\frac{3S}{\sqrt{h_1}}\right) \\
&+ 16d\sigma_{nc}\sqrt{2S\log(1/\beta)} \left(L-1+S/h_1\right).
\end{aligned}
\end{equation*}
Note that $ h_L \leq T_L \leq T$, which indicates $2^{L-1} \leq T/h_1$, and $L\leq \log (2T/h_1)$.
Then, the above inequality becomes
\begin{equation*}
\begin{aligned}
R_g =& \sum_{l=1}^L r_l\\
\leq & 2(h_1+S) \\
&+ 8\sqrt{2d\log(1/\beta)  } \cdot 4\sqrt{T} + 24S\sqrt{2d\log(1/\beta)}/\sqrt{h_1} \\
&+ 16d\sigma_{nc}\sqrt{2S\log(1/\beta)}\cdot \log(T/h_1) \\
&+ 16dS/h_1\sigma_{nc}\sqrt{2S\log(1/\beta)}.
\end{aligned}
\end{equation*}
On the other hand, $R_b \leq 2T$ 
since $\langle \theta^*, x^*-x \rangle \leq 2$ for all $x \in \cD$. 
Choose $\beta = \frac{1}{kT}$ in Algorithm~\ref{alg:dp-pe}. Finally, we have the following results:
\begin{equation*}
\begin{aligned}
&\E[R(T)] = P(\mathcal{E}_g)R_g + (1-P(\mathcal{E}_g)) R_b \\
\leq & R_g + 2k\beta L \cdot 2T\\
\leq & 2(h_1+S) + 32\sqrt{2dT\log(kT)  }+24S\sqrt{2d/h_1\log(kT)} \\
&+16d\sigma_{nc}\sqrt{2S\log(kT)}\log(T) \\
&+16dS/h_1\sigma_{nc}\sqrt{2S\log(kT)} + 4\log (2T/h_1)\\
= &O(\sqrt{dT\log (kT)} \\&+O(\sigma_{nc}d^{3/2}\log(T)\sqrt{\log(kT)} \\&+O(\sigma_{nc}d^{3/2}\sqrt{\log(kT)}). 
\end{aligned}
\end{equation*}
Finally, substituting $\sigma_{nc}=\frac{2B\sqrt{2\ln(1.25/\delta)}}{\epsilon}$, we have the total expected regret under DP-PE with the central model \textsc{Privatizer} is
\begin{equation*}
\begin{aligned}
\E[R(T)] =& O(\sqrt{dT\log (kT)} \\
&+O\left( \frac{Bd^{3/2}\log(T)\sqrt{\ln(1/\delta)\log(kT)}}{\epsilon}  \right) \\
&+ O\left(\frac{Bd^{3/2}\sqrt{\ln(1/\delta)\log(kT)}}{\epsilon}\right) .
\end{aligned}
\end{equation*}
\end{proof}

\subsubsection{Proof of Theorem~\ref{thm:ldp_lb}}
\begin{proof}

\textbf{1) Regret in a specific phase.} 
Recall the ``good" event at $l$-th phase $\mathcal{E}_l = \{\langle \theta^*- \Tilde{\theta}_l, x^*\rangle \leq W_l\} \bigcap \{\langle \Tilde{\theta}_l-\theta^*, x\rangle \leq W_l, \forall x\in \cD_l\backslash \{x^*\}  \}$. 
Under event $\mathcal{E}_l$, we have the following two observations:
\begin{itemize}
\item[\textbf{1.}] If the optimal action $x^*\in \cD_l$, then $x^* \in \cD_{l+1}$. 
\item[\textbf{2.}] For any $x\in \cD_{l+1}$, we have $\langle \theta^*, x^* - x\rangle \leq 4W_l$. 
\end{itemize}
For any particular $l$, under event $\cE_l$, we have the regret incurred in the $(l+1)$-th phase satisfies (according to the second observation)
\begin{equation*}
\begin{aligned}
r_{l+1} &=  \sum_{t \in \cT_{l+1}}  \langle \theta^*, x^* - x_{t} \rangle \\
&\leq \sum_{t\in\cT_{l+1}} 4W_{l}\\
& = 4T_{l+1}W_l \\
& \leq \underbrace{4T_{l+1}\sqrt{\frac{4d}{h_l}\log \left(\frac{1}{\beta}\right)}}_{\text{\textcircled{1}}} + \underbrace{4T_{l+1}\sqrt{2\sigma_n^2\log\left(\frac{1}{\beta}\right)}}_{\text{\textcircled{2}}}.
\end{aligned}
\end{equation*}

For \textcircled{1}, we already have 
$\text{\textcircled{1}} = 8\sqrt{2d\log \left(\frac{1}{\beta}\right)} \left(\sqrt{h_1\cdot 2^l}+ \frac{S}{\sqrt{h_1\cdot 2^l}}\right)$.
As to the second term \textcircled{2}, we have $\sigma_n^2 = \frac{8d^2s_l\sigma_{nl}^2}{h_l}$ in the local model and then
\begin{equation*}
\begin{aligned}
\text{\textcircled{2}}&\leq 4(h_1\cdot 2^l+s_{l+1})\sqrt{\frac{16s_ld^2\sigma_{nl}^2\log (1/\beta)}{h_1\cdot 2^{(l-1)}}}\\ 
& = 16d\sigma_{nl}\sqrt{2s_l\log\left(\frac{1}{\beta}\right)}\left(\sqrt{h_1\cdot 2^l} + \frac{s_{l+1}}{\sqrt{h_1\cdot 2^l}}\right)\\
& \leq 16d\sigma_{nl}\sqrt{2S\log\left(\frac{1}{\beta}\right)}\left(\sqrt{h_1\cdot 2^l} + \frac{S}{\sqrt{h_1\cdot 2^l}}\right).
\end{aligned}
\end{equation*}
Then, the regret in the $l$-th phase is upper bounded by
$$
\begin{aligned}
r_l \leq &8\sqrt{2d\log \left(\frac{1}{\beta}\right)} \left(\sqrt{h_12^{l-1}}+ \frac{S}{\sqrt{h_1 2^{l-1}}}\right)\\
&+16d\sigma_{nl}\sqrt{2S\log \left(\frac{1}{\beta}\right)} \left(\sqrt{h_1 2^{l-1}}+ \frac{S}{\sqrt{h_12^{l-1}}}\right)\\
\end{aligned}
$$

\medskip
\textbf{3) Total regret.}

Under the ``good" event, the total regret is
\begin{equation*}
\begin{aligned}
&R_g = \sum_{l=1}^L r_l  \\
\leq & 2(h_1+S)\\
& +\sum_{l=2}^L 8\sqrt{2d\log \left(\frac{1}{\beta}\right)} \left(\sqrt{h_12^{l-1}}+ \frac{S}{\sqrt{h_1 2^{l-1}}}\right)\\
&+\sum_{l=2}^L 16d\sigma_{nl}\sqrt{2S\log \left(\frac{1}{\beta}\right)} \left(\sqrt{h_1 2^{l-1}}+ \frac{S}{\sqrt{h_12^{l-1}}}\right)\\
\leq & 2(h_1+S)+ 8\sqrt{2d\log\left(\frac{1}{\beta}\right)}\left(4\sqrt{h_1 2^{L-1}}+ \frac{3S}{\sqrt{h_1}}\right)\\
&+ 16d\sigma_{nl}\sqrt{2S\log\left(\frac{1}{\beta}\right)}\left(4\sqrt{h_1 2^{L-1}}+ \frac{3S}{\sqrt{h_1}}\right)\\
\leq &2(h_1+S)+ 8\sqrt{2d\log\left(\frac{1}{\beta}\right)} \left( 4\sqrt{T} +\frac{S}{\sqrt{h_1}}\right) \\
&+16d\sigma_{nl}\sqrt{2S\log\left(\frac{1}{\beta}\right)} \left( 4\sqrt{T} +\frac{S}{\sqrt{h_1}}\right) \\
& \deemph{(h_L=2^{L-1}h_1\leq T)}\\
\leq & 2(h_1+S)+ 32\sqrt{2dT\log(kT)}+64d\sigma_{nl}\sqrt{2ST\log(kT)}\\
&+ (8S\sqrt{d/h_1}+16d\sigma_{nl}S^{3/2}/\sqrt{h_1})\sqrt{2\log\left(kT\right)}   \\
& \deemph{(\beta=\frac{1}{kT})}\\
= & O(\sqrt{dT\log(kT)}) \\
&+ O(d^{3/2}\sigma_{nl}\sqrt{T\log(kT)})\\
&+O(\sigma_{nl}d^2\sqrt{\log(kT)}). 
\end{aligned}
\end{equation*}
Finally, substituting $\sigma_{nl}=\frac{2B\sqrt{2\ln(1.25/\delta)}}{\epsilon}$ and combining $R_b\leq 2k\beta L \cdot 2T=O(\log(T))$, we have the total expected regret under DP-PE with the local model \textsc{Privatizer} is
\begin{equation*}
\begin{aligned}
\E[R(T)] = & O(\sqrt{dT\log (kT)} \\
&+O\left( \frac{Bd^{3/2}\sqrt{\ln(1/\delta)T\log(kT)}}{\epsilon} \right) \\
&+O\left(\frac{Bd^{2}\sqrt{\ln(1/\beta)\log(kT)}}{\epsilon}\right).
\end{aligned}
\end{equation*}
\end{proof}

\subsubsection{Proof of Theorem~\ref{thm:sdp_lb}}
\begin{proof}

\textbf{1) Regret in a specific phase. } In the shuffle model, $\sigma_n^2 = \frac{4s_ld^2\sigma_{ns}^2}{h_l^2}$, and then the regret in the $l$-th phase has the same form as in the central model, i.e., 
\begin{equation*}
\begin{aligned}
r_l\leq &8\sqrt{2d\log \left(\frac{1}{\beta}\right)} \left(\sqrt{h_1 2^{l-1}}+ \frac{S}{\sqrt{h_1 2^{l-1}}}\right) \\
&
+16d\sigma_{ns}\sqrt{2S\log \left(\frac{1}{\beta}\right)} \left(1+ \frac{S}{h_1 2^{l-1}}\right).
\end{aligned}
\end{equation*}

\medskip
\textbf{2) Total regret.}
Under the ``good" event, the total regret is
\begin{equation*}
\begin{aligned}
&R_g = \sum_{l=1}^L r_l  \\
\leq & 2(h_1+S)+ \sum_{l=2}^L 8\sqrt{2d\log \left(\frac{1}{\beta}\right)} \left(\sqrt{h_1 2^{l-1}}+ \frac{S}{\sqrt{h_1 2^{l-1}}}\right)\\
& + \sum_{l=2}^L 16d\sigma_{ns}\sqrt{2S\log \left(\frac{1}{\beta}\right)} \left(1+ \frac{S}{h_1 2^{l-1}}\right) \\
\leq &2(h_1+S)+ 8\sqrt{2d\log(1/\beta)}\left(4\sqrt{h_12^{L-1}}+\frac{3S}{\sqrt{h_1}}\right) \\
&+ 16d\sigma_{ns}\sqrt{2S\log(1/\beta)} \left(L-1+S/h_1\right)\\
\leq & 2(h_1+S)+ 8\sqrt{2d\log(1/\beta)  } \cdot 4\sqrt{T} \\
&+ 24S\sqrt{2d\log(1/\beta)}/\sqrt{h_1} \\
&+ 16d\sigma_{ns}\sqrt{2S\log(1/\beta)}\cdot \log(T/h_1) \\
&+ 16dS/h_1\sigma_{ns}\sqrt{2S\log(1/\beta)} \\ &\deemph{(h_L=h_1\cdot 2^{L-1}\leq T)}\\
\leq & 2(h_1+S)+ 32\sqrt{2dT\log(kT) } + 24S\sqrt{2d/h_1\log(kT)}\\
&+ 16d\sigma_{ns}\sqrt{2S\log(kT)}\cdot \log(T/h_1) \\
&+ 16dS/h_1\sigma_{ns}\sqrt{2S\log(kT)}  \\
& \deemph{(\beta=\frac{1}{kT})}\\
= &O(\sqrt{dT\log (kT)} \\
&+O(\sigma_{ns}d^{3/2}\log(T)\sqrt{\log(kT)}\\
&+O(\sigma_{ns}d^{3/2}\sqrt{\log(kT)}). 
\end{aligned}
\end{equation*}
Finally, substituting $\sigma_{ns}=O\left(\frac{B\sqrt{d}\ln(d/\delta)}{\epsilon}\right)$ and combining $R_b\leq 2k\beta L \cdot 2T=O(\log(T))$, we have the total expected regret under DP-PE with the shuffle model \textsc{Privatizer} is
\begin{equation*}
\begin{aligned}
\E[R(T)] &= O(\sqrt{dT\log (kT)} \\
&+O\left( \frac{Bd^{3/2}\log(T)\sqrt{\ln(1/\delta)\log(kT)}}{\epsilon} \right)\\
&+ O\left(\frac{Bd^{3/2}\sqrt{\ln(1/\delta)\log(kT)}}{\epsilon}\right).
\end{aligned}
\end{equation*}
\end{proof}

\section{Additional Numerical Results}\label{app:experiments}

In this appendix, we present additional numerical results that readers might be interested in. 

\textbf{How other parameters affect regret performance.} In this paper, we assume a large amount of population, that is, the number of users is infinite. We claim that the number of users has little impact on the regret as long as the population is large enough.  Considering readers might still be interested in how population affects the final regret, we conducted more simulations regarding this question and show the result in Figure~\ref{fig:users}. It is run on the non-private DPE algorithm to remove the possible impact of DP. The results show that the final regret of DPE is almost the same after the number of users $N$ is large enough, which validates our claim. 

\begin{figure}[t!] 
\centering
\includegraphics[width=0.5\linewidth]{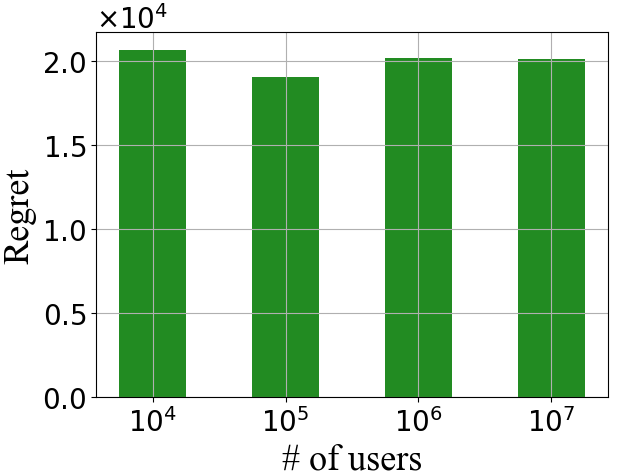}
\caption{Regret vs. $N$}
\label{fig:users}  
\vspace{-3mm}
\end{figure}

\begin{figure}[t!] 
\centering
\includegraphics[width=0.5\linewidth]{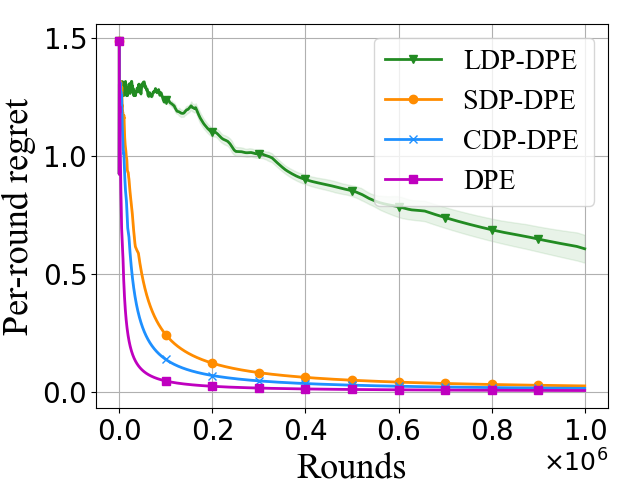}
\caption{Regret performance comparisons of DP-DPE with different DP models: Per-round regret vs. time with privacy parameters $\epsilon=10$ and $\delta=0.01$. }
\label{fig:dp-dpe}  
\end{figure}
\textbf{DP-DPE on real-world data set.}
In Section~\ref{sec:simulation}, we evaluate DP-DPE on synthetic data where the unknown global parameter is generated by uniform random sampling in the range $(-1,1)^d$. In this subsection, we present additional numerical results by evaluating DP-DPE on an unknown global reward function extracted from real-world data set of the Yahoo! Learning to Rank challenge  \cite{chapelle2011yahoo}, which has been used previously in some works regarding linear bandits (e.g., \cite{foster2018practical}, \cite{zanette2021design}). We present the regret performance of DP-DPE in different DP models in Figure~\ref{fig:dp-dpe}. From Figure~\ref{fig:dp-dpe}, we have a similar observation to the result in Figure~2(b): LDP-DPE while providing the strongest DP guarantee incurs the largest regret; CDP-DPE and SDP-DPE achieve a similar regret to that of the DPE algorithm (non-private version of DP-DPE) which coincides with our theoretical result of achieving privacy for ``free''. 

\end{document}